\documentclass[a4paper,fleqn,12pt]{report}

\usepackage[]{geometry}
\usepackage[latin1]{inputenc}
\usepackage[UKenglish]{babel}
\usepackage[UKenglish]{isodate}
\usepackage{amsmath}
\usepackage{amsfonts}
\usepackage{amssymb}
\usepackage{amsthm}
\usepackage{graphicx}
\usepackage{chngpage}
\usepackage{calc}
\PassOptionsToPackage{hyphens}{url}
\usepackage{hyperref}
\usepackage[nameinlink]{cleveref}
\usepackage{fancyhdr}
\usepackage{titletoc}
\usepackage[explicit]{titlesec}
\usepackage[numbers]{natbib}
\usepackage[dvipsnames]{xcolor}
\usepackage[T1]{fontenc}

\hypersetup{
	colorlinks=true,
	linkcolor=black,
	urlcolor=black,
	citecolor=black
}

\cleanlookdateon

\makeatletter
\newcommand{\@assignment}[0]{Assignment}
\newcommand{\assignment}[1]{\renewcommand{\@assignment}{#1}}
\newcommand{\@supervisor}[0]{}
\newcommand{\supervisor}[1]{\renewcommand{\@supervisor}{#1}}
\newcommand{\@yearofstudy}[0]{}
\newcommand{\yearofstudy}[1]{\renewcommand{\@yearofstudy}{#1}}
\makeatletter

\newtheorem{theorem}{Theorem}
\newtheorem{definition}{Definition}

\newcommand{\Csharp}{%
  {\settoheight{\dimen0}{C}C\kern-.05em \resizebox{!}{\dimen0}{\raisebox{\depth}{\#}}}}

\usepackage{subcaption}
\usepackage{caption}

\usepackage{tikz-cd}

\newcommand\given[1][]{\:#1\vert\:}

\usepackage{dsfont}

\usepackage{longtable}

\author{Omar Tanner}
\title{Multi-Agent Car Parking using\\ Reinforcement Learning}
\supervisor{Dr. Paolo Turrini}
\yearofstudy{3\textsuperscript{rd}}

\makeatletter
\fancypagestyle{plain}{
	\fancyhf{}
	\fancyhead[LE]{\thepage}
	\fancyhead[RE]{\textit{\@author}}
	\fancyhead[RO]{\thepage}
	\fancyhead[LO]{\textit{\@title}}
}
\makeatother

\begin{document}

\input{common/titlepage.tex}

\pagestyle{plain}

\begin{abstract}

As the industry of autonomous driving grows, so does the potential interaction of groups of autonomous cars. Combined with the advancement of Artificial Intelligence and simulation, such groups can be simulated, and safety-critical models can be learned controlling the cars within. This study applies reinforcement learning to the problem of multi-agent car parking, where groups of cars aim to efficiently park themselves, while remaining safe and rational.

Utilising robust tools and machine learning frameworks, we design and implement a flexible car parking environment in the form of a Markov decision process with independent learners, exploiting multi-agent communication. We implement a suite of tools to perform experiments at scale, obtaining models parking up to $7$ cars with over a $98.1\%$ success rate, significantly beating existing single-agent models. We also obtain several results relating to competitive and collaborative behaviours exhibited by the cars in our environment, with varying densities and levels of communication. Notably, we discover a form of collaboration that cannot arise without competition, and a 'leaky' form of collaboration whereby agents collaborate without sufficient state.

Such work has numerous potential applications in the autonomous driving and fleet management industries, and provides several useful techniques and benchmarks for the application of reinforcement learning to multi-agent car parking.

\textit{\textbf{Keywords:} Reinforcement Learning, Multi-Agent, Deep Learning, Self-Driving, Car Parking, Simulation, Markov Decision Process}

\end{abstract}

\tableofcontents

\chapter*{Acknowledgements}
\label{ch:ack}
\addcontentsline{toc}{chapter}{Acknowledgements}

I would like to thank my supervisor Dr. Paolo Turrini for his close guidance throughout the project, kindly sparing his time to steer the project and provide his wisdom. I would also like to thank Dr. Miguel Alonso Jr. for providing mentorship on the use of reinforcement learning and ML-Agents, advising areas of the implementation. Due to the computational requirements for the project, I extend my thanks to the University of Warwick's Department of Computer Science for providing sufficient compute resources. Finally, I would like to acknowledge Michael Gale for providing the LaTeX template and providing additional high-level advice on writing the report.
\chapter{Introduction} 

\label{ch:introduction}

Self-driving cars are an active area of research and industrial application, with autonomous car parking being an important component of their operation. During parking scenarios, cars may potentially be in very close proximity to other cars, making the contexts in which self-parking systems operate very dense and dangerous. In addition, human drivers exhibit both collaborative and competitive behaviours during car parking, in the interests of minimising congestion and optimising their own parking efficiency, respectively. Thus, self-parking systems are critical to get right, in the interests of safety, congestion and efficiency.

With the advancement of self-driving cars, in the future there may arise groups of mutually autonomous cars, which could potentially communicate with each-other in real-time through the means of, for example, a shared server. Such communication could greatly ease existing estimation techniques (mostly powered by computer vision) for the position, velocity, and other local information of other cars, by communicating it directly. By modelling each autonomous car in a car park as an agent, we have a natural multi-agent system, on which deep reinforcement learning (RL) techniques can be applied, since deep RL has been found to be able to solve multi-agent tasks with real-world complexity. With such a multi-agent abstraction and assumed communication abilities, this project studies how multi-agent reinforcement-learning techniques can be applied to park groups of mutually autonomous cars in ways that are safe, efficient, and logical from the perspective of human drivers. 

Such an approach comes with its own difficulties. One difficulty is that, from the perspective a single agent, the environment (car park) is non-stationary, due to the existence of the other agents. In addition, the environment is highly unforgiving, as a single crash could be fatal. There is an inherent complexity in the behaviour exhibited by the agents, since a balance between competition and collaboration must be met, which may require extensive foresight -- for example, when an agent concedes a mutually desired space to another agent that is far away but slightly closer. Even beyond successfully parking cars without collisions, there is much more to be desired from high-performing models, which must also reduce congestion and optimise the efficiency and 'human-ness' of their behaviour, in order to be trusted by human drivers.

Overall, abstracting groups of self-parking cars as a multi-agent system is innovative and complex, and self-parking systems are critical to get right in the interests of safety, potentially risking human life. Therefore, our project is worthwhile to study.

\section{Related work}

There has been a lot of recent work in the literature on single-agent and multi-agent reinforcement learning (MARL), with MARL being applied to complex problems relating to traffic control and even car parking. Much of this research uses bespoke or existing industry-standard simulations in which RL is applied. However, we could not find studies that model the cars in a car park as a multi-agent system.

\subsection{Single-Agent RL}

For the single-agent case, recent work by \citet{rl-collision-avoidance} (2021) applied RL to the problem of navigating a robot to a goal while avoiding dynamic obstacles. Their problem is similar to the single-agent agent version of our problem, where the goal in our case is a parking space and their robot is a car in our case. Their approach trains an RL algorithm inside simulations using the Stage\footnote{\url{http://playerstage.sourceforge.net}} and Gazebo \cite{gazebo} simulators, and uses lidar sensors around the robot as information of nearby obstacles to avoid them. They harness the PPO \cite{ppo} algorithm to successfully learn a model that controls the robot to avoid the obstacles in real-time. PPO is a highly performant modern deep RL algorithm that uses policy-gradient and actor-critic methods to learn high-performing models in record time and with minimal hyper-parameter tuning. Currently, PPO is one of the most commonly used RL algorithms in practise due to its scalability, data-efficiency and robustness. \citet{rl-collision-avoidance} transferred their model into the real-world and found that their robot avoided real obstacles without fault.

Another recent work by \citet{ppo-self-parking} (2020) applies RL to essentially the single-agent version of our problem, where a single car parks itself in a static environment of moderate complexity. Their approach uses the Unity3D engine to simulate a car park, with a single moving car in the environment as the agent and other stationary (static) parked cars as obstacles. They too successfully apply PPO to learn models that park cars in their environment with more than a $95\%$ success rate, with varying configurations and parking zones.

A more primitive work can be seen by \citet{autoparking-qlearning} (2018), whom applies Q-Learning \cite{qlearning} to a single-agent car parking scenario. While successful, the parking space in their simulation remains fixed, and Q-Learning essentially over-fits to the specific configuration of their environment rather than using sensor data to dynamically avoid obstacles, making the work slightly unrealistic. However, the study formulates the motion of a car in a discrete manner appropriate for the application of Q-Learning, which may be useful to harness when modelling our environment in a discrete manner.

\subsection{Multi-Agent RL}

For the multi-agent case, a recent survey by \citet{marl-survey} (2021) provides an overview of the current developments in the field of MARL, concluding that the latest accomplishments address problems with real-world complexity. Their survey reviews recent advancements in multi-agent deep reinforcement learning (MADRL), showing that the use of deep learning methods has enabled the successful application of RL to problems of real-world complexity, as opposed to tabular problems that earlier approaches are limited to. They also discuss existing challenges in MARL, and methods in which they are addressed. Such problems and their respective solutions are considered in our application of MARL, as we tackle the problem of non-stationary using experience replay, enable communication among the agents via broadcasting, and aid coordination and scalability via independent learners. Such techniques are further discussed in Background.

Also mentioned by \citet{marl-survey} are policy-gradient and actor-critic RL algorithms, including PPO. In a recent study by \citet{ppo-in-ma-games} (2021), multi-agent PPO (MAPPO) - a multi-agent variant of PPO - is investigated, and found to achieve surprisingly strong performance in three popular multi-agent environments. Their study concludes that 'properly configured MAPPO is a competitive baseline for MARL tasks', motivating the use of PPO for MARL tasks.

MARL has also been applied to related problems. \citet{marl-for-traffic-control} (2019) apply MADRL to the problem of adaptive traffic signal control in complex urban traffic networks, describing it as a 'promising data driven approach'. They harness the microscopic traffic simulator (SUMO) \cite{sumo} to simulate urban traffic, and train both traditional Q-Learning based algorithms and their own actor-critic deep RL method in their environments. While agents in their problem control traffic signals at road intersections - different to our agents - their study highlights many of the limitations with traditional Q-Learning algorithms when applied in multi-agent scenarios. In addition, their study highlights the advantages of using deep neural networks (DNNs) with actor-critic methods for MARL tasks.

\subsection{Implementation Methods}

RL algorithms are often trained within simulations, due to their ability to create vast quantities of training data, which is required by RL due to its significant data-inefficiency at present. In the literature, it's common to use either bespoke or existing industry-standard simulations, and interface with them to apply RL algorithms within.

Bespoke simulations are often created using the Unity3D engine\footnote{\url{https://unity.com}}, which is a general-purpose video-game engine that's also harnessed for 2D and 3D simulation work due to its generality and ease of use. In conjunction with the Unity Machine Learning Agents Toolkit (ML-Agents) \cite{ml-agents}, the Unity3D engine enables easy implementation of RL simulation environments, and eases the training of various RL algorithms. \citet{ppo-self-parking} used such a combination of technologies to implement their car parking RL simulation and successfully apply PPO. In addition, \citet{deep-rl-for-traffic-control} (2018) used a similar approach to apply RL to the problem of optimizing vehicle flow through road intersections, also harnessing the Unity3D engine to build a traffic simulator closely based on real-world traffic. They used Python to train their policy-gradient deep RL algorithms, and established an interface between Unity3D and Python using socket programming.

Despite the flexibility a bespoke simulation allows, pre-built simulations are commonly used in the literature for simulating urban traffic in realistic road networks. One such simulation suite is SUMO \cite{sumo}, harnessed by \citet{marl-for-traffic-control} in their application of MARL to large-scale traffic signal control, among many other works relating to traffic management on road networks. While useful for the study of high-level routing and traffic management, the tool isn't suited to every problem and has its limitations as identified by \citet{deep-rl-for-traffic-control}, whom instead built their own traffic simulator using the Unity3D engine to realistically validate their research idea. Since our study is future-facing and the intention of the SUMO suite is to simulate present-day scenarios, it may also not be suitable for our project. Since the author is familiar with the Unity3D engine and it enables the most flexibility, our study takes the route of implementing bespoke simulations using the Unity3D engine, harnessing ML-Agents to simplify the application of RL. Such an approach builds off of the success in \citet{ppo-self-parking}'s highly related work.

\section{Objectives}
\label{sec:objectives}

Our project intends to investigate how to effectively model cars in a car park as a multi-agent system, and apply RL and deep learning techniques to enable the learning of safe, efficient and logical models that autonomously park groups of mutually autonomous cars. To achieve this, the project investigates the following research questions (known as the project's objectives):

\begin{enumerate}
    \item \label{obj-1} How can a car park be realistically abstracted as an environment such that RL can be feasibly applied to learn high-performing models for the cars within?
    \item \label{obj-2} What are the constraints that different RL algorithms impose on the abstraction of the environment, and how to their obtained models differ?
    \item \label{obj-3} To what extent does communication between the cars in the environment vary our obtained models?
    \item \label{obj-4} How much competitive and collaborative behaviour is exhibited by the cars under the control of our models, and how do such behaviours affect their performance?
\end{enumerate}

Such objectives are pursued with the motivation to better understand the performance and limitations of RL-based self-parking systems, and the technologies assumed and required by the autonomous cars within. We also seek to identify and understand the behavioural patterns that emerge among groups of autonomous cars trained with MARL techniques, to assess the interoperability of RL-based self-parking systems with human drivers, and the social acceptability of their behaviour. With such improved understanding, RL-based self-parking systems may be better designed in the future, and their safety and efficiency may be subsequently improved.
\chapter{Background}
\label{ch:background}

\section{Environment}
\label{ch:Environment}

To use RL to learn models that control an agent in an environment, one must first model the environment as a Markov Decision Process (MDP) \cite{comprehensive-survey-marl}. A Markov Decision Process is a tuple $\langle \mathbb{S}, \mathbb{A}, \mathbb{P}, \mathbb{R} \rangle$, where $\mathbb{S}$ is the finite set of states, $\mathbb{A}$ is the finite set of actions, $\mathbb{P} : \mathbb{S} \times \mathbb{A} \times \mathbb{S} \to \mathbb{R}_{[0,1]}$ is the state transition probability function, and $\mathbb{R} : \mathbb{S} \times \mathbb{A} \times \mathbb{S} \to \mathbb{R}$ is the reward function. For episodes in the simulation of bounded length $\tau$ time-steps, an episode starts at time $t = 0$ with an agent in an initial state $s_0 \in \mathbb{S}$, and ends at time $t = \tau - 1$, with each integer time-step $t \in \{0,1,...,\tau - 1\}$. $s_0$ is commonly a probabilistic function of $\mathbb{S}$: $P_0 : \mathbb{S} \to \mathbb{R}_{[0,1]}$, where the environment is instantiated in state $s_0$ with probability $P_0(s_0)$.  The dynamics of an MDP are such that at time $t$ an agent is in a state $s \in \mathbb{S}$, and takes an action $a \in \mathbb{A}$ to result in a state $s' \in \mathbb{S}$ at time $t + 1$ with probability $\mathbb{P}(s,a,s')$, and receive a reward $r \in \mathbb{R}(s,a,s')$ from the environment, where a larger value of $r$ indicates a better state $\to$ action $\to$ resulting state transition. \cite{comprehensive-survey-marl} For fully deterministic environments, the MDP formulation may be simplified, with a state transition function $\mathbb{P} : \mathbb{S} \times \mathbb{A} \to \mathbb{S}$ (as $s'$ is always known from $s$ and $a$), and a reward function $\mathbb{R} : \mathbb{S} \times \mathbb{A} \to \mathbb{R}$.

For environments where an agent only has a limited or probabilistic view of its current state $s$ at time $t$, the environment can be modelled as a Partially Observable Markov Decision Process (POMDP) \cite{pomdps-and-robotics}. A POMDP is a tuple $\langle \mathbb{S}, \mathbb{A}, \mathbb{P}, \mathbb{R}, \mathbb{O}, \Omega \rangle$, where $\mathbb{S}$, $\mathbb{A}$, $\mathbb{P}$ and $\mathbb{R}$ are the same as in the MDP, $\mathbb{O}$ is the finite set of observations the agent can make of the environment, and $\Omega : \mathbb{S} \times {A} \times \mathbb{O} \to \mathbb{R}_{[0,1]}$ is the probability $P(o | s', a)$ that an agent receives an observation $o$ at time $t + 1$ after taking an action $a$ at time $t$ and resulting in state $s'$ at time $t + 1$. \cite{pomdps-and-robotics} For environments in which the observations and state transitions are fully deterministic (known as being fully deterministic), one may simplify $\Omega$ (as $s'$ is always known from $a$), thus $\Omega : \mathbb{S} \to \mathbb{O}$, which may be thought of as a reduction of the actual state of the agent $s \in \mathbb{S}$ to $o \in \mathbb{O}$, which may be useful to compress $s$ when it's more effective to use $o$ than $s$, or to limit the information an agent has of its current state when it's infeasible to know it accurately. Since a fully deterministic POMDP with is equivalent to an MDP $\langle \mathbb{O}, \mathbb{A}, \mathbb{P}, \mathbb{R} \rangle$, we refer to such POMDPs as MDPs for simplicity, and henceforth refer to them as MDPs in this study.

\section{Reinforcement Learning}

\subsection{Q-Learning}
With an MDP modelling the environment, RL algorithms can be applied to learn models controlling agents within it \cite{comprehensive-survey-marl}. An agent's behaviour is described by its policy $\pi$, specifying how the agent chooses its actions given the state, which may be either stochastic $\pi : \mathbb{S} \times \mathbb{A} \to \mathbb{R}_{[0,1]}$ or deterministic $\pi : \mathbb{S} \to \mathbb{A}$. A common goal of an RL algorithm is to maximise an agent's expected discounted future reward at each time-step $t$. To do so, it maintains an estimate of the expected discounted future reward after executing an action $a_t$ in state $s_t$ at time $t$, receiving a reward $r_{t+1}$ at time $t+1$, and following the current policy $\pi$ from time $t + 1$ onwards as a Q-Function:

\begin{equation}
Q^\pi(s_t,a_t) = \mathop{\mathbb{E}}_\mathbb{P}[\sum_{t'=t+1}^{\tau - 1}\gamma^{t' - t - 1}r_{t'}]
\end{equation}

where the discount factor $\gamma \in \mathbb{R}_{(0,1)}$ encodes increasing uncertainty of future rewards and bounds the sum. The optimal Q-Function is defined as $Q^*(s,a) = \max_{\pi}Q^{\pi}(s,a)$, and may be iteratively approximated via Q-Learning \cite{qlearning} (harnessing the Bellman equation): 

\begin{equation}
\label{eq:q-learning}
Q_{t+1}(s_t, a_t) \leftarrow (1 - \alpha)Q_t(s_t, u_t) + \alpha[r_{t+1} + \gamma \max_{a'}Q_t(s_{t+1}, a')]
\end{equation}

where $\alpha \in \mathbb{R}_{(0,1]}$ controls the amount in which the agent updates its estimate of the Q-Function from the feedback (reward) it receives from the environment.

$Q_t$ maintained by Q-Learning has been proven to converge to $Q^*$ when the agent tries all state-action pairs with non-zero probability, and explicitly stores the value of the Q-Function for all such pairs in a Q-Table \cite{qlearning}. In practise, to ensure the agent explores enough of the state-action space to guarantee Q-Learning to converge, an epsilon-greedy policy is used, where the agent selects a random action with probability $\epsilon \in \mathbb{R}_{(0, 1)}$, and the greedy action with probability $1 - \epsilon$. In Q-Learning, $\gamma$, $\alpha$ and $\epsilon$ are known as hyper-parameters. After Q-Learning has converged to $Q^*$, the optimal policy $\pi^*$ can be obtained via the greedy policy $\pi^*(s) = arg\max_a(Q^*(s,a))$, making Q-Learning a popular RL method to learn an optimal policy controlling an agent in an MDP. \cite{comprehensive-survey-marl}

\subsection{Information in MDPs}

As mentioned in \ref{ch:Environment}, fully deterministic POMDPs may reduce their source MDP via $\Omega$ to yield a new MDP with lower dimensionality ($|\mathbb{O}| < |\mathbb{S}|$). However, such a reduction may lose critical information with respect to the optimal policy $\pi^*$ in the original MDP. When such an optimal policy is preserved in the reduced MDP, it is said to have full information, which occurs when $\Omega$ enables sufficient distinction of $\mathbb{S}$ from $\mathbb{O}$ (Definition \ref{def:fi-pomdp}, Theorem \ref{thm:optimal-policy-fi-pomdp}).

\begin{definition}[Full Information Reduced MDP]\label{def:fi-pomdp}
Let $\langle\mathbb{S}, \mathbb{A}, \mathbb{P}, \mathbb{R}, \mathbb{O}, \Omega \rangle$ be a fully deterministic POMDP reducing the deterministic MDP $\langle\mathbb{S}, \mathbb{A}, \mathbb{P}, \mathbb{R} \rangle$, and policies $\pi : \mathbb{S} \to \mathbb{A}$ and $\pi_\Omega : \mathbb{O} \to \mathbb{A}$ denote policies in the MDP and POMDP respectively. If $\forall s \in \mathbb{S}, a \in \mathbb{A} : \pi^*(s,a) = \pi_{\Omega}^*(\Omega(s),a)$, then the POMDP, as well as the reduced MDP $\langle \mathbb{O}, \mathbb{A}, \mathbb{P}, \mathbb{R} \rangle$, are said to have full information.
\end{definition}

\begin{theorem}[Optimal Policy Learning in Full Information Reduced MDPs]\label{thm:optimal-policy-fi-pomdp}
Let $\langle\mathbb{S}, \mathbb{A}, \mathbb{P}, \mathbb{R}, \mathbb{O}, \Omega \rangle$ be a fully deterministic deterministic POMDP with full information extending the deterministic MDP $\langle\mathbb{S}, \mathbb{A}, \mathbb{P}, \mathbb{R} \rangle$, and policies $\pi : \mathbb{S} \to \mathbb{A}$ and $\pi_\Omega : \mathbb{O} \to \mathbb{A}$ denote policies in the MDP and POMDP respectively. Then, $\pi^*$ can always be obtained from $\pi_\Omega^*$.
\end{theorem}
\begin{proof}
$\pi^* = \pi_\Omega^* \circ \Omega$.
\end{proof}

Theorem \ref{thm:optimal-policy-fi-pomdp} implies that if the environment is a reduced MDP with full information, then an RL algorithm can be trained on the reduced MDP (receiving the observations as the state, rather than the definite underlying MDP state) to yield an optimal policy in the underlying MDP. This is useful because it enables the number of possible states of an MDP to be reduced by abstracting it as as a reduced MDP with fewer observations than states, simplifying the task of learning an optimal policy.

\subsection{Deep Reinforcement Learning}

Despite Q-Learning having attractive optimality guarantees, it is often intractable when applied in practise. Deep learning techniques have successfully been applied to RL, vastly improving its scalability. At present, PPO is a state of the art deep RL algorithm, with the ability to learn extremely high performing policies in MDPs. \cite{ppo}

\subsubsection{Deep Q-Learning}

Q-Learning is often impractical when applied in practise, due to it potentially taking too long to visit all of the state-action pairs, or being too expensive to store the whole Q-Table. Q-Learning's application cost is proportional to $|\mathbb{S} \times \mathbb{A}|$ for the MDP on which it is applied, quickly becoming infeasible for complex environments. Thus, deep learning is harnessed to approximate the Q-Function by paramaterising it over weights $\theta$: $Q(s,a;\theta) \approx Q^*(s,a)$, and using a function approximator such as a neural network to optimise a loss function $L(\theta)$ that conforms to the Bellman equation: a technique known as Deep Q-Learning (DQL) \cite{dqn}. 

The neural network approximating the Q-Function is such that the inputs are a set of real numbers $s \in \mathbb{R}^n$ (usually bounded between 0 and 1) corresponding to the state, and the outputs are the values of $Q(s,a) \forall a \in A$ corresponding to the Q-value of each action for the input state. Encoding the state and the actions in such a way resolves the scalability issues of Q-Learning. By way of example, suppose the state set $\mathbb{S}$ of an MDP is $\mathbb{S} = S_1 \times S_2$ where $S_1 = S_2 = \mathbb{N}_{[0,n_s)}$, and the action set is $\mathbb{A} = A_1 \times A_2$ where $|A_1| = |A_2| = \mathbb{N}_{[0,n_a)}$. The Q-Table in Q-Learning has $|\mathbb{S} \times \mathbb{A}| = n_s^2 * n_a^2$, which scales poorly and becomes prohibitively large with e.g. $n_s = 1000, n_a = 10 \implies |\mathbb{S} \times \mathbb{A}| = 100,000,000$. However, we can map the state $(s_1, s_2) \in \mathbb{S}$ to two real numbers $(s_1 / n_s, s_2 / n_s)$ and the action $(a_1, a_2) \in \mathbb{A}$ to two real numbers $(a_1 / n_a, a_2 / n_a)$, yielding a neural network approximating the Q-Table with two input neurons and two output neurons (with appropriate hidden units and layers in-between), which is far more tractable. Note how this mapping resolves vertical scalability issues with Q-Learning, as the neural network would have the same number of input and output neurons regardless of the value of $n_s$ and $n_a$, whereas the Q-Table increases in size proportional to $n_s$ and $n_a$, eventually becoming intractable.

\subsubsection{PPO}
In addition, there are policy-gradient, actor-critic methods such as PPO, which make the policy stochastic $\pi(a|s;\theta)$ by parameterising over weights $\theta$, which are optimised via stochastic gradient ascent methods on a loss function $L(\theta)$ \cite{ppo}. In addition to approximating the agent's policy, we simultaneously approximate the agent's value function, which is the value of $Q(s,a)$ in Q-Learning. `Actor-critic' refers to the combination of these approximators, where the actor is the approximator of the policy, and the critic is the approximator of the value function: both approximated using neural networks with a loss defined for each. The actor decides which action to take, while the critic tells the actor how good its action was and how it should adjust. \cite{rl-book}. \Cref{fig:actor-critic} shows a diagram of the actor-critic architecture \cite{rl-book}.

\begin{figure}[!h]
\centering
\includegraphics[width=0.5\textwidth]{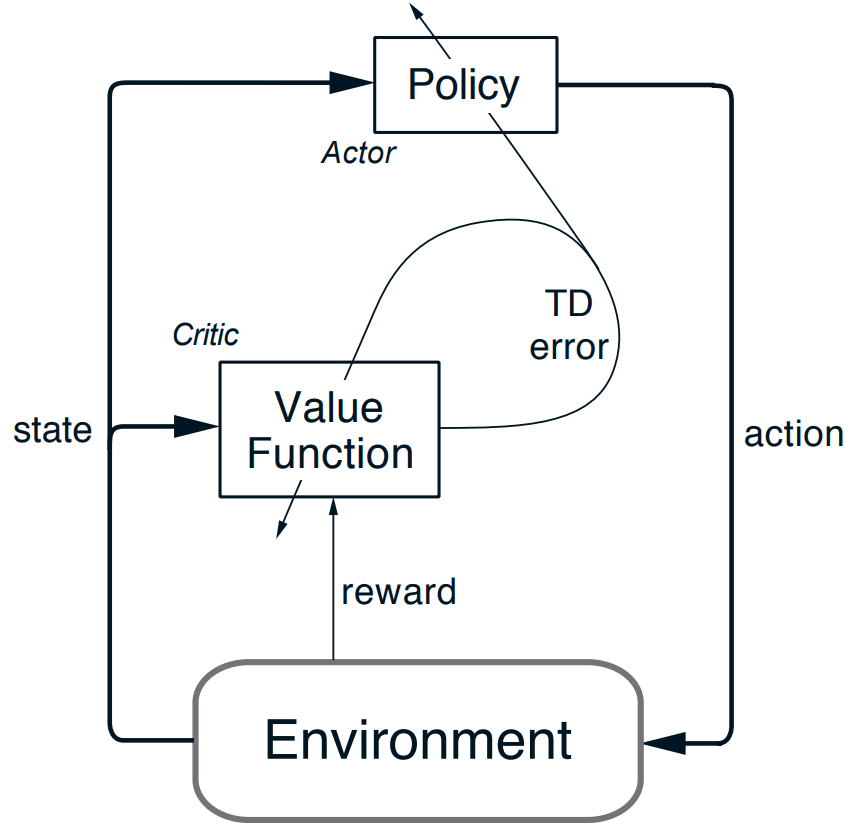}
\caption{The actor-critic architecture, depicted by \citet{rl-book}.}.
\label{fig:actor-critic}
\end{figure}

PPO's loss functions are optimised with respect to a learning rate hyperparameter $\alpha$, and the weights are updated via $\theta \leftarrow \theta - \alpha \triangledown_\theta L(\theta)$, where the loss function is evaluated in an off-policy fashion over a random batch of trajectories/experiences from the experience buffer $\mathbb{B}$, containing tuples of the form $(s_t,a_t,r_t,s_{t+1})$ \cite{off-policy-rl}. Such a process is repeated a number of times determined by the `number of epochs' hyperparameter. The `buffer size' and `batch size' hyperparameters refer to $|\mathbb{B}|$ and the number of random samples from $\mathbb{B}$ during an epoch, respectively. There are also hyperparameters defining the neural network architectures for the actor and critic approximators: specifically the `layers' and `hidden units' hyperparameters referring to the number of layers and hidden units in each layer of the neural networks respectively. In PPO, trajectories are only added to the buffer once they exceed a minimum size, known as the `time horizon'. \cite{ppo} \cite{ml-agents}

While PPO has many more hyperparameters, we have briefly discussed the important ones with respect to our study, and henceforth we utilise it mostly as a black-box tool that seeks an optimal policy $\pi^*$ in our MDP. For more information, the reader is referred to \citet{ppo} (2017), building upon earlier work by \citet{trpo} (2015). 

There are many recommended ranges and baselines for PPO's hyperparameters \cite{ml-agents}, which can be used to tune it via grid-searching techniques. Grid-searching is a basic technique used to optimise a set of $n$~parameters by sampling a subset~$P^{\subset}_i$ of values for each parameter~$i$, i.e. $P^{\subset}_i \subset P_i$ where $P_i$ is all of the possible values for the parameter~$i$, and trying every combination of the parameters~$P^{\subset}_1 \times ... \times P^{\subset}_n$, picking the combination which performs the best.

\section{Multi-Agent Reinforcement Learning}
\label{sec:marl-background}

An extension to the MDP model for environments in which $n$ agents act is the stochastic game, which is a tuple $\langle \mathbb{S}, \mathbb{A}_1, ..., \mathbb{A}_n, \mathbb{P}, \mathbb{R}_1, ...,\mathbb{R}_n \rangle$ where $\mathbb{S}$ is the discrete set of environment states, $\mathbb{A}_i$ are the actions available to the agents yielding the joint action set $\mathbb{A} = \mathbb{A}_1 \times ... \times \mathbb{A}_n$ which is a combined action for all agents, $\mathbb{P} : \mathbb{S} \times \mathbb{A} \times \mathbb{S} \to \mathbb{R}_{[0,1]}$ is the state transition probability function and $\mathbb{R}_i: \mathbb{S} \times \mathbb{A} \times \mathbb{S} \to \mathbb{R}$ are the reward functions for the agents. \cite{comprehensive-survey-marl}

One approach of multi-agent reinforcement learning (MARL) - learning an optimal policy for multiple agents - is to use Q-Learning on a stochastic game, however such an approach suffers from the curse of dimensionality from the exponential growth of the state-action space with respect to the number of agents. In addition, the environment is non-stationary, as an agent's best policy changes as the other agents' policies change. \cite{comprehensive-survey-marl}

In symmetric games, where $A_i$ and $R_i$ are the same for all agents, one can harness independent learners, where each agent learns its own policy independently, and models the other agents as part of the environment dynamics by extending $\mathbb{S} \leftarrow \mathbb{S} \times \mathbb{S}_\perp$ with information of the other agents $\mathbb{S}_\perp$. Q-Learning (and other single-agent RL algorithms such as PPO) could be applied to each agent in the environment independently \cite{marl-for-traffic-control}, potentially sharing the same Q-Table (or neural network for DQL) as the agents are symmetric. Such an approach reduces the exponential blow-up in the state-action space in the symmetric game representation, however the choice of a good $\mathbb{S}_\perp$ may be difficult, and the approach still suffers from non-stationarity.

In practise, PPO extended to multiple agents (MAPPO) has been found to be a `competitive baseline for MARL tasks' \cite{ppo-in-ma-games}, abstracting the environment as a DEC-POMDP with symmetric agents: a scalable multi-agent extension of the POMDP model. However, MAPPO can be outperformed by the far simpler independent PPO (IPPO): PPO with multiple independent learners, which is much easier to model and implement than a DEC-POMDP, making it a simple and attractive approach for MARL tasks \cite{ippo}.

\section{Motion Models}
\label{sec:motion-models}

A common and simple model for the motion of a robot with a translational and rotational velocity is the velocity motion model \cite{probablistic-robotics}. Commonly it is used as a simple model of the the motion of robots, however it can also be used to model the motion of cars for basic simulation purposes. For more advanced models, physical forces and Ackermann-type steering mechanisms \cite{ackermann-steering} can be simulated.

The velocity motion model assumes the robot's position and orientation in the world at time $t$ is specified by $(x_t,y_t,\theta_t)$ where $(x_t,y_t)$ corresponds to its global two-dimensional Cartesian position and $\theta_t$ corresponds to its global rotation in radians. The robot moves via a translational velocity $v_t$ and a rotational velocity $\omega_t$ at time $t$, where $v_t = \frac{\Delta d}{\Delta t}$ and $\omega_t = \frac{\Delta \theta}{\Delta t}$. Assuming the world is fully deterministic with discrete time and quantifying $\Delta t$ as unit time, $\Delta t = 1 \implies v_t = \Delta d$, $\omega_t = \Delta \theta$. Thus, the robot's motion is determined by its translational velocity $v_t$ corresponding to the distance the robot moves forwards (positive) or backwards (negative) during a timestep, and its change in rotation $\Delta \theta_t$ which corresponds the amount the robot's rotation changes during a timestep. Using trigonometry, after a timestep we compute the robot's new position and orientation $(x_{t+1},y_{t+1},\theta_{t+1})$ as:

\begin{equation}
\begin{aligned}
&\theta_{t+1} = \theta_t + \Delta \theta_t \pmod{2\pi} \\
&x_{t+1} = x_t + v_t \sin(\theta_{t+1}) \\
&y_{t+1} = y_t + v_t \cos(\theta_{t+1})
\end{aligned}
\end{equation}

\citet{probablistic-robotics} provide formal derivations of the velocity motion model and more advanced models.

\section{Implementation Methods}
In practice, MDPs are often simulated using a variety of software tools and frameworks. One such tool is the Unity3D engine: a general-purpose video-game engine. Another tool is the Gazebo simulator \cite{gazebo} which is often used in the field of robotics due to its highly realistic simulation of physics. For very simple MDPs, it can suffice to not use a simulation engine at all. RL algorithms and tools maintaining their operation are typically implemented using Python.

\subsection{Unity}

In this study, we harness the Unity3D engine. The Unity3D engine provides abstractions for real-time 2D and 3D modelling and a highly functional editor, enabling the creation of complex and flexible simulation environments via \Csharp\space programming without the concern of the technicalities of low-level graphics programming. In conjunction with the engine, Unity ML-Agents \cite{ml-agents} provides tools for modelling and running MDPs as real-time 2D or 3D simulations for arbitrary numbers of agents, and provides implementations of several key RL algorithms that can be easily applied to the environment, such as single-agent PPO and multi-agent independent PPO. ML-Agents also integrates with tools to aid the analysis of the application of RL algorithms, such as TensorBoard \cite{tensorflow2015-whitepaper}, which provides visualization tools to analyse the application of machine learning algorithms.

In addition to ML-Agents' attractive out-of-the-box functionality, it is also highly extensible. Harnessing its Python Low Level API, one can interface with the environment in Unity from Python, implementing their own RL algorithms in Python or integrating existing implementations. A bi-lateral communication channel between Unity and Python can be established via a Custom Side Channel, enabling the communication of messages between the two. Or, the environment can be instantiated with a set of Environment Parameters, enabling the re-use of a single generic environment rather than re-building it for different configurations. ML-Agents also has functionality to report custom metrics from the environment to TensorBoard, aiding analysis.

\subsection{Python}

RL algorithms and automation suites are often implemented using Python, due to its ease of use and rich collection of scientific computing and machine learning libraries and frameworks. NumPy \cite{numpy} is a popular numerical computation package, and PyTorch \cite{pytorch} is a popular machine learning framework, both frequently used in modern RL implementations.

\chapter{Design}
\label{ch:design}

In this chapter, we describe the overall design of our solution to the problem defined in \Cref{ch:introduction}, building on the work described in \Cref{ch:background}. We begin by describing the high-level methodology used to model our MDP and apply RL within. Then, we provide a formal theoretical construction of our MDP for both single and multiple agents, with varying degrees of dimensionality and control available to the agents. Finally, we introduce the concept of global information, and define our framework for measuring and enforcing collaboration among the agents in our environment.

\section{Methodology}

Inspired by lean practises, we follow a careful methodology of iterative improvement. We iteratively add complexity into our MDP, beginning with a very simple one, and repeating the following three phases continuously:

\begin{enumerate}
    \item Model the environment as an MDP.
    \item Apply RL to our MDP to learn policies (models) controlling the agents within.
    \item Evaluate the models obtained from step 2.
    \item With the results gained from step 3:
    \begin{itemize}
        \item If the obtained models perform sufficiently well, add complexity to the MDP by returning to step 1.
        \item If they did not perform well enough and the MDP was the cause, modify the MDP by returning to step 1. If RL was the cause, re-apply RL in a modified fashion by returning to step 2.
    \end{itemize}
\end{enumerate}

This process ensures that our MDP doesn't become too complicated too quickly, and enables us to evaluate the outcomes of specific changes. When the MDP's state is large, the curse of dimensionality may make the application of RL intractable. In addition, with the presence of many rewards, correlated rewards may arise, making them difficult to optimise. Thus, it's important that one takes a lean approach and the MDP only contains what is necessary, what is achieved by following our process. 

As our MDP grows in complexity, more powerful RL algorithms are required to learn high-performing policies within. We begin by applying Q-Learning to our simple MDP, due to its simplicity and attractive convergence and optimality guarantees \cite{comprehensive-survey-marl}. After the state becomes too large and Q-Learning becomes intractable, we harness deep learning techniques and use deep RL algorithms such as PPO, due to its robustness \cite{ppo}. We apply RL to multiple agents by harnessing independent learners due to their simplicity, applying IPPO as it `performs competitively on a range of state-of-the-art benchmark tasks' \cite{ippo}. While other MARL algorithms such as MAPPO also perform well on MARL tasks \cite{ppo-in-ma-games}, they require much more complex modelling of the environment than independent learners, and don't extend single-agent models as easily as independent learners. Thus, following our lean approach, we harness independent learners until they no longer are feasible.

The astute reader may question the move from Q-Learning to PPO, and believe it's more natural to move to DQL instead of PPO. However, both are applied to similar models of the MDP, and PPO performs better than DQL in practice for following reasons:

\begin{itemize}
    \item Q-Learning based methods (including DQL) fail on many simple problems and is poorly understood \cite{ppo}.
    \item Policy-gradient methods (including PPO) are preferred over Q-Learning based methods in stochastic environments, as Q-Learning based methods become theoretically intractable \cite{deep-rl-for-traffic-control}.
    \item PPO is robust and requires minimal hyper-parameter tuning \cite{ppo}.
\end{itemize}

\Cref{fig:high-level-method} shows a diagram of our high-level methodology.

\begin{figure}[!h]
\centering
\includegraphics[width=\textwidth]{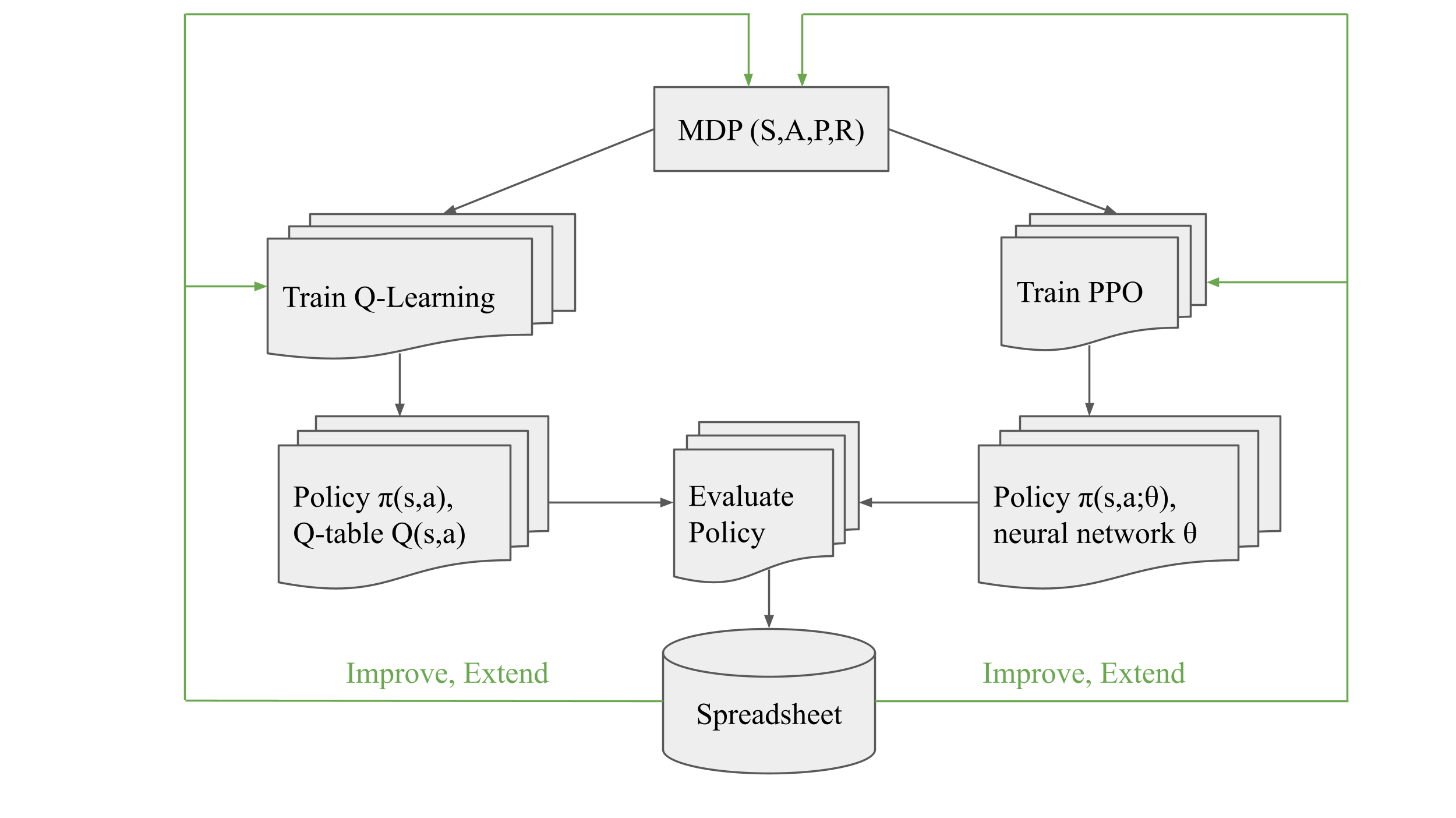}
\caption{Our high-level methodology.}
\label{fig:high-level-method}
\end{figure}

\section{MDP}

This section describes the theoretical construction of our MDP as it evolved throughout the project.

The MDP began discrete with a single agent, suitable for the application of Q-Learning. After Q-Learning became intractable in our MDP, it was extended to be suited to the application of deep RL. After deep RL was successfully applied to our MDP with single agents, it was extended with multiple independent agents.

While the number of agents in the environment was a key factor of its complexity, another factor was whether the agents had fixed or dynamic goals. An agent's goal in our environment is to park in its goal parking space, which may be either fixed and unique (known henceforth as `fixed'), or non-fixed and chosen by the agent at will (known henceforth as `dynamic'). Since dynamic goals offer more flexibility to the agent, the MDP modelling such behaviour is more complex, as the agent requires additional actions to change its goal, among other additions.

In what follows in this section, we begin with the high-level design of the MDP (\Cref{mdp-design}), and discuss our approach to encode the position of an object nearby an agent that exploits spatial symmetry (\Cref{sec:spatial-symmetry}). Then, we detail the theoretical construction of our MDP for:

\begin{enumerate}
    \item Single agents with fixed goals (\Cref{sec:single-agent-mdp});
    \item Multiple agents with fixed goals (\Cref{continuous-mdp-fixedgoals});
    \item Multiple agents with dynamic goals (\Cref{continuous-mdp-dynamicgoals}).
\end{enumerate}

With the MDP defined, we thus discuss collaborative behaviours in our MDP. In \Cref{sec:collaboration}, we introduce the various contexts in which collaboration can arise, and our methods of enforcing and measuring it.

To aid notation in what follows, let $\mathbb{Z}_n$ denote the set of integers with absolute value less than or equal to $n \in \mathbb{N}$, and $\mathbb{Z}_{(n_-,n_+)}$ denote the set of integers in the bounds $[-n_-,n_+]$ for $n_-, n_+ \in \mathbb{N}$. Also, let $\mathbb{N}_n$ denote the set of natural numbers less than or equal to $n$, and let the notation $S|n$ for some set $S$ denote the set of elements in $S$ that are divisible by $n \in \mathbb{N}$.

\subsection{High-Level Design}
\label{mdp-design}

\subsubsection{Schematics}

Firstly, we devise schematics for the car park environment (\Cref{fig:env-schematics}).

\begin{figure}[!h]
\begin{subfigure}[b]{0.2\textwidth}
     \centering
     \includegraphics[width=\textwidth]{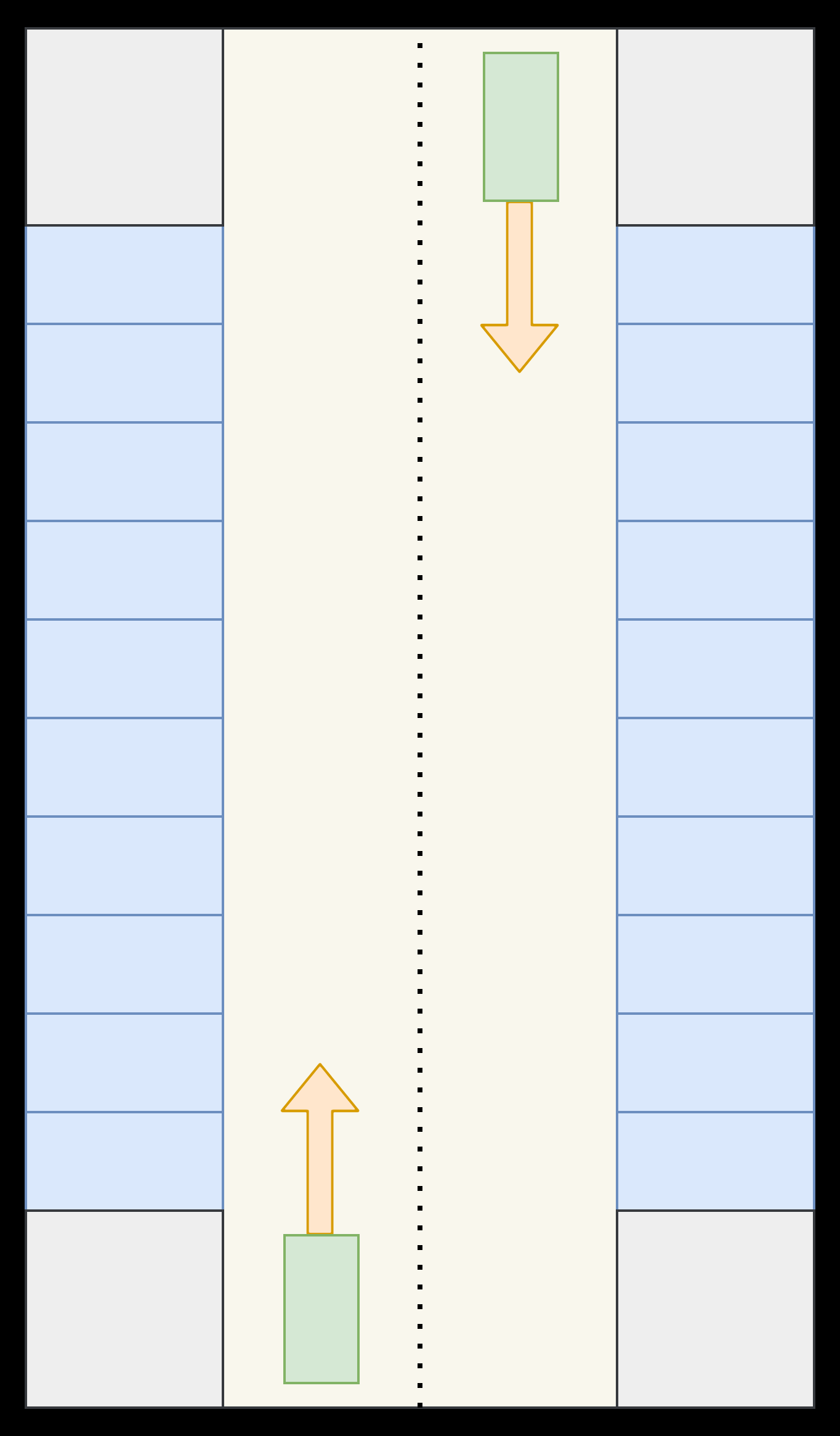}
     \caption{Normal parking environment.}
     \label{fig:normal-parking-env}
\end{subfigure}
\hfill
\begin{subfigure}[b]{0.2\textwidth}
     \centering
     \includegraphics[width=\textwidth]{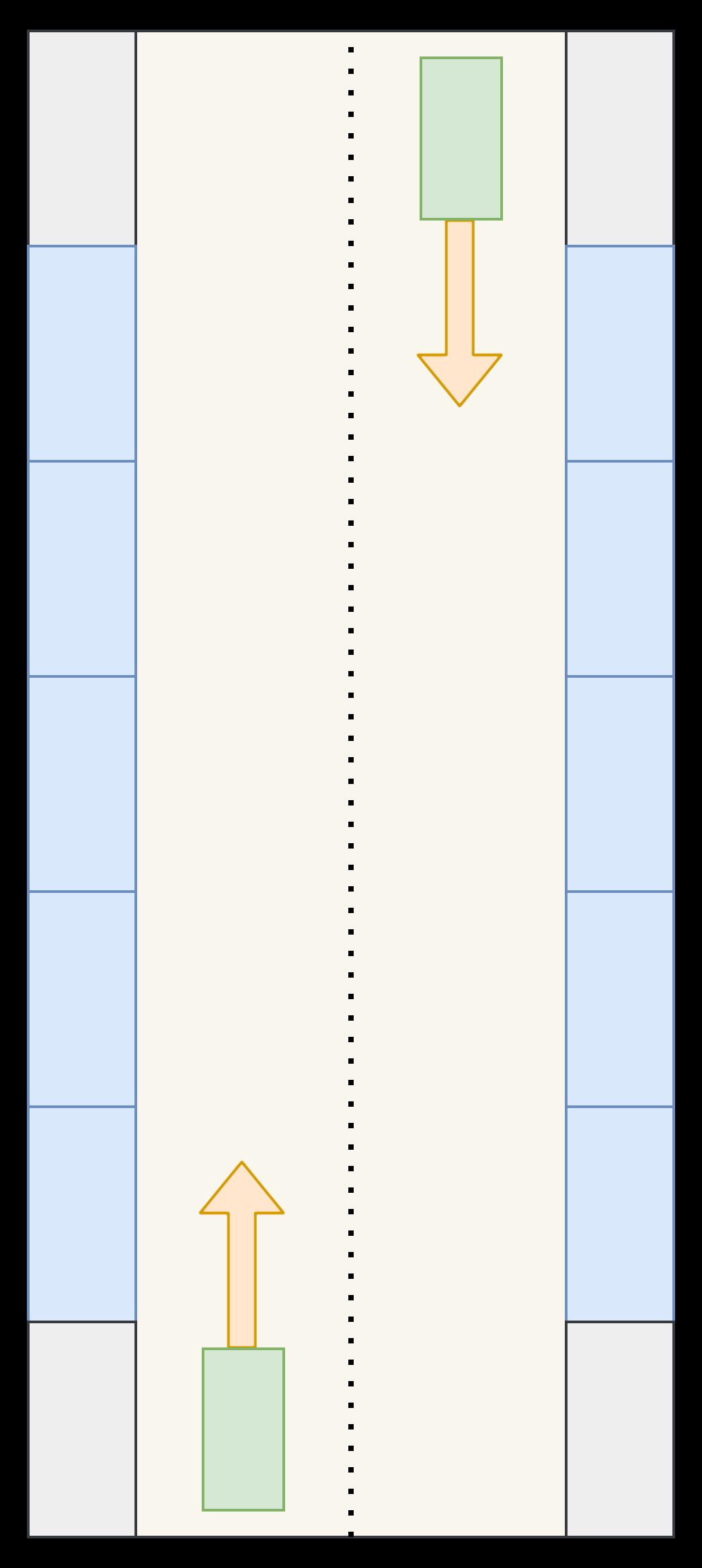}
     \caption{Parallel parking environment.}
     \label{fig:parrallel-parking-env}
\end{subfigure}
\hfill
\begin{subfigure}[b]{0.4\textwidth}
     \centering
     \includegraphics[width=\textwidth]{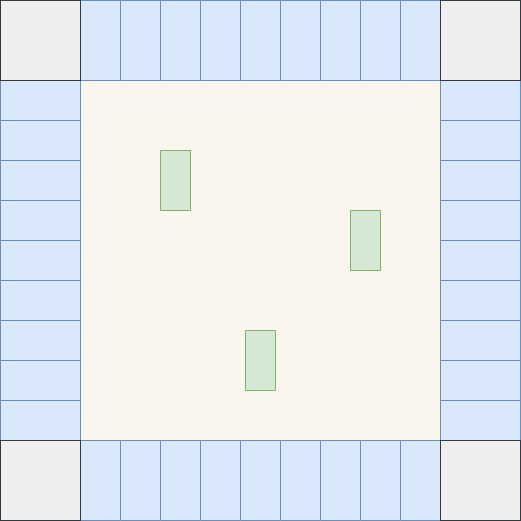}
     \caption{`Free for all' parking environment.}
     \label{fig:ffa-parking-env}
\end{subfigure}
\caption{Environment schematics.}
\label{fig:env-schematics}
\end{figure}

In the schematics, cars are shown by the green rectangles, and must conform to the road directions as shown by the orange arrows. The cars must park in the blue parking spaces, and avoid the grey walls. Schematic (a) represents a normal parking scenario with parking spaces on either side of the road, and Schematic (b) is similar to Schematic (a) but represents a parallel parking scenario. Schematic (c) is significantly different from the others in that it represents a `free-for-all' scenario in which all of the cars can park in any space around it with complete absence of the road rules, and aims to focus on the challenging aspects of multi-agent car parking such as collision avoidance and planning, rather than making the cars conform to predefined road rules. Since Schematic (c) focuses on the more difficult aspects of multi-agent car parking, where the agents have more freedom, it serves as the basis for our environment.

\subsubsection{Parked Cars}
Our environment may also contain stationary cars parked in the blue parking spaces, serving as obstacles. The number of parked cars in the environment vary its complexity, since an environment with lots parked cars has tighter parking spaces, as there is a higher chance that a parked car is neighbouring the parking space that an agent is trying to park in.

\subsubsection{Initial States}

In our environment, each agent has an independent episode with bounded length $\tau$ time-steps, since we harness the approach of independent learners with independent episodes. The initial state $s_0$ at time-step $0$ of an agent's episode is such it's spawned at a random position on the square road separating the parking spaces as shown in \Cref{fig:ffa-parking-env}. Their spawned location is such that it is beyond some threshold distance to any other obstacle, so the agent is spawned at a safe location. Upon spawning, if the agents have fixed goals then we assign the agent a random parking space as its goal, and if the agents dynamic goals we do not assign the agent a parking space and let it pick its own (by spawning it in the `exploring' state, as explained in \Cref{continuous-mdp-dynamicgoals}).

\subsubsection{Terminal States}

An agent's episode ends when it reaches a terminal state, when:
\begin{itemize}
    \item The agent crashes with any obstacle in the environment;
    \item The agent successfully parks;
    \item The time-step of the agent's episode reaches $\tau$.
\end{itemize}

If the agent successfully parked in its episode, then after the agent's episode ends and it re-spawns at a random location, the furthest parked car from all of the agents is moved into the agent's parking space, such that the parking space becomes occupied (since the agent parked in it). We pick the furthest parked car from any other agent so the disappearance of the parked car (since its location is moved) is least likely to be noticed by the other agents, since we assume they have a limited view of the cars in the environment. Moving the parked cars in this manner keeps the number of parked cars fixed in our environment.

\subsubsection{Evolution}

Over time, our environment evolves without `deadlocks'. Here, we refer to a `deadlock' as the situation whereby the environment converging to some particular configuration over time, rather than remaining randomly distributed. One such deadlock that we consider are the positions of the parked cars converging to particular areas/clusters, since we only move parked cars that are the furthest away from the other agents. However, since the positions of the agents in our environment are randomly distributed (due to the fact we spawn them randomly at time-step $0$), our environment is free from such a deadlock, the basis of \Cref{th:no-deadlock}.

\begin{theorem}[Deadlock Free Evolution]
\label{th:no-deadlock}
The parked cars in our environment do not form unbreakable `clusters', which are collections of parking spaces that are all always occupied by parked cars, as long as a free parking space outside the cluster exists.
\end{theorem}
\begin{proof}
See \Cref{appendix:proof-no-deadlock}.
\end{proof}

\Cref{th:no-deadlock} gives us confidence in the implementation of our MDP.

\subsection{Local Object Pose Encoding}
\label{sec:local-object-pose-encoding}

In the state of our MDP, it's necessary that we encode the poses of objects nearby an agent, so they can avoid them. Assume here that the objects being avoided always have the same hitbox, and assume that the positions of the cars in our environment lie on a two-dimensional Cartesian plane. We refer to a car's `global rotation' as the direction it is facing relative to the y axis. \Cref{fig:pose-localisation} shows two cars in green and grey respectively, facing in the direction of the arrows passing through their centers. In order for the green car to avoid the grey car, it must be given sufficient information to determine its pose (position and rotation) with respect to itself.

\begin{figure}[!h]
\centering
\includegraphics[width=0.5\textwidth]{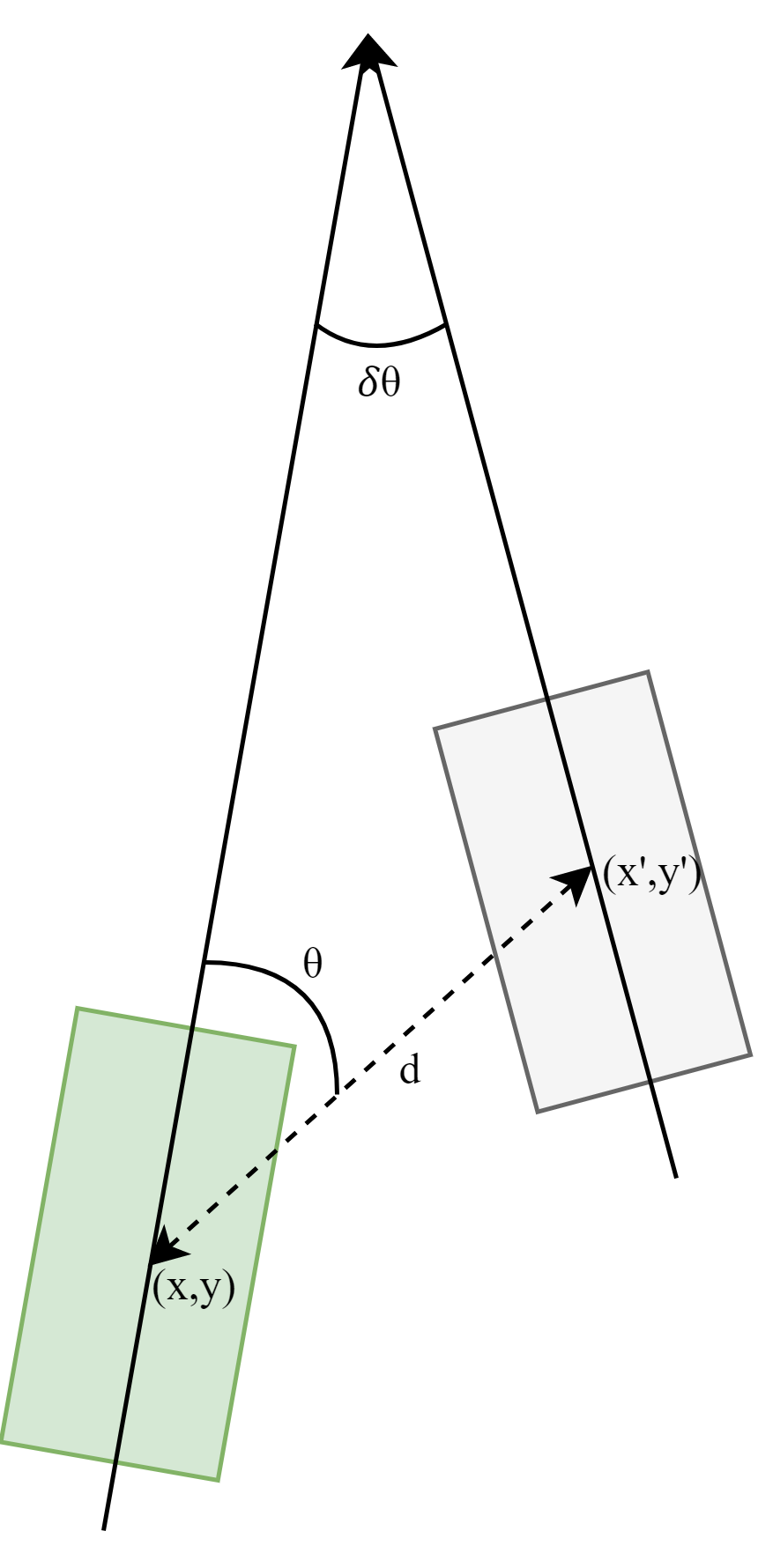}
\caption{Pose information for nearby objects.}
\label{fig:pose-localisation}
\end{figure}

\subsubsection{Naive Approach}

A naive approach to achieve this is to encode the two-dimensional Cartesian coordinates $(x,y)$ and $(x',y')$ into the state of the MDP, as well as the global rotation of both of the cars. However, such an approach isn't invariant to the global position and rotation of the cars, since if the cars are at a different position only differing by a constant $a$ and $b$ on either axis, their new positions of the green and grey car are $(x+a,y+b)$ and $(x'+a,y'+b)$ respectively, different to their original position thus a different state in the MDP. The same occurs after a constant rotation of both of the cars. This is undesirable because with respect to each other (i.e. locally), the cars are in the same position after such transformations, due to spatial symmetry. Ideally the state would be the same after such transformations.

\subsubsection{Localised Approach}
\label{sec:localised-pose}

One way to encode the position and rotation of the grey car with respect to the green car while remaining invariant to their global positions and rotations is to use the green car as a reference, using polar coordinates with respect to the green car's global position and rotation. Specifically, in our diagram, $(d,\theta)$ are the polar coordinates of the point $(x',y')$ of the center of the grey car from the center of the green car at the point $(x,y)$, with $\theta$ taken relative to the green car's global rotation, and $-\delta \theta$ is the difference in global rotation of the two cars (we negate $\delta \theta$ here as it's a counter-clockwise rotation, and we assume positive rotations are clockwise). By encoding $d$, $\theta$ and $-\delta \theta$ into the state of the MDP, the green car has sufficient information to determine the position and pose of the grey car from itself and thus avoid it. This is because the green car has sufficient information determine the position of the grey car with respect to itself, since $(x',y') = (x + d \sin(\theta),y + d \cos(\theta))$. The dimensions of the grey car's hitbox can be obtained since the two cars' hitboxes differ in orientation by $-\delta \theta$ and the grey car's hitbox is centered at $(x',y')$ which we have already shown is obtainable from $(d, \theta)$. Since we assume the objects' hitboxes have the same dimensions, we have sufficient information to use trigonometry to determine the dimensions of the grey car's hitbox. 

\subsubsection{Hitbox Assumption}

Our assumption that the nearby objects have the same hitboxes seems strong, however cars tend to have similar dimensions, and for safety we can set the hitbox to be an upper bound of the size of a car. If we need to encode the poses of another kind of objects with different hitboxes to cars (but the same among themselves), we can use the same encoding but with separate elements of the state, reserving specific elements of the state for different object types.

\subsubsection{Spatial Symmetry}
\label{sec:spatial-symmetry}

Our encoding also exploits spatial symmetry. Specifically, if both cars' positions change only by a constant $a$ and $b$ on either axis, $d$ and $\theta$ are preserved, and if both of the cars rotate by the same constant $c$, $\delta \theta$ is preserved. This is shown in \Cref{appendix:local-object-pose-encoding}. This spatial symmetry vastly improves the generality of our MDP, since we re-use states that are spatially symmetric. Since our MDP has lower dimensionality and the same policy is used for symmetric states, RL algorithms train faster on our MDP.

\subsection{Single-Agent MDP with Fixed Goals}
\label{sec:single-agent-mdp}

With an idea of the environment in place, we next model it a deterministic MDP, for single agents with fixed, unique parking spaces. Let $\langle \mathbb{S}, \mathbb{A}, \mathbb{P}, \mathbb{R} \rangle$ denote the states, actions, transitions and rewards in the MDP respectively. Since the application cost of Q-Learning is proportional to $|\mathbb{S} \times \mathbb{A}|$, we discretise $\mathbb{S}$ and $\mathbb{A}$ as much as possible, while retaining sufficient complexity for the environment to be non-trivial.

\subsubsection{Positions}

Firstly, the positions of the cars are made to lie on a discrete 2-dimensional grid with tunable granularity. To determine a lower bound for the size of the grid, we first discretise the size of the car and parking spaces (\Cref{fig:parking-space-and-car-dimensions}), from which the dimensions of the entire grid are obtained as $74$x$74$ (since its width and height are composed of an equal number of adjacent parking spaces on each side), yielding a total of $5476$ total unique $(x,y)$ positions. Importantly, the width and height of a parking space is at least 2 more than that of a car, so we can fully distinguish between a car being in a parking space and being on its border, and the car and parking space dimensions were designed such that they are in realistic proportion to one-another, yet remaining relatively concise in terms of the number of points. Our grid can be further divided into smaller squares, by dividing each square into four smaller squares, and so on. Thus, we have a granularity factor $G_p \in \mathbb{N}$, which represents the number of times we divided a square into 4 smaller squares. Thus, $\frac{1}{2^{G_p}}$ represents the width and height of a single square, yielding a divided grid with dimensions $(74 * 2^{G_p})$x$(74 * 2^{G_p})$.

\begin{figure}[!h]
\centering
\includegraphics[width=0.3\textwidth]{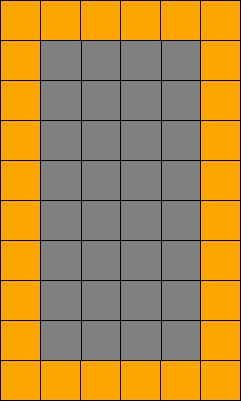}
\caption{Discretised parking space and car dimensions.}
\label{fig:parking-space-and-car-dimensions}
\end{figure}

In our MDP, the center of an agent lies on this discrete grid, and the dynamics round the agent's position to a point on this grid at each time-step.

\subsubsection{Rotations}

Similar to the positions, the rotations of the cars are discretised with a tunable granularity. A car has a global rotation $\theta \in \mathbb{N}_{G_{\theta} - 1}$, corresponding to an angle ${\frac{360}{G_{\theta}} * \theta}^{\circ}$, where $G_{\theta}$ is the rotation granularity. We require that $G_{\theta}$ divides 360, so there are only $G_{\theta}$ possible values of $\theta$ modulo $360^\circ$. For example, $G_\theta = 4$ corresponds to 4 possible angles ${0,90,180,270}^{\circ}$, and $G_\theta = 8$ corresponds to 8 possible angles in increments of $45^{\circ}$.

\subsubsection{Velocities}

Each moving car in our environment has a discrete velocity, which determines distance the car moves at each time-step. A car has a velocity $v \in \mathbb{Z}_{({-v_{max}}^-,{v_{max}}^+)}$, which causes it to move forward $G_v * v$ points at each time-step, where a positive $v$ causes forwards movement and a negative $v$ causes backwards movement (reversing) respectively. Here, $G_v \in \mathbb{N}$ is the granularity of the velocity, existing to scale the velocity in proportion to increases in granularity of the positions.

At each time-step, the dynamics of the MDP applies the velocity motion model as described in \Cref{sec:motion-models} to update the $(x,y)$ position of the car to $(x',y')$, given by:
\begin{equation}
\begin{aligned}
    &d = G_v * v\\
    &x' = x + d\sin(\theta)\\
    &y' = y + d\cos(\theta)
\end{aligned}
\end{equation}



\subsubsection{State}

With the positions, rotations and velocities of the agents in our environment defined, we thus describe the state $\mathbb{S}$ of our MDP. An agent in our MDP is a car with a fixed goal parking space, and its position in $\mathbb{S}$ is localised with respect to the position of its parking space. We localise the pose of the agent's parking space in the same way as cars in \Cref{sec:localised-pose}, and encode in $\mathbb{S}$ the agent's polar coordinates to its parking space $(d_p, \theta_p)$ and the difference in rotation between the parking space and the agent $\delta \theta_p$. We use this approach to exploit spatial symmetry, as explained in \Cref{sec:spatial-symmetry}. In addition to its parking space, $\mathbb{S}$ also contains the localised pose of the $n_{track}$ closest cars to the agent, which we refer to as the `tracked' cars. In the single-agent case, the agent only tracks parked cars.

To discretise the localised pose of an objects, we round their distances to the nearest multiple of an empirically determined granularity constant, and divide the rounded value by that granularity to get them in sequential natural number bounds. For their rotations, we round them to the nearest multiple of $G_\theta$, and divide the rounded value by $G_\theta$ to obtain a valid rotation in $\mathbb{N}_{G_\theta - 1}$.

$\mathbb{S}$ also contains the velocity $v$ of the agent, controlled via the actions, as follows.

\subsubsection{Actions}

An agent has two actions: acceleration $a \in \mathbb{Z}_{a_{max}}$ and angular velocity $\omega \in \mathbb{Z}_{\omega_{max}}$, which correspond to thrust and steering respectively. This is because an arbitrary acceleration $a$ is defined as $a = \frac{\Delta v}{\Delta t}$ changing the velocity by $\Delta v$ over time $\Delta t$, and an arbitrary angular velocity $\omega$ is defined as $\omega = \frac{\Delta \theta}{\Delta t}$, changing the rotation by $\Delta \theta$ over time $\Delta t$. Since we have that $\Delta t = 1$ over a time-step, the acceleration and angular velocity simply correspond to additive changes in velocity and global rotation respectively. Thus, with acceleration $a$, an agent's velocity is updated as:
\begin{equation}
v \leftarrow \left\{\begin{array}{ll}
    v + a & v + a \in \mathbb{Z}_{({-v_{max}}^-,{v_{max}}^+)}\\
    {v_{max}}^+ & v + a \notin \mathbb{Z}_{({-v_{max}}^-,{v_{max}}^+)} \land a > 0 \\
    {-v_{max}}^- & v + a \notin \mathbb{Z}_{({-v_{max}}^-,{v_{max}}^+)} \land a < 0
    \end{array}
    \right.
\end{equation}
which additively increases $v$ and ensures the result stays within the domain of velocities $\mathbb{Z}_{({-v_{max}}^-,{v_{max}}^+)}$ by `clamping' the result to the closest side of the domain. Similarly, with angular velocity $\omega$, an agent's global rotation $\theta$ is updated as: 
\begin{equation}
    \theta \leftarrow \theta + \omega \pmod{G_\theta}
\end{equation}
which additively increases $\theta$ and ensures the result always stays within the domain of possible rotations $\mathbb{N}_{G_\theta - 1}$ via the modulo operator.

\subsubsection{Rewards}
\label{sec:mdp-rewards}

The rewards in our MDP encourage the agent to park in its parking space while not crashing with other cars or obstacles. We also reward the agent such that it drives in a `smooth' manner. In what follows, we refer to `sparse' rewards as rewards that occur infrequently per episode (e.g. once per episode), and `dense' rewards as rewards that occur frequently per episode (e.g. at every time-step). The rewards are as follows:

\begin{itemize}
    \item \textbf{Parking}: the agent receives a large positive reward when it successfully parks in its parking space, which is determined by whether the agent is within some threshold distance to its parking space. However, at the time of parking, this reward is reduced based upon the magnitude of the agent's velocity and the agent's difference in rotation to the parking space, since it's more optimal for the agent to park with zero velocity and to be rotated parallel to its parking space. Specifically, let the reward for successfully parking be $p$, the agent's velocity be $v$, and difference on rotation between the agent and its parking space be $\delta \theta_p$. Then, the agent is rewarded $p - r_v|v| - r_\theta|\delta \theta_p|$ when it reaches its parking space, where $r_p > 0$ and $r_\theta > 0$ quantify the punishment for the agent's velocity and rotation upon parking respectively. 
    \item \textbf{Crashing}: the agent receives a large negative reward (punishment) when it crashes with an obstacle. In our environment, obstacles can be other cars (moving or parked), and walls. Each obstacle has a hitbox, and when the agent's hitbox intersects with the hitbox of an obstacle, it has collided with that obstacle.
    \item \textbf{Dense Time}: the agent receives a dense small constant punishment at each time-step, such that an optimal policy in our MDP reaches the parking space in the fewest number of steps, since an optimal policy maximises the accumulated reward per episode.
    \item \textbf{Dense Movement Towards Goal} the agent receives a dense positive or negative reward at each time-step, based upon whether it moved towards or away from its parking space respectively. This reward turns the originally sparse reward for reaching the parking space into a dense reward, providing the agent more information and thus speeding up the training process, a commonly used technique \cite{rl-collision-avoidance} \cite{ppo-self-parking}. In practise we discovered an important constraint on the size of this reward in relation to the dense time reward, as identified by sub-optimal paths to the parking space being learned. We require that the magnitude of this reward is less than the magnitude of the dense time reward, so the agent still seeks to take the least amount of steps. Specifically, let the magnitude of the dense time punishment be $t$, and the magnitude of the dense movement towards goal punishment be $m$. We must have that $m < t$, otherwise when the agent moves towards the goal we would have that $m + t > 0$, resulting in the optimal policy taking the longest path towards the goal, since it would optimise the accumulation of $m + t$. Such a sub-optimal path may be a zig-zag that's always closer to the goal after every time-step.
    \item \textbf{Smoothness}: the agent receives a dense punishment at each time-step with respect to the magnitude of its angular velocity $\omega$, such that the agent drives smoothly. Smaller angular velocities result in the agent turning less, yielding in smoother motion. In addition, we can incorporate the agent's velocity into this punishment, since larger turns are more dangerous at higher velocities. Thus, we may also punish the agent proportional to $|\omega| |v|$, making the agent reduce its velocity while turning.
\end{itemize}

\subsubsection{Rings}

In our MDP, the agent currently has no way to sense nearby objects that aren't cars, and thus avoid them. Thus, inspired by lidar, we extend the state with `rings', which provide distance of nearby obstacles to the agent. Specifically, we attach $n_r$ rings around the agent with different diameters, and the $i$'th ring reports a count $r_i$ of the number of obstacles inside it. We bound each ring to be able to report a maximum of $n_o$ obstacles inside it, thus the state is extended with $(r_1,...,r_{n_r}) \in {\mathbb{N}_{n_o}}^{n_r}$. \Cref{fig:rings} shows a car with 9 equidistant rings around it.

\begin{figure}[!h]
\centering
\includegraphics[width=0.2\textwidth]{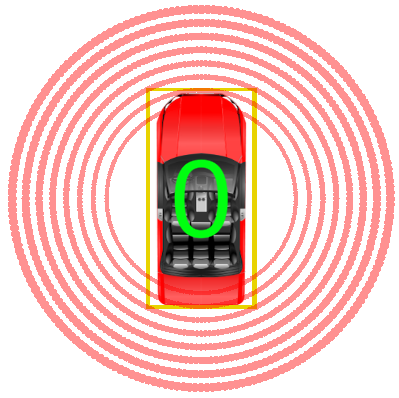}
\caption{Agent with 9 equidistant rings around it.}
\label{fig:rings}
\end{figure}

Such rings are utilised to sense nearby walls, and with appropriate values of $n_o$ and $n_r$, they're a cheaper way of encoding the poses of nearby cars. However, they do not provide full information of the pose of nearby obstacles, since the agent cannot know the angle of the obstacle within its ring, as the rings only provide distance approximations. In addition, it may be non-trivial to determine appropriate values of $n_o$, $n_r$ and the diameters of the rings such that they provide sufficient information.

One approach to tackle the problem of pose estimation with the rings is to encode historical ring states into the state of the MDP, a similar approach used by \citet{rl-collision-avoidance}. The idea behind such historical states is that when the agent moves, an obstacle may move through its rings, which is captured in the historical states, thus the agent may be able to approximate the position of the obstacle. Although promising, the dimensionality of the MDP drastically increases exponentially with such an approach, since if we encode $n_h$ historical ring states into the MDP, the rings extend the state of the MDP with $({{\mathbb{N}_{n_o}}^{n_r}})^{n_h}$.


\subsection{Multi-Agent MDP with Fixed Goals}
\label{continuous-mdp-fixedgoals}

With the MDP defined for single agents with fixed goals, we extend it to model our environment with the presence of multiple independent agents, which are moving cars as in \Cref{sec:single-agent-mdp}. As explained in \Cref{sec:marl-background}, independent learners do this by extending the state $\mathbb{S}$ of a single agent with information of the other agents $\mathbb{S}_\perp$.

\subsubsection{Nearby Agents}

$\mathbb{S}$ currently contains the localised pose of $n_{track}$ nearby parked cars. However, since our environment now contains multiple agents, $\mathbb{S}_\perp$ may contain the localised pose of nearby agents, since they are also cars. Now, $\mathbb{S}\times\mathbb{S}_\perp$ contains the localised pose of $n_{track}$ cars, which may be either other agents or parked cars, thus $\mathbb{S}\times\mathbb{S}_\perp$ may contain a mix of parked cars and other agents ($n_{track}$ in total).

\subsubsection{Shared Goals}

In addition to the agent having information of the poses of other nearby agents, it may also have information of their goals.

For each tracked car that's an agent, we may also encode into $\mathbb{S}_\perp$ the localised pose of that agent's goal parking space. Here, we do not re-localise the pose of the other agent's parking space with respect to the agent's reference frame, because the agent has access to the pose of the other agent localised to its reference frame, thus it can determine the other agent's position and rotation with respect to itself. Then, in the same way, it can determine the position of the other agent's parking space with respect to the other agent, and it can compose the two positions to yield the position of the other agent's parking space with respect to the agent's reference frame. It also has sufficient information to determine the pose of the other agent's parking space with respect to itself. This is shown formally as follows.

\begin{figure}[!h]
\centering
\includegraphics[width=0.6\textwidth]{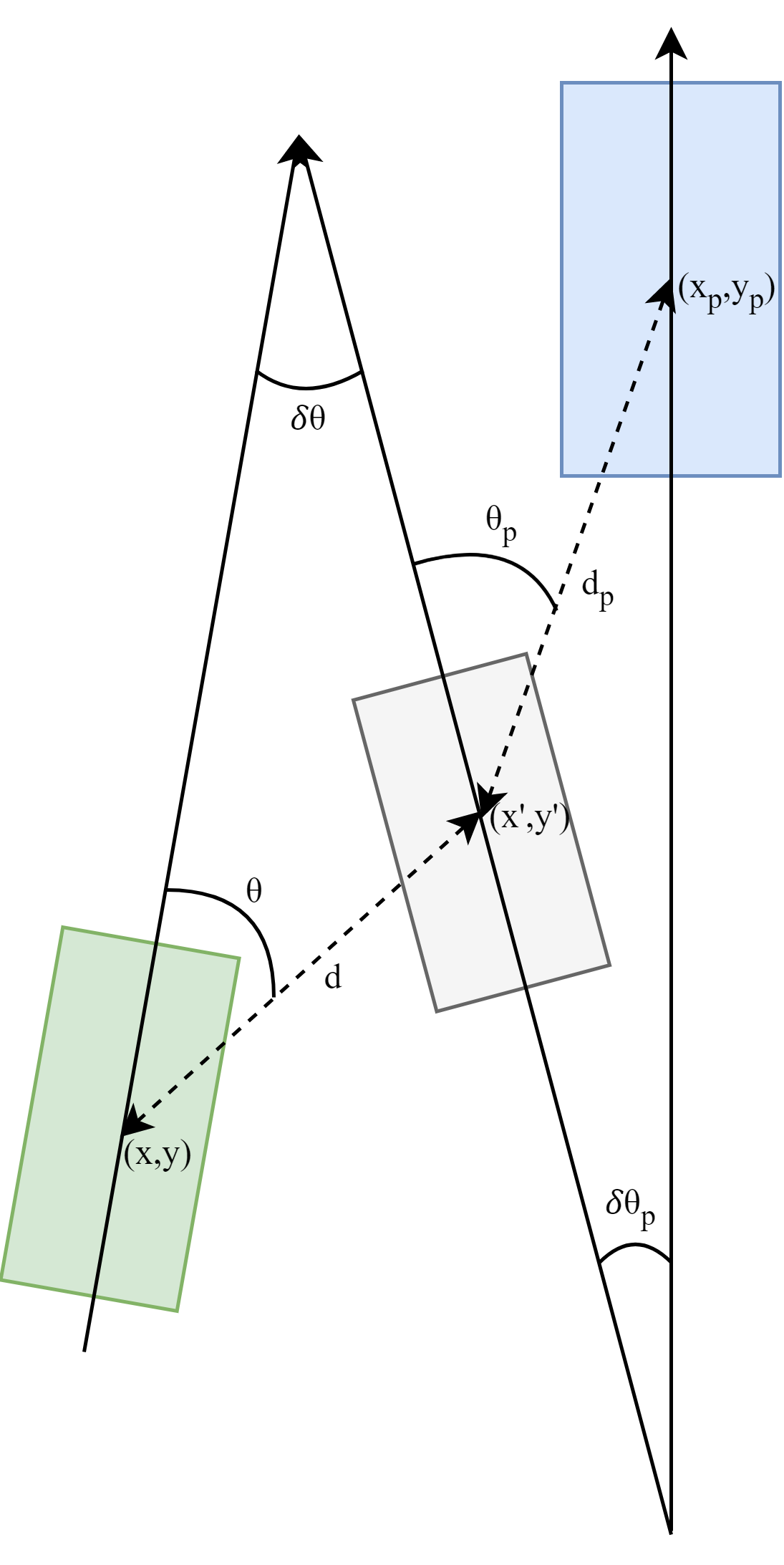}
\caption{Pose and shared goal information for a nearby agent.}
\label{fig:pose-localisation-shared-goal}
\end{figure}

\Cref{fig:pose-localisation-shared-goal} shows an agent in green with a nearby agent in grey, and the grey agent has a goal parking space in blue. As in \Cref{sec:localised-pose}, the grey agent localises the pose of its parking space as ($d_p, \theta_p,  \delta \theta_p)$, and the green agent localises the pose of the grey agent as $(d, \theta, -\delta \theta)$. With shared goals, $\mathbb{S}_\perp$ for the green agent contains ($d_p, \theta_p,  \delta \theta_p)$, which can be used to determine $(x_p, y_p)$ since as shown in \Cref{sec:localised-pose} it can compute $(x',y') = (x + d\sin(\theta), y + d\cos(\theta))$, then $(x_p, y_p)$ can be computed from $(x',y')$ as $(x_p,y_p) = (x' + d_p\sin(\theta_p),y' + d_p\cos(\theta_p))$, thus by substitution:
\begin{equation}
\begin{aligned}
    &x_p = x + d\sin(\theta) + d_p\sin(\theta_p)\\
    &y_p = y + d\cos(\theta) + d_p\cos(\theta_p)
\end{aligned}
\end{equation}
Thus, $(x_p, y_p)$ can be computed from $(x,y)$ and the information available in the state. In addition, the green agent can obtain the change in rotation of the parking space with respect itself as $-\delta \theta + \delta \theta_p$. Thus, the green agent has full information of the pose of the other agent's parking space when sharing goals.

Sharing goals in this manner may enable better path planning, since if an agent knows another agent's goal, it can try to avoid the path that the other agent is likely to take to their goal, reducing the chance of a crash.

\subsubsection{Shared Velocities}

Similar to the agents sharing their goal parking spaces, they may also share their velocities. Since parked cars don't have goals, when sharing goals we do not encode the goals of the parked cars being tracked. However, parked cars do have velocities, namely $0$ velocity, thus when sharing velocities we encode into $\mathbb{S}$ and $\mathbb{S}_\perp$ the velocity of each tracked car in $\mathbb{S}$ and $\mathbb{S}_\perp$, for parked cars and other agents respectively.

Sharing velocities in this manner may also enable better path planning.

\subsubsection{Crash Spawning}

So far, agents in our MDP spawn at random positions at the start of their episodes. However, now that our environment has multiple agents, we can spawn the agents more intelligently such that they're more likely to crash. Doing so may increase the number of of dangerous scenarios exposed to the agents during training, which may increase the safety of the learned policies.

Firstly, let $a_1$ be the agent we are spawning, and $a_2$ be another random agent which we would like $a_1$ to crash with. Let $g_1$ be $a_1$'s new parking space (picked randomly among the free parking spaces), and $g_2$ be $a_2$'s current goal parking space. \Cref{fig:spawn-crash} shows a position where we can spawn $a_1$ such that it's likely to crash with $a_2$. This is because a good policy will direct the agents towards their goals, thus following paths similar to the lines $l_1$ and $l_2$ for $a_1$ and $a_2$ respectively. Since $l_1$ and $l_2$ collide at the point shown in red, known as the `crash point', and the agents' distance to the crash point are $d_1$ and $d_2$ respectively, the two agents are likely to collide when $d_1 = d_2$, since under the same policy (which is the case as we use independent learners) they'd have similar velocities and reach the crash point at the same time.

\begin{figure}[!h]
\centering
\includegraphics[width=0.5\textwidth]{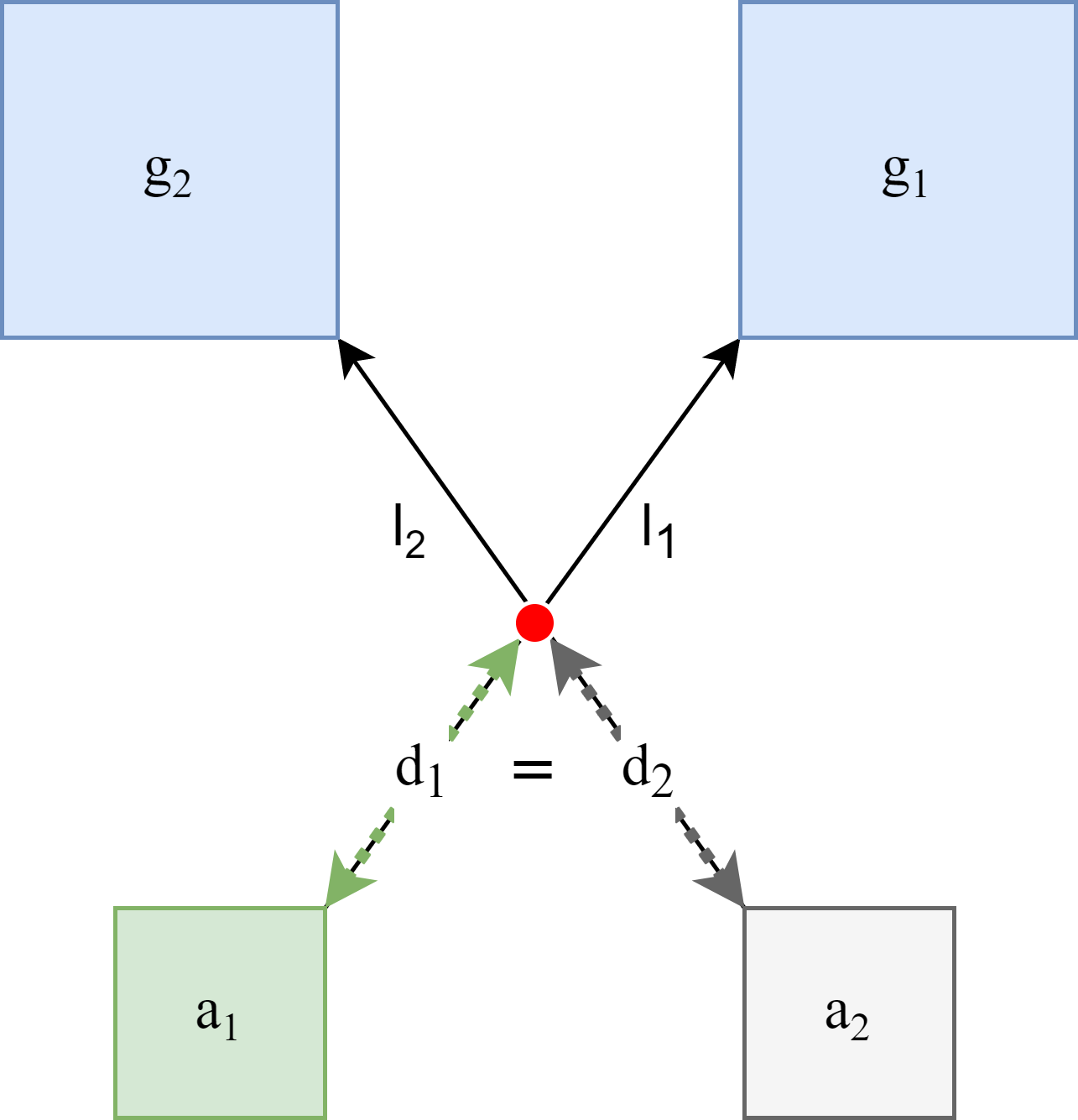}
\caption{Positioning agents such that they're likely to crash.}
\label{fig:spawn-crash}
\end{figure}

We obtain the position of $a_1$ from $a_2$, $g_2$, $d_2$ and $g_1$ in \Cref{fig:spawn-crash} via the following steps:

\begin{enumerate}
    \item Compute the line $l_2$ from the positions of $a_2$ and $g_2$.
    \item Pick the crash point as a random point along $l_2$ such that it is sufficiently far from $a_2$ and $g_2$.
    \item Compute the line $l_1$ from the positions of the crash point and $g_1$.
    \item Extend $l_1$ away from $g_1$ by a distance of $d_2$ from the crash point, to obtain the position of $a_1$.
\end{enumerate}

It should be noted that this method only works when the obtained position of $a_1$ is sufficiently far from the other agents in the environment, so we do not violate our existing random spawning method. Hence, it may take multiple tries of the method to succeed. In addition, we pick a \textit{random} point along $l_2$ and a \textit{random} other agent $a_2$ so the environment retains sufficient randomness in its configurations, a requirement for \Cref{th:no-deadlock} to hold. 

Since there may be many other ways in which the agents can crash, we only apply this method for a small proportion of the agents' initial states, so the configurations of our environment remain sufficiently general and the learned policies do not overfit to the subset of scenarios that we have identified \cite{overfitting}.


\subsection{Multi-Agent MDP with Dynamic Goals}
\label{continuous-mdp-dynamicgoals}

In the formerly described MDPs, the agents have unique fixed goals. However, fixed goals are limiting, for the following reasons:

\begin{itemize}
    \item There may be a better goal for the agent.
    \item The agent may be less likely to crash when pursing a different goal. For example, if the path towards its current goal intersects the path another agent will take towards their goal, the agents may come in close contact, and possibly crash if they do not resolve the conflict. With dynamic goals, one of the agents could change their goal to easily resolve the conflict.
    \item The agent may wish to compete for a parking space that other agents also have as their goal. With fixed goals, no two agents can have the same goal, removing the chance for competition.
    \item The agent may wish to collaborate with another agent by changing their goal if the other agent has the same goal but is closer. Since goal conflicts cannot occur with fixed goals, such collaborative behaviour cannot be exhibited.
\end{itemize}

Thus, we extend our MDP with dynamic goals to overcome such limitations.

\subsubsection{State}

Since the agent now has the ability to change its goal, it needs to know the pose of potential candidate parking spaces to change its goal to. Thus, we encode the localised pose of the $n_{space}$ nearest parking spaces to the agent in the state $\mathbb{S}$ of the MDP, which we refer to henceforth as the `tracked' parking spaces. We assign each tracked parking space a unique index $p_i \in \{1,...,n_{space}\}$ that remains fixed for the duration that the parking space is tracked, and encode the agent's current goal $g \in \mathbb{N}_{n_{space}}$ into $\mathbb{S}$. Here, $g = 0$ means the agent has no goal and is `exploring', and $g = i, i \neq 0$ means the agent's goal parking space is $p_i$.

When sharing goals, the goals of other agents $a'$ are now encoded as $g' \in \mathbb{N}_{n_{space} + 1}$ in agent $a$'s state, where $g' = 0$ means $a'$ is exploring, $g' \in \{1,...,n_{space}\}$ means $a'$ has goal parking space $p_{g'}$, and $g' = n_{space} + 1$ means $a'$ has a goal parking space that is not tracked by $a$.

We reserve $g = 0$ as an exploration state to further increase the generality of our MDP, enabling the agent to not pursue any of its tracked goals if it deems all of them as bad. We may also wish to constrain the agent's field of view such that they're unable to track parking spaces that are too far away, forcing them to explore if they cannot track any parking spaces.

\subsubsection{Actions}

To let the agent change its goal at each time-step, we add the $\delta g \in \mathbb{N}_{n_{space}}$ action to our MDP, which updates the agent's goal to their newly chosen goal (or to explore) via $g \leftarrow \delta g$. Thus, if $\delta g = 0$, the agent is set to explore, and if $\delta g \neq 0$, the agent's goal is set to the parking space $p_{\delta g}$.

\subsubsection{Losing Goals}

Since the dynamics of our MDP updates the position of an agent based upon its velocity and rotation, the parking space that an agent has as its goal may become un-tracked. We refer to such an event as a `lost goal', potentially occurring if:
\begin{itemize}
    \item The parking space is no longer part of the set of the $n_{track}$ nearest parking spaces to the agent;
    \item Another agent parks in the parking space, making it occupied and thus unavailable to the agent.
\end{itemize}

When an agent loses its goal, the dynamics of the MDP force the agent to explore by setting $g = 0$. We force the agent to explore here rather than setting the agent's goal to another parking space because one of the main purposes of dynamic goals are to let the agent choose its own parking spaces.

\subsubsection{Rewards}

Several additional rewards are added to our MDP with dynamic goals.

Firstly, we correct the dense movement towards goal reward (as described in \Cref{sec:single-agent-mdp}) such that it's only applied when the agent has a parking space as its goal, and not when the agent is exploring. Specifically, when $g = i \neq 0$, we reward the agent for moving towards its goal parking space $p_i$, and punish it for moving away from it. When exploring, no such reward or punishment is applied, letting the agent roam freely.

To help the agent progress towards a parking space, we reward the agent based upon its goal transitions $\delta g$. Specifically, let the agent's current goal be $g$ and its new goal be $g' = \delta g$. Let the notation $g \rightarrow g'$ denote a goal transition from $g$ to $g'$, and let $e$ denote exploration (when $g$ or $g' = 0$) and $p$ denote having a goal parking space (when $g$ or $g' \neq 0$). We classify the $g \rightarrow g'$ goal transitions into five categories:

\begin{itemize}
    \item \textbf{Stop Explore}: when the goal transition is $e \rightarrow p$, and the agent transitions from exploring to having a parking space as its goal.
    \item \textbf{Stop Goal}: when the goal transition is $p \rightarrow e$, and the agent transitions from having a parking space as its goal to exploring.
    \item \textbf{Change Goal}: when the goal transition is $p \rightarrow p'$, and the agent transitions from having a parking space $p$ as its goal to a different parking space $p'$ as its goal.
    \item \textbf{Continue Explore}: when the goal transition is $e \rightarrow e$, and the agent continues to explore.
    \item \textbf{Continue Goal}: when the goal transition is $p \rightarrow p$, and the agent continues to have the same parking space as its goal.
\end{itemize}

Harnessing such categories, we thus define the function $r_{\delta g}(g \rightarrow g') = r \in \mathbb{R}$, which outputs the reward $r$ applied to the agent after a goal transition $g \rightarrow g'$, where $g \rightarrow g'$ is one of our five categories. For example, $r_{\delta g}(e \rightarrow e) = -0.1$ corresponds to punishing the agent with a reward of $-0.1$ when it decides to continue to explore.

\subsubsection{Goal Transition Rewards}

We place several constraints on the goal transition rewards to aid the agent's progression.

Firstly, observe that the transition $p \rightarrow p'$ is equivalent to the transition $p \rightarrow e$ followed by $e \rightarrow p'$, thus we enforce:
\begin{equation}
\label{eq:rew-change-goal-equiv}
r_{\delta g}(p \rightarrow p') = r_{\delta g}(p \rightarrow e) + r_{\delta g}(e \rightarrow p') = r_{\delta g}(p \rightarrow e) + r_{\delta g}(e \rightarrow p)
\end{equation} This constraint ensures both ways of changing goal yield equal reward, so the agent does not prefer or avoid one way or another. Note that since both methods are now equivalent, we could remove $p \rightarrow p'$ transitions from our MDP. However doing so would require making the domain of the $\delta g$ action vary in size based upon the current state, complicating our model, thus we keep $p \rightarrow p'$ transitions.

To help the agent progress towards a parking space, we enforce $r_{\delta g}(e \rightarrow e) < 0$ so the agent avoids exploration, and enforce $r_{\delta g}(p \rightarrow p') < 0$ so the agent sticks to a goal parking space. Note that the rewards here are \textit{pessimistic} in that they're all punishments, and we could harness additional optimistic rewards to avoid exploration by enforcing $r_{\delta g}(e \rightarrow p) > 0$ and $r_{\delta g}(p \rightarrow p) > 0$. However, to avoid the problem of unconstrained positive rewards potentially yielding undesirable optimal policies due to exploitation of the rewards, as seen with the dense movement towards goal reward, we simplify our rewards and enforce $r_{\delta g}(e \rightarrow p) = r_{\delta g}(p \rightarrow p) = 0$, simplifying constraint \ref{eq:rew-change-goal-equiv} to:
\begin{equation}
\label{eq:rew-change-goal-equiv-simplified}
r_{\delta g}(p \rightarrow p') = r_{\delta g}(p \rightarrow e)
\end{equation} and introducing no additional constraints.

Thus, summarising our rewards, we have variable  $r_{\delta g}(e \rightarrow e) < 0$ and $r_{\delta g}(p \rightarrow p') < 0$, fixed $r_{\delta g}(e \rightarrow p) = r_{\delta g}(p \rightarrow p) = 0$, and variable $r_{\delta g}(p \rightarrow e)$ computed by the constraint $r_{\delta g}(p \rightarrow e) = r_{\delta g}(p \rightarrow p')$. While such rewards satisfy all of our constraints, they are heavily pessimistic, which may limit the behaviour of the agents. However, since two of the rewards are variable, the author believes they are sufficiently flexible for an initial design.

\subsection{Collaboration by Giving Way}
\label{sec:collaboration}

Since our agents can have dynamic goals, as explained in \Cref{continuous-mdp-dynamicgoals}, they may change their goals to exhibit collaborative behaviours, by giving way to the other agents. Specifically, if another agent is closer to their goal parking space, they may sacrifice their goal for the other agent because the other agent is more likely to park in it before them, saving their time and potentially relieving congestion.

Hence, in this section, we extend our MDP to enforce such collaborative behaviours, and introduce metrics that measure the extent to which the agents exhibit such collaborative behaviours.

\subsubsection{Give Way Contexts}

There are several different situations in which agents may give way to other agents, which we group as `give-way contexts'.

We define a give-way context with respect to an agent as the subset of the full state $S^*$ in the MDP in which the agent should give way to another agent. Here, we refer to the `full state' as the state of the MDP with full information, which the agent may not have.

One of the simplest contexts is when another tracked agent has the same goal parking space as the agent and is closer to it than the agent, thus the agent should concede the parking space to the other agent and pursue a different goal. Such a scenario is shown in \Cref{fig:give-way-local-same-goal}, where two agents $a_1$ and $a_2$ are shown in green and grey squares respectively, with distances $d_1$ and $d_2$ to the same goal parking space shown in the blue square, respectively. The blue parallel lines between them indicate they're tracking each-other. Since $d_1 < d_2$ in the diagram, $a_1$ should give way to $a_2$, since $a_2$ will park in the space before $a_1$, voiding $a_1$'s effort to park in it. Since $a_1$ can track $a_2$ and they have the same goal parking space, this scenario is part of the $(L,S)$ context, where $L$ means the agent to give way to is local (i.e. tracked), and $S$ means the agent to give way to has the same goal parking space. Formally, we define $(L,S)$ as the give-way context in there exists another agent $a_2$ such that $a_2$ is tracked, $g' = g \neq 0$ where $g$ and $g'$ are the goals of $a_1$ and $a_2$ respectively, and $a_2$ is closer to its parking space than $a_1$.

\begin{figure}[!h]
\centering
\includegraphics[width=0.4\textwidth]{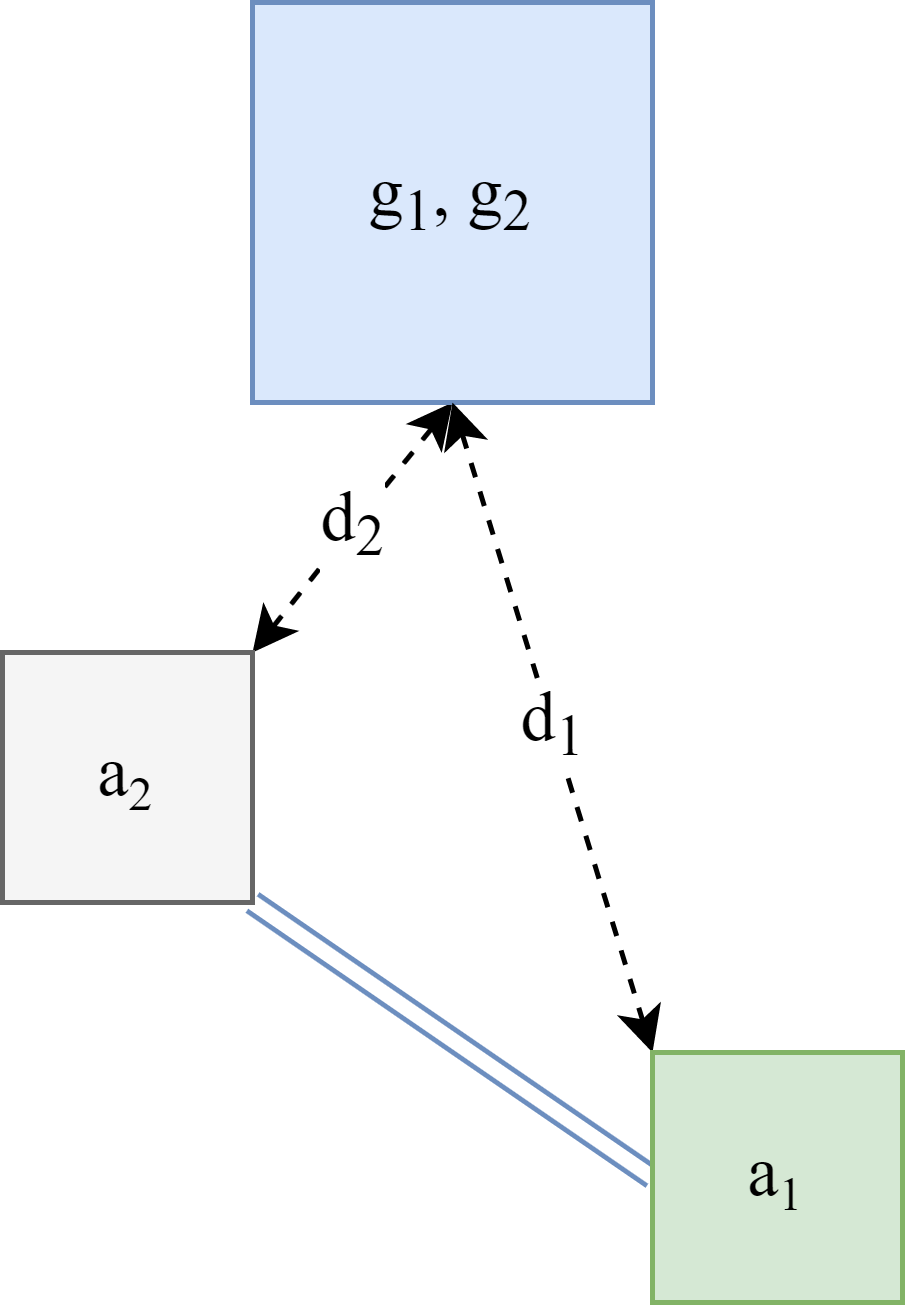}
\caption{Scenario in the $(L,S)$ give-way context, in which the agent should give way.}
\label{fig:give-way-local-same-goal}
\end{figure}

We have more contexts than just the $(L,S)$ context. Another is the $(L,A)$ context, where an agent should give way to another agent irrespective of the other agent's goal. Thus, formally, the $(L,A)$ context is all situations in which there exists another tracked agent $a_2$, where $a_2$ is closer to $a_1$'s goal than $a_1$, a weakened condition of the $(L,S)$ context thus $(L,S) \subseteq (L,A)$. \Cref{fig:give-way-local-any-goal} shows a scenario in the $(L,A)$ context in which $a_2$ is exploring but closer to $a_1$'s goal $g_1$. If the two agents travel along the paths shown in black arrows, the will eventually collide at the point in red unless $a_1$ gives way to $a_2$ and stops pursuing $g_1$. However, scenarios under the $(L,A)$ context may not always lead to crashes, as shown in \Cref{fig:give-way-local-any-goal-bad} where both agents are pursuing different goals but the paths they take to their goals do not intersect. Thus, it may not always be necessary for an agent to give way in all of the states of a give-way context (known as conforming to the context).



\begin{figure}[!h]
\begin{subfigure}[b]{0.4\textwidth}
    \centering
    \includegraphics[width=\textwidth]{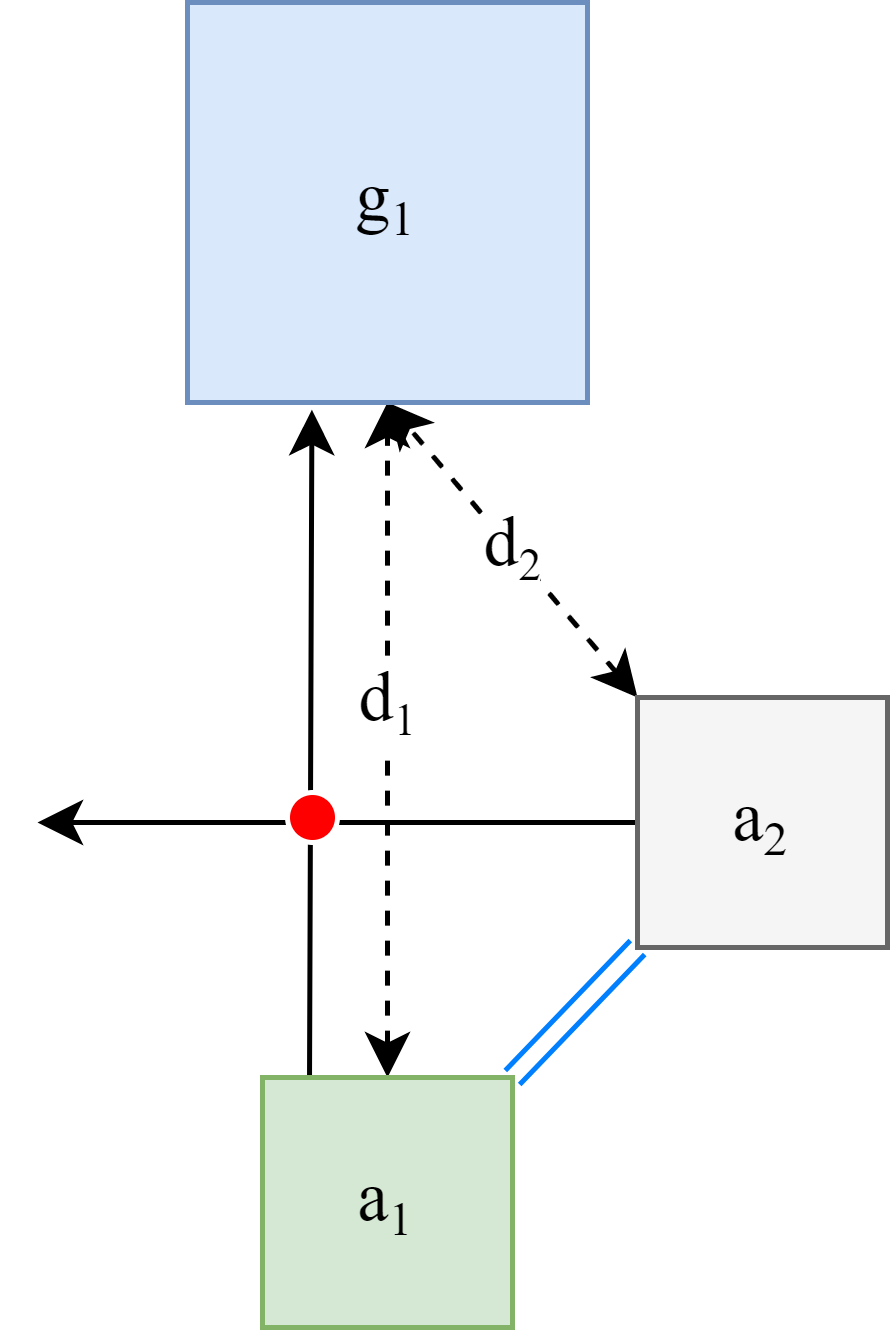}
    \caption{Scenario in which the agent should give way.}
    \label{fig:give-way-local-any-goal}
\end{subfigure}
\hfill
\begin{subfigure}[b]{0.4\textwidth}
    \centering
    \includegraphics[width=\textwidth]{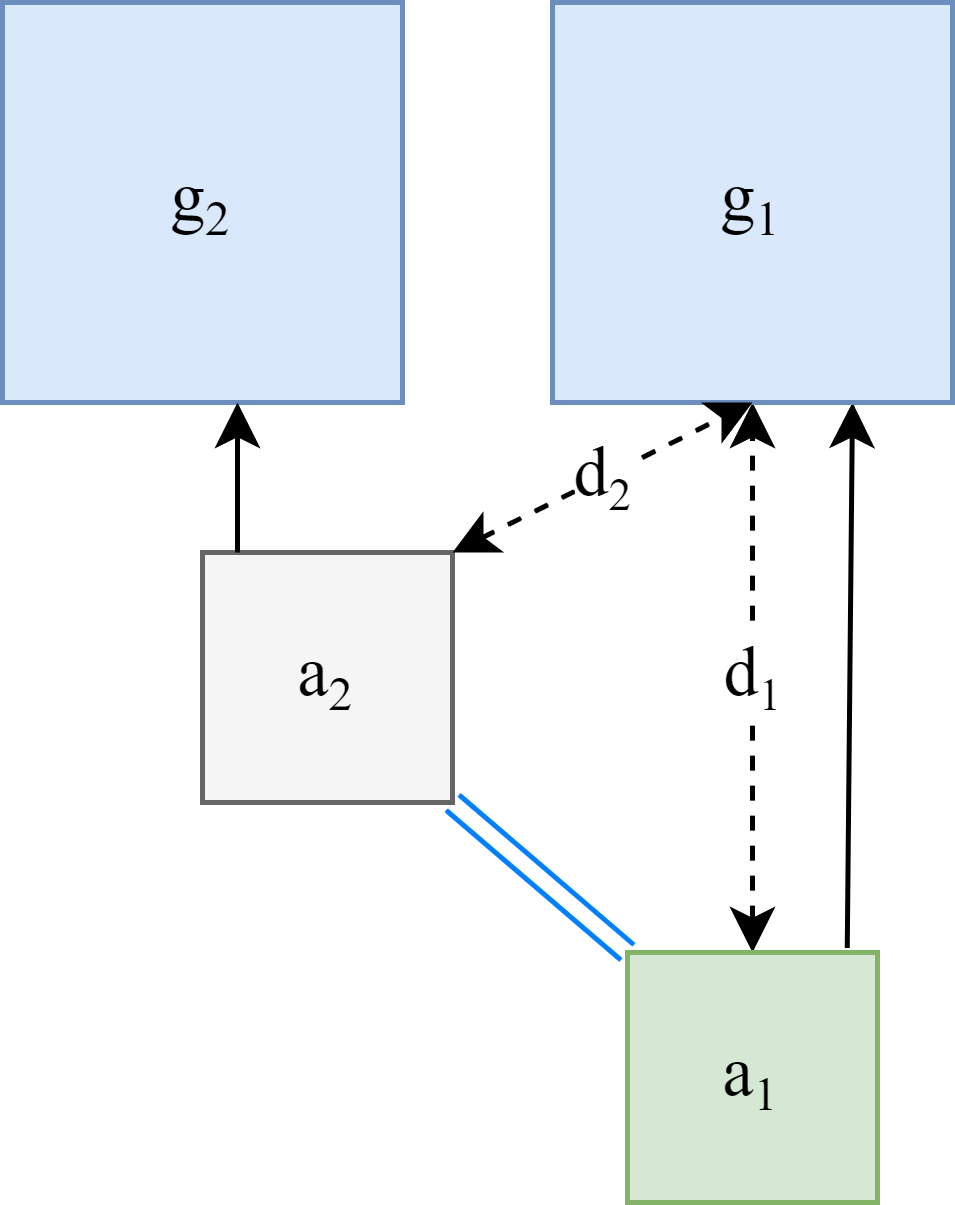}
    \caption{Scenario in which the agent should not give way.}
    \label{fig:give-way-local-any-goal-bad}
\end{subfigure}
\caption{Scenarios in the $(L,A)$ give-way context.}
\end{figure}

\subsubsection{Global Contexts and Information}

Further, we may relax the constraint in our give-way contexts that the other agent needs to be tracked (known as `any' other agent, or a `global' agent). Thus, we also have `global' give-way contexts $(G,S)$ and $(G,A)$, which respectively correspond to the $(L,S)$ and $(L,A)$ contexts with the condition that the other agent needs to be tracked relaxed. Formally, a situation is in the $(G,S)$ if another agent has the same goal parking space as the agent and is closer to it than the agent, and a situation is in the $(G,A)$ context if another agent is closer to the agent's parking space than it. Since the global contexts have weakened constraints, we have that $(L,S) \subseteq (G,S)$ and $(L,A) \subseteq (G,A)$, and we also have that $(G,S) \subseteq (G,A)$ for the same reason that $(L,S) \subseteq (L,A)$.

To distinguish the situations in the context $(G,S)$ from $(L,S)$, and the situations in the context $(G,A)$ from $(L,A)$, we further introduce two additional contexts $(G^+,S) = (G,S) \cap (L,S)$ and $(G^+,A) = (G,A) \cap (L,A)$, which are the $(G,S)$ and $(G,A)$ contexts constrained to only non-local situations, i.e. situations in which the other agent isn't tracked.

In non-local contexts, the agent does not have sufficient information to know that it should give way to the other agent, since the other agent is not tracked. \Cref{fig:give-way-global} shows a scenario where $n_{track} = 1$ and agents $a_1$ and $a_2$ are tracking different grey agents attached via the blue parallel lines. In this scenario, $a_1$ is not tracking $a_2$ since $n_{track} = 1$ (as shown by the red parallel lines), but $d_2 < d_1$, thus the scenario is part of any global or non-local contexts because $a_1$ should give way to $a_2$. 

In order for $a_1$ to know that it should give way to $a_2$, it needs to know that there is another agent closer to its goal than it, which it currently cannot see. However, notice that both $a_1$ and $a_2$ have the same parking space, thus they can use it as a `middle-man' to communicate their respective distances to each-other, from which they can determine whether they're the closest or not. Following this idea, for each tracked parking space in an agent's state, we encode into the state the `global information' of the minimum distance any agent has to the parking space, as well as the minimum distance only agents with the parking space as its goal have to the parking space. Thus, for each tracked parking space $p$, we encode:
\begin{equation}
    (\min_{a \in \Lambda}(d(a,p)), \min_{a \in \Lambda_p}(d(a,p)))
\end{equation}
into the state of the MDP, where $\Lambda$ is the set of agents, $\Lambda_p$ is the set of agents with goal parking space $p$, and $d(a,p)$ is the L2 distance between agent $a$ and parking space $p$. We encode the minimum over $\Lambda_p$ in addition to $\Lambda$ to enable the agent to distinguish between situations part of $(G^+,S)$ but not part of $(G^+,A)$, since only agents with the same goal parking space are considered in $(G^+,S)$.

Now, with global information, $a_1$ has sufficient information to know that $a_2$ is closer to its parking space, because the global information encodes $(d_2, d_2)$, thus $a_1$ has access to $d_2$ as well as its own distance $d_1$, being able to compute $d_2 < d_1$.

\begin{figure}[!h]
\centering
\includegraphics[width=0.4\textwidth]{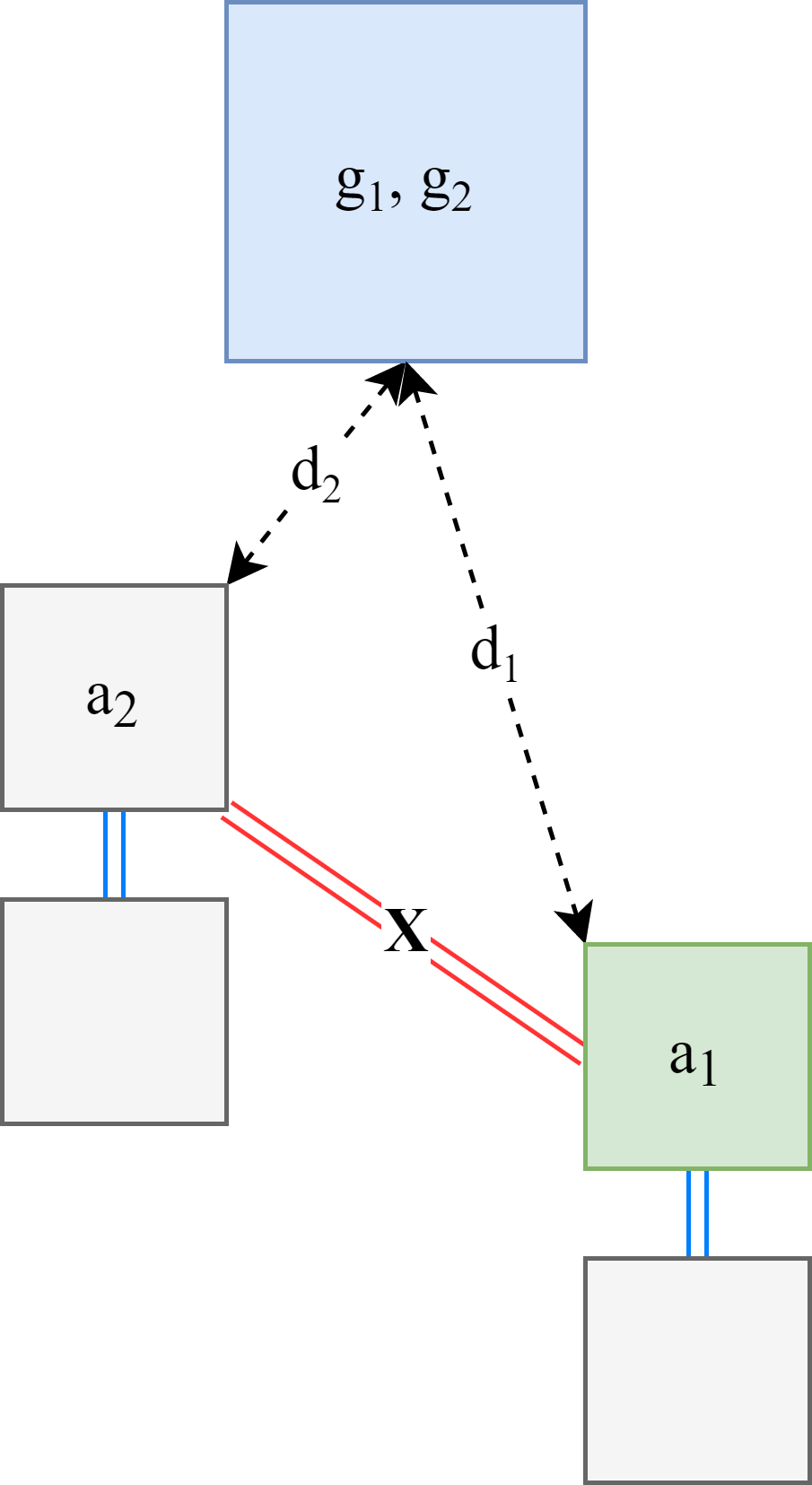}
\caption{Scenario in the $(G,S)$ give-way context, in which the agent should give way.}
\label{fig:give-way-global}
\end{figure}

\subsubsection{Context Hierarchy}
\label{sec:context-heirarchy}

One may notice significant overlap between the contexts, with many contexts being subsets of others. Naturally, this introduces a hierarchy of contexts, where we define a context $C_1$ to be `stronger' than another context $C_2$, denoted by $C_1 \succ C_2$, iff $C_2 \subseteq C_1$. This is because there are more situations that come under $C_1$ than $C_2$ (as $C_2 \subseteq C_1 \implies |C_2| \leq |C_1|$), which constrains the agents' behaviour more if the contexts are conformed to.

\Cref{fig:give-way-context-subsets} shows a graph of the subset conditions for each of our give-way contexts, visualising the context hierarchy. In the graph, there exists an edge $a \leftarrow b$ in the graph iff $a \subseteq b$. By subset transitivity, all contexts reachable from some context $a$ in the graph are a subset of $a$, and are thus weaker than $a$. Thus, as can be seen from the graph, every context is a subset of the $(G,A)$ context, thus $(G,A)$ is the strongest context.

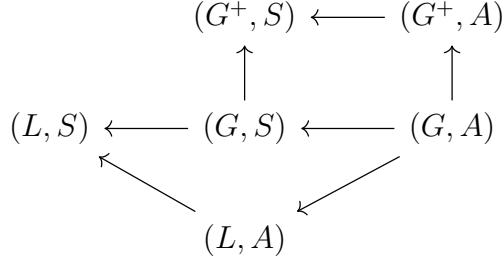
\begin{figure}[!h]
\centering
\begin{tikzcd}
        & {(G^+,S)}                   & {(G^+,A)} \arrow[l]                    \\
{(L,S)} & {(G,S)} \arrow[l] \arrow[u] & {(G,A)} \arrow[u] \arrow[l] \arrow[ld] \\
        & {(L,A)} \arrow[lu]          &                                       
\end{tikzcd}
\caption{Give-way context subsets.}
\label{fig:give-way-context-subsets}
\end{figure}

From \Cref{fig:give-way-context-subsets}, we can see that:
\begin{equation}
\label{eq:give-way-context-strength-gs}
(G,A) \succ (G,S) \succ (L,S)
\end{equation}
and
\begin{equation}
\label{eq:give-way-context-strength-la}
(G,A) \succ (L,A) \succ (L,S)
\end{equation}
However, it's not clear whether $(G,S) \succ (L,A)$ or $(L,A) \succ (G,S)$. Our study further investigates this non-trivial question.

\subsubsection{Give Way Schemes}
\label{sec:give-way-schemes}

Now we have defined the contexts in which an agent can give way to another agent, we seek to enforce giving way in such contexts in our MDP. In what follows, let $C$ be a give-way context with respect to an agent, and define $\Delta_C(s^*)$ iff the current full state $s^* \in \mathbb{S}^*$ is an element of $C$, thus $\Delta_C(s^*)$ is true iff the agent should give way in the MDP's current full state according to the context $C$.

We encourage the agent to change its goal when its current goal parking space is `bad' by first defining a function $G_{\downarrow}(s^*)$ iff the agent's current parking space is `bad' with respect to the full state of the MDP. Then, we punish the agent for continuing to pursue its parking space $p$ on the goal transition $p \rightarrow p$ when $G_{\downarrow}(s^*)$ is true, thus $G_{\downarrow}(s^*) \implies r_{\delta g}(p \rightarrow p) < 0$.

Thus, to enforce giving way in full states determined by the context $C$ (known henceforth as `enforcing' a give-way context), we enforce:
\begin{equation}
\label{eq:give-way-context-enforce}
G_{\downarrow} \equiv \Delta_C
\end{equation} defining $G_{\downarrow}$ in terms of $C$.

Using $\Cref{eq:give-way-context-enforce}$ to define our bad goal function $G_{\downarrow}$, our choice of $C$ arises several different `giving way schemes' that we can enforce. We also vary whether or not the agent has access to global information. Thus, we denote a giving way scheme by the tuple $(a,b,c)$, where $(b,c)$ is a give-way context and $a \in \{G,L\}$ determines whether the agent has global or local information respectively (here, local means the absence of global information). For example, $(G,L,S)$ corresponds to the agent having global information and enforcing the $(L,S)$ give-way context, and $(L,L,A)$ corresponds to the agent having no global information and enforcing the $(L,A)$ give-way context.

\subsubsection{Context Conformity}
\label{sec:context-conformance}

While we now have the ability to \textit{encourage} the agents to give way in certain contexts, the agents may still decide to not give way in certain contexts, or the agents may give way in contexts we did not enforce them to give way in. Thus, it's useful to measure the extent to which the agents conform to different contexts, giving us insights on the effectiveness of the various give way schemes and our methods of enforcing them.

For a context $C$ with respect an an agent, we define the agent's conformity to the context $C$ as:
\begin{equation}
\label{eq:context-conformance}
\mathds{P}(\delta g \neq g \given \Delta_C(s^*))
\end{equation}

where $\delta g\neq g$ means the agent changed its goal and $\Delta_C(s^*)$ is true if the agent should give way to another agent at the current time-step according to the context $C$. Thus, intuitively, we measure the ratio of the number of times the agent changed its goal when the give-way context $C$ said it should. If we take the mean of the conformity of the context $C$ across all of the agents, we obtain a measure of the overall conformity of the context $C$.

Particularly interesting conformities to measure may be the conformity to the give-way contexts $(G^+,S)$ and $(G^+,A)$ under the $(L,L,S)$ and $(L,L,A)$ give way schemes, since the agents do not have access to global information in such schemes, making their conformity to the non-local contexts non-trivial if it were to occur. In addition, conformity may also be used as an alternate measure of strength of different give-way contexts. If the agents show worse conformity to one context over another when both are enforced (with a punishment of the same magnitude), then one is more reluctant to be conformed to than the other, making it `stronger'. Overall, measuring conformities in this way enables us to quantify and compare the collaborative behaviours exhibited by the agents' learned policies.


\chapter{Implementation}
\label{ch:implementation}

In this chapter, we describe the implementation of our environment, following our methodology described in \Cref{ch:design}. In what follows, we refer to a `world position' as a position on the grid with a granularity of 1 (which we consider its default size), and an `MDP position' as a position on the grid with its actual granularity (specified by the environment). 

\section{MDP}

We implement our environment, as designed in \Cref{mdp-design} in the form of an MDP, using the Unity3D Engine and the Unity ML-Agents framework \cite{ml-agents}. 

In ML-Agents, each agent is an instance of the \texttt{Agent} class, which abstracts the implementation of the MDP for a single independent agent. Each \texttt{Agent} must implement the \texttt{CollectObservations} and  \texttt{OnActionRecieved} methods, which collect the state of the MDP and apply the actions received by the agent at each time-step, respectively. In addition, the dynamics of the MDP are applied in an agent's \texttt{FixedUpdate} method, and ML-Agents automatically sequences the \texttt{CollectObservations}, \texttt{OnActionRecieved} and \texttt{FixedUpdate} methods to simulate the MDP for multiple independent agents simultaneously, as depicted in \Cref{fig:mlagents-agent-lifecycle} for a single agent with respect to an MDP $\langle \mathbb{S}, \mathbb{A}, \mathbb{P}, \mathbb{R} \rangle$ (with initial state $s_0$, and actions sampled from the external policy neural network $\pi(a|s;\theta)$). In ML-Agents, an agent's state is encoded as a list of numbers $(s_1,...,s_{n_s}) \in \mathbb{S}_1 \times ... \times \mathbb{S}_{n_s}$, where $\mathbb{S}_i$ is the domain of the $i$'th element of the state (with $n_s$ in total), which may be either discrete integers for Q-Learning or real values bounded between $-1$ and $1$ for PPO. Similarly, an agent's actions are encoded as a list of natural numbers $(a_1,...,a_{n_a}) \in \mathbb{A}_1 \times ... \times \mathbb{A}_{n_a}$, where $\mathbb{A}_i$ is the domain of the $i$'th action (with $n_a$ in total), which must be equal to $\mathbb{N}_j$ for some $j$. We use the same encodings for the design of our MDP in \Cref{mdp-design}, making the implementation naturally follow from our design.

\begin{figure}[!h]
\centering
\includegraphics[width=.5\textwidth]{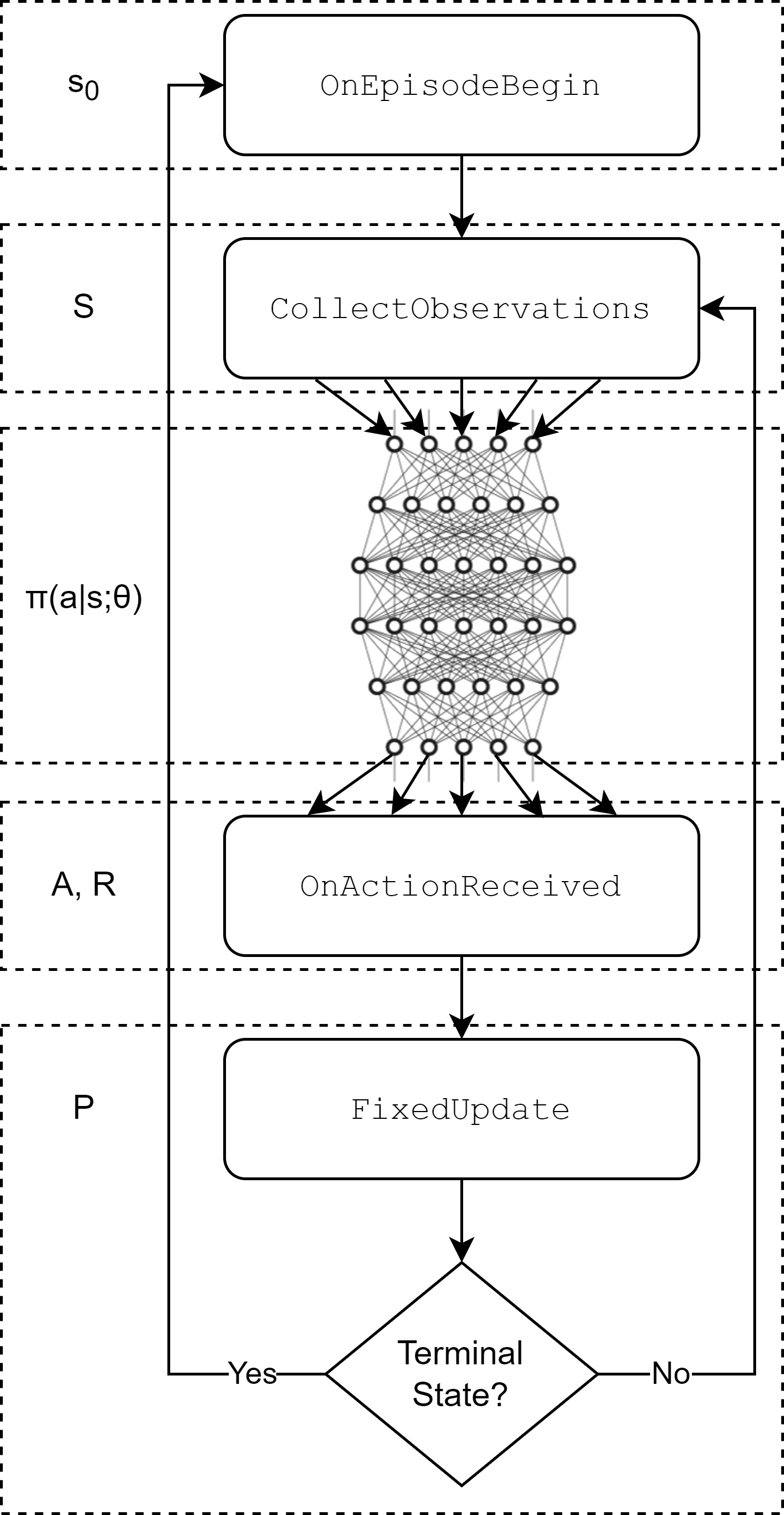}
\caption{An agent's lifecycle in ML-Agents.}
\label{fig:mlagents-agent-lifecycle}
\end{figure}

We implement our environment in a \Csharp \space Unity project, in which there is a single implementation of the \texttt{Agent} class named  \texttt{CarAgent}, implementing the MDP for each symmetric car agent in our environment. The objects in the environment are arranged in a Unity scene conforming to the scales of the objects described in \Cref{mdp-design}, with the coordinates in the Unity scene on the same scale as our world positions. \Cref{fig:env-impl} shows a screenshot of the implementation of our environment within our Unity scene, looking very similar to our originally devised schematic in \Cref{fig:ffa-parking-env}. In the implementation our environment, each agent is labelled with their index shown in green, and their goal parking space is labelled with their index. If an agent doesn't have a parking space, their ID if followed by \texttt{e}, indicating they are exploring. Cars in our environment have yellow borders if they are agents, and red borders if they are parked cars. One can also specify an additional `visualise tracked objects' parameter, which attaches a blue line to each of the tracked parking spaces for an agent, and a green line to each of their tracked cars (in \Cref{fig:env-impl}, $n_{track} = 1 = n_{space} = 1$).

\begin{figure}[!h]
\centering
\includegraphics[width=.6\textwidth]{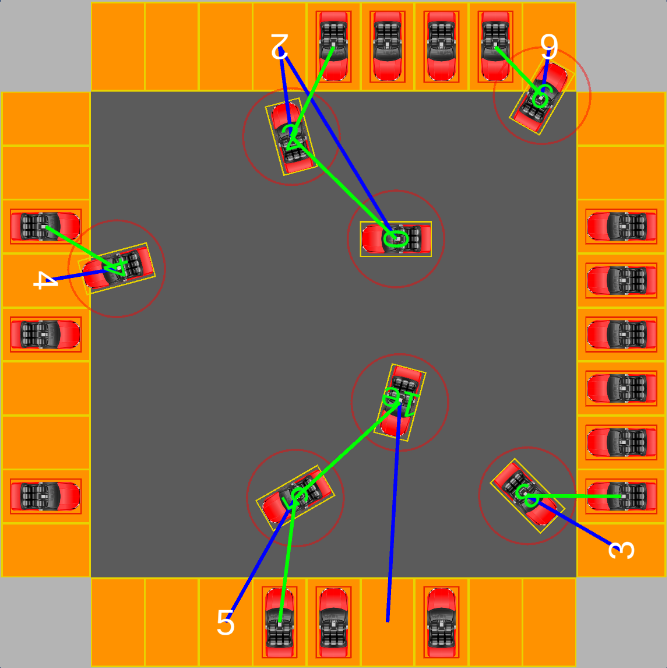}
\caption{Environment implementation in Unity3D.}
\label{fig:env-impl}
\end{figure}

After implementing the MDP in Unity, it is packaged and used for training by building the \texttt{MARL-env} scene in the Unity project as a standalone Windows executable for local runs with graphics, or a standalone Linux executable for use on a Linux server without graphics (with the `Server Build' setting enabled to disable rendering).

\subsection{Configuration}
\label{sec:env-config}
Our MDP's implementation is highly configurable, where individual elements of the state and the actions can be toggled to be included or excluded from the MDP. In addition, aspects of the environment as a whole (such as the number of agents it contains) can be configured. Such flexibility enables our environment to be re-used with different configurations without requiring modifications of the code. 

Our environment has a set of parameters defining its configuration, contained in the \texttt{EnvironmentConfig} class. A single instance of  \texttt{EnvironmentConfig} defines our environment at run-time, and this instance is managed within the \\\texttt{EnvironmentManager} class which manages the running of the environment.

At run-time, the contents of the \texttt{EnvironmentConfig} is received as a JSON serialized string, and the instance is created by serializing the JSON string. The  \texttt{EnvironmentConfig} may be sent to the environment via two external interfaces: an ML-Agents Side Channel or ML-Agents Environment Parameters, or alternatively the instance within the \texttt{EnvironmentManager} may be used directly when testing or running models within the Unity3D editor. The source of the \\\texttt{EnvironmentConfig} is determined by the by the \texttt{initMode} field of the \\\texttt{EnvironmentManager} instance (with respective values \texttt{SIDE\_CHANNEL}, \texttt{ENV\_PARAMS} or \texttt{MANUAL}), and the appropriate value is to be set before building the Unity project.

In ML-Agents, a Side Channel is a bidirectional channel whereby Python can communicate with Unity and vice-versa, and Environment Parameters are a set of global key-value constants communicated to and retrieved by Unity at run-time. Since strings can be sent in a Side Channel, we are able to serialize the  \texttt{EnvironmentConfig} to a JSON string and send it to Unity via the side channel. However, Environment Parameters can only be floats, thus when using Environment Parameters we send Unity the values of the fields of the  \texttt{EnvironmentConfig} as individual Environment Parameters, and construct the instance at run-time by injecting the field values into the class. Since Environment Parameters can only be floats, we also convert the types of the received Environment Parameters to their respective types via casting and additional transformations, since the members of  \texttt{EnvironmentConfig} may also be booleans, integers, floats and arrays.

Once the  \texttt{EnvironmentConfig} is received by the environment, the \\\texttt{EnvironmentManager} instantiates the environment based upon the configuration defined in the \texttt{EnvironmentConfig}. Specifically, it spawns the agents and parked cars based upon the number of agents and parked cars specified, and initialises the agents based upon the given parameters in the  \texttt{EnvironmentConfig}. Since the size of the state and action spaces may vary depending on the contents of the  \texttt{EnvironmentConfig}, the dimensions of the state and action spaces are computed at run-time based upon the respective ranges of the enabled elements of the state and action spaces.

\Cref{tab:env-params} and \Cref{tab:env-params-2} lists the parameters of our environment, listing the available fields of the \texttt{EnvironmentConfig} class and their respective types, default values and descriptions. Note that the descriptions use the constants defined in \Cref{ch:design}, and the values of some of the dense rewards are given as `sums'. Here, a `sum reward' is the value of the accumulation of the dense reward across every time-step in an episode, and is defined as the value $r$ in the reward  $r/\tau$ given to the agent at every time-step, where $\tau$ is the number of time-steps per episode. Sum rewards are useful for scaling purposes, since their values in the \texttt{EnvironmentConfig} are very small otherwise.

\subsection{Metrics}
\label{sec:mdp-metrics}

In addition to the implementation of our environment being highly flexible, it also records various metrics pertaining to the behaviour of the agents, useful for evaluating the agents' learned behaviour within the environment. We harness the  \texttt{StatsRecorder} instance in ML-Agents to record the metrics, which automatically aggregates them as an averages or most recent values, across fixed time-step intervals known as the summary frequency. For example, if recording a metric using the average aggregation method with a summary frequency of 10 steps, over 200 steps the  \texttt{StatsRecorder} will record the average over the first 10 steps ($1$ to $10$), then the average over the next 10 steps ($11$ to $20$), and so on until the last 10 steps ($191$ to $200$). ML-Agents outputs the recorded metrics to unique paths in TensorBoard \cite{tensorflow2015-whitepaper}, which contains various other metrics recorded by ML-Agents such as the agents' average cumulative reward over their episodes, their average episode length, and various RL-specific metrics. These metrics are then further aggregated in our automatic analysis script to obtain a set of metrics evaluating our learned models, which shall be explained further in \Cref{sec:ppo-analysis}. Column 3 in \Cref{tab:analysis-metrics} and \Cref{tab:analysis-metrics-2} lists the paths of metrics recorded into TensorBoard, and \Cref{sec:ppo-analysis} explains the additional aggregations performed on them as described by the other columns in the table.

\subsection{Optimisations}
\label{sec:mdp-optimisations}

Some optimisations are present in our MDP to improve the training time and the quality of the models learned within. As seen in \Cref{continuous-mdp-fixedgoals} where we improve the episode instantiation scheme in the multi-agent MDP by spawning the agents such that they're likely to crash, we further improve it by spawning the agents such that they're more likely to park. We do this to increase the chance an agent reaches its goal during training, in-tern reducing the time it takes for training to converge to a model that successfully parks the agents. This modification may seem to contradict the crash spawning modification, however we do not reduce the number of crash spawns - the proportion of episodes reserved for crash spawns remains the same.

Another attempted optimisation present in our MDP is increasing the size of the cars' hitboxes during training, specifically by multiplying their widths and heights by a constant factor. The idea behind this modification is that, if the agents can learn to park and avoid each-other with bigger hitboxes than their actual hitboxes, then when evaluating the model with the original smaller hitboxes, the agents would keep a threshold distance away from the cars, rather than potentially getting very close to them, increasing safety. For example, if the hitbox of a car is originally $1$ unit (arbitrary unit) but during training we increase it to $1.2$ units, and the agents learn to avoid cars during training, then they have learned to keep a threshold distance of at least $1.2 - 1 = 0.2$ units away from the other cars, safer than potentially no threshold distance otherwise.

\section{Q-Learning}

With the environment implemented, we apply RL algorithms to learn policies controlling the agents within, with our first RL algorithm being Q-Learning. Since ML-Agents does not have an implementation of Q-Learning, we implement Q-Learning ourselves, interfacing with the environment via the ML-Agents Python Low Level API. Our implementation is based upon exemplar implementations of Q-Learning with ML-Agents, as well as other open-source Q-Learning implementations. 

\subsection{Training}
\label{sec:q-learning-training}

The training portion of Q-Learning is implemented in a Python script named \texttt{iql\_train.py}, with arguments listed in \Cref{tab:iql-train-args}. Firstly, we delegate the connection to the environment and the communication of the  \texttt{EnvironmentConfig} to a Python script named \texttt{env.py}, and use the script as a dependency in \texttt{iql\_train.py}. In order to send Unity the  \texttt{EnvironmentConfig} and initialize the Q-Table with its appropriate dimensions (corresponding to the dimensions of the state space), we first serialize our  \texttt{EnvironmentConfig} in JSON form and send it to Unity via a Side Channel. Then, in Unity we compute the dimensions of the Q-Table based upon the received  \texttt{EnvironmentConfig}, and return a message to Python with the dimensions of the Q-Table, in the form \texttt{configACK:obs[$s_1, ..., s_{n_s}$]act[$a_1, ..., a_{n_a}$]} where $s_i$ are the sizes of the domain of the respective elements of the state, and $a_i$ are the sizes of the domain of the respective actions. Python recieves the \texttt{configACK} message from the Side Channel, and initializes the Q-Table as a NumPy matrix \cite{numpy} with dimensions
\begin{equation}
\label{eq:q-table-dims}
    (\prod_{i=1}^{n_s} s_i) \times (\prod_{i=1}^{n_a} a_i)
\end{equation} Here, an index $(s,a)$ in our Q-Table is such that $s$ is the index of the state in the set of possible states $\mathbb{S}$, and $a$ is the index of the action in the set of possible actions $\mathbb{A}$. Such indexes are obtained via hashing.

With the Q-Table constructed, we update the Q-Table at each time-step via the the Q-Learning Bellman equation ($\Cref{eq:q-learning}$), receiving the state and communicating the actions via interfaces provided by the ML-Agents Python Low Level API. After training finishes (after a chosen number of episodes in the environment), we store the learned model for later evaluation by de-serializing the Q-Table and storing it in an output file via Python's \texttt{pickle} module\footnote{\url{https://docs.python.org/3/library/pickle.html}}. 

In addition, during training we record metrics pertaining to the performance of the agent within our environment. We record the cumulative reward received by the agent in each episode, storing the results in an array where index~$i$ is the cumulative reward of the agent in episode~$i$. After training finishes, we also de-serialize this array of metrics and store in an output file using \texttt{pickle}, for later analysis.

It should be noted that since \texttt{\_numAgents} is a parameter of our \\\texttt{EnvironmentConfig}, our Q-Learning implementation supports multiple independent agents, actually being an implementation of Independent Q-Learning. However, we did not find a use for such multi-agent functionality, due to limitations of Q-Learning in the single-agent MDP, as explained in our results (\Cref{ch:results}).

\subsection{Batch Compute System}

Since training via \texttt{iql\_train.py} may take a long time, we delegate the running of the script to the department's Batch Compute System, enabling us to train Q-Learning on the department's compute cluster. The cluster uses the Slurm Workload Manager \cite{slurm} to manage the running of compute jobs, encapsulating the job's commands and required computational resources in a file with the \texttt{.slurm} file extension, from which jobs can be spawned via the \texttt{srun} command. We have various \texttt{.slurm} files which train Q-Learning with different parameters on the department's compute cluster. 

\subsection{Model Evaluation}

To evaluate the learned models from Q-Learning, we have two Python scripts \texttt{iql\_run.py} and \texttt{plot.py}, with arguments listed in \Cref{tab:iql-run-args} and \Cref{tab:plot-args}, respectively. \texttt{iql\_run.py} runs the agents in the environment with their learned policy, and \texttt{plot.py} plots the cumulative reward of the agent per episode during training.

To implement \texttt{iql\_run.py}, we first read the stored Q-Table from training using \texttt{pickle}, then harness \texttt{env.py} to connect and instantiate the environment as in \Cref{sec:q-learning-training}. Then, we retrieve the state and send actions to the enviornment via the ML-Agents Python Low Level API, where the actions are obtained from the Q-Table. After running \texttt{iql\_run.py}, the built Unity environment opens in a window, providing a graphical view of the environment with the agents controlled by the policy learned by Q-Leanring, enabling visual analysis of their learned behaviours.

\texttt{plot.py} is implemented by first reading the stored array of cumulative rewards per episode from training using \texttt{pickle}, and then plotting a graph of the agents' cumulative reward per episode on the y and x axes respectively, using the \texttt{matplotlib} Python library \cite{matplotlib}. Such a plot enables one to determine if Q-Learning had converged and whether the converged policy is optimal, by observing whether the cumulative reward reached a horizontal asymptote and the cumulative reward of that asymptote, respectively.

With the tools implemented to evaluate the learned policies from Q-Learning, they are used in conjunction with \texttt{iql\_train.py} to train and evaluate Q-Learning in our environment. \Cref{fig:q-learning-impl} shows a diagram of the work-flow, with the users' entry point depicted by the circle.

\begin{figure}[!h]
\centering
\includegraphics[width=\textwidth]{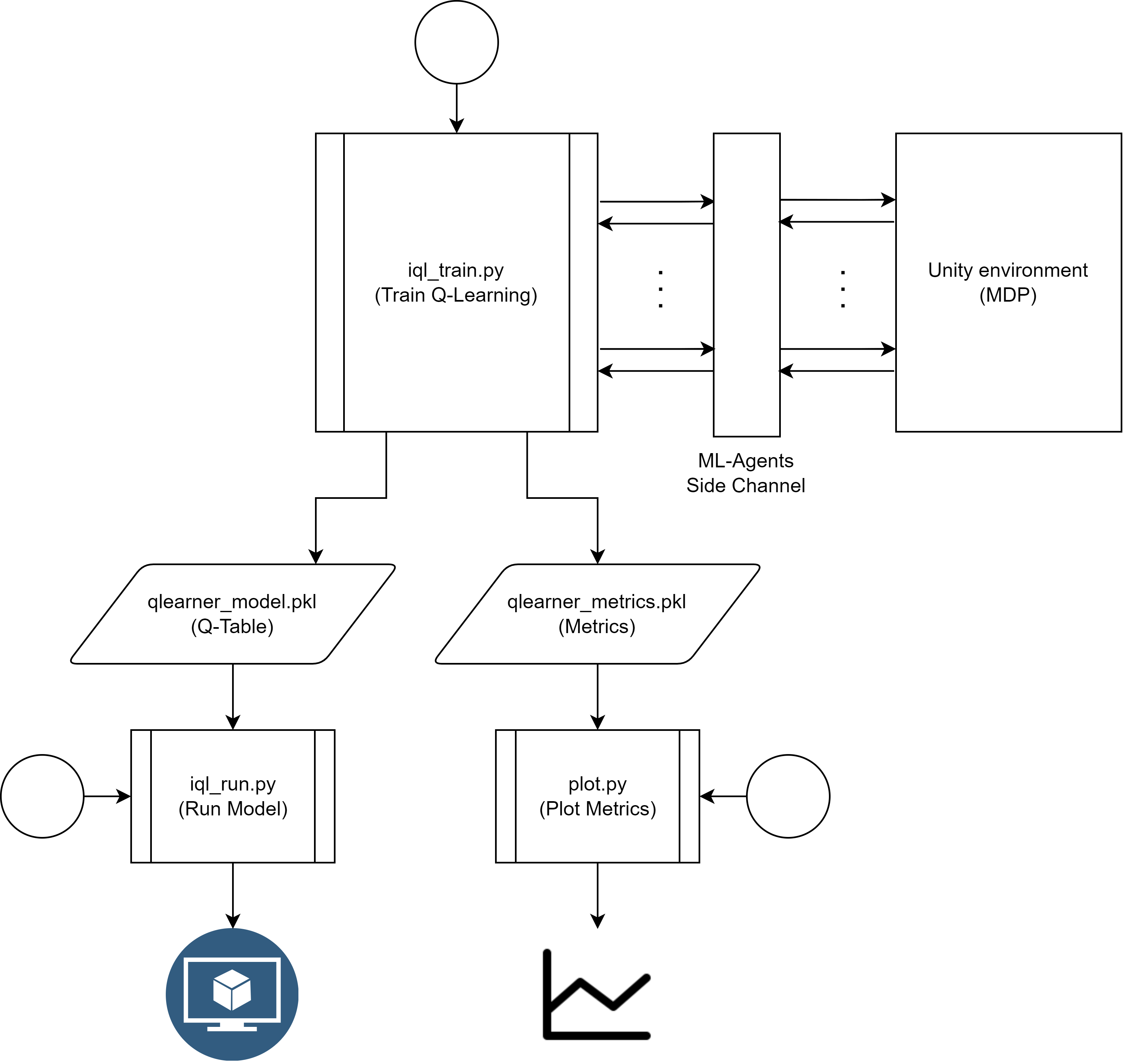}
\caption{Q-Learning implementation.}
\label{fig:q-learning-impl}
\end{figure}

\subsection{Environment Configurations}

In our Q-Learning experiments, the \texttt{EnvironmentConfig} uses the default values shown in \Cref{tab:env-params} and \Cref{tab:env-params-2}. We do not encode the agents' distance to their goal in the state of the MDP in Q-Learning experiments because it has a large range of values, making the dimensionality of the state space intractably large. Instead, we use the rings for nearby distance information. When using rings in our experiments we have \texttt{\_obsRings} true in our \texttt{EnvironmentConfig}, with non-empty \texttt{ringDiams}.

\section{PPO}

Contrary to Q-Learning, PPO has a robust existing implementation provided by ML-Agents, simplifying training. However, to overcome issues pertaining to manual analysis and tuning, we have a tool that aggregates the metrics recorded by Unity into a CSV file, and a tool that automatically spawns training jobs to grid-search parameters, respectively.

\subsection{Training}
\label{sec:ppo-train}

To train independent PPO on our environment, we use the \texttt{mlagents-learn} command in ML-Agents. In ML-Agents, the values of PPO's hyperparameters as well as the Environment Parameters are contained in a YAML file, which is passed to \texttt{mlagents-learn} to define PPO's training parameters and the Environment Parameters sent to Unity (which are subsequently injected into the  \texttt{EnvironmentConfig} as explained in \Cref{sec:env-config}). Thus, this YAML configuration file contains the metadata pertaining to an application of PPO in our environment. In addition to the path of the YAML file, \texttt{mlagents-learn} has three additional parameters: 
\begin{itemize}
    \item \texttt{-{}-env}: the path to the Unity build of the environment;
    \item \texttt{-{}-run-id}: an identifier for the specific run and the path which the results are output to;
    \item \texttt{-{}-base-port}: the port that ML-Agents uses during training.
\end{itemize}

Since we seek to automate the evaluation of our models, we perform a modification to ML-Agents' implementation of PPO, enabling it to also evaluate our models. Out of a total of $T$ time-steps in our environment used to train the agents,  we reserve the first $0.8T$ time-steps for training the model, and the last $0.2T$ time-steps for evaluating the learned model. In order to ensure that the model doesn't change during the period in which it's being evaluated, we set the learning rate equal to $0$ during the evaluation phase, and decay the learning rate as it was originally decayed over the first $0.8T$ time-steps instead of the full $T$ time-steps. The learned model (policy) doesn't change with a learning rate $\alpha$ of $0$ because ML-Agents' implementation of PPO uses the Adam optimiser \cite{adam} to optimise its loss function, which updates the parameters $\theta$ of its neural networks defining the policy via $\theta \leftarrow \theta - \alpha\Delta$ for some term $\Delta$, thus with $\alpha = 0$ the update becomes $\theta \leftarrow \theta$, leaving $\theta$ unchanged. Then, since we record metrics in our environment during training, we aggregate such metrics over the evaluation period, obtaining metrics evaluating the final learned model from PPO.

As with Q-Learning, we also delegate the training of PPO to the department's compute cluster via \texttt{.slurm} files which execute the \texttt{mlagents-learn} command within.

\subsection{Analysis}
\label{sec:ppo-analysis}

As explained in \Cref{sec:mdp-metrics}, various metrics pertaining to the behaviour of the agents are recorded to TensorBoard. However, TensorBoard only provides a graph view of each individual metric over each summary period, which is time-consuming to analyse, especially when comparing multiple models. Hence, we implement a Python script \texttt{analysis.py} which aggregates the recorded metrics in TensorBoard and outputs them to a CSV file, for faster analysis.

After training a model with PPO, its results are output to the \texttt{results/<id>} directory, where \texttt{<id>} is the \texttt{-{}-run-id} of the run. In the results directory, there is a \texttt{.tfevents} file containing the recorded metrics and metadata pertaining to the run, and a \texttt{.onnx} file containing the neural networks for the learned PPO model. To aggregate the metrics, we read the them from the \texttt{.tfevents} file using a TensorBoard \texttt{EventAccumulator}, and aggregate them over the evaluation period, which as explained in \Cref{sec:ppo-train} is the last $20\%$ of the total time-steps during training. We aggregate the metrics using two different functions:

\begin{itemize}
    \item \texttt{mean\_metric\_period} which takes the mean of the metric over the evaluation period (or, metric period). This function is used for metrics that are recorded as averages during training, since taking the mean of the means over the metric period yields the mean over the metric period.
    \item \texttt{last\_minus\_start\_metricperiod} which subtracts the value of the metric at the end of the metric period by the value at the start of the metric period. This function is used for metrics that are recorded as running sums during training, since subtracting the sum at the end of the metric period by the sum at the start yields the sum between the metric period.
    \item \texttt{context\_conformance} which is used to compute the conformity for different give-way contexts. We record the number of times an agent gave way in a context in a TensorBoard metric with path \texttt{Metrics/<context>\_PostiveCount} where \texttt{<context>} is a descriptor of the context, and the number of times they should have gave way in a TensorBoard metric with path \\\texttt{Metrics/<context>\_TotalCount}. Then, we use \\\texttt{last\_minus\_start\_metricperiod} on both to obtain the totals over the metric period, then divide the former by the latter to obtain the conformity of the agents to the context in the metric period.
    \item \texttt{per\_eps} which computes the total of the metric averaged per episode by computing the total over the metric period via \\\texttt{last\_minus\_start\_metricperiod} then dividing it by the number of episodes the agents were evaluated over.
\end{itemize}

We also compute additional metrics as functions of the aggregated metrics and the metadata of the PPO run. In addition, environment parameters of the run can be recorded as metrics, enabling comparison of the metrics between models with different values of an environment parameter.

The output of \texttt{analysis.py} is a CSV file where each row corresponds to the metrics for an individual model, enabling easy comparison of the metrics between different models. We extract the metrics for multiple models by recursively seeking all \texttt{.tfevents} files in a given subdirectory, and extracting the metrics for each model in-tern. \Cref{tab:analysis-metrics} and \Cref{tab:analysis-metrics-2} lists each metric output by \texttt{analysis.py} and their corresponding paths in TensorBoard and computation functions, and \Cref{tab:analysis-args} lists the arguments of the \texttt{analysis.py} script.

\subsection{Grid-Searching}

Since PPO and our environment have many parameters, it's time-consuming to manually grid-search them. This is because we can only train PPO with one combination of the parameters at once, thus with lots of combinations we must manually set the parameters and run PPO for each combination. Hence, we automate the process with a Python script \texttt{gridsearch\_ppo.py}, which grid-searches a set of parameters in the PPO YAML configuration file, via the following steps:

\begin{enumerate}
    \item Compute all of the possible combinations of the parameters we are grid-searching over.
    \item Create an individual PPO YAML configuration file for each combination of the parameters, injecting the values of each combination into a copy of a base YAML file.
    \item Create individual \texttt{.slurm} files for each combination of the parameters, wherein \texttt{mlagents-learn} is called with the respective YAML files from step 2.
    \item Run the jobs on the department's compute cluster to train PPO for every combination of the parameters, by calling \texttt{srun} with each of the \texttt{.slurm} files created in step 3.
    \item Automatically run \texttt{analysis.py} on all of the models after they have finished training, by running \texttt{analysis.py} after \texttt{mlagents-learn} in the \texttt{.slurm} files.
\end{enumerate}

To enable the output models to be compared against the values of the grid-searched parameters, step 5 passes into \texttt{analysis.py} each parameter that was grid-searched over, so they appear as columns in the output CSV.

Figure \Cref{fig:grid-search-impl} shows a diagram of the script's work-flow, and \Cref{tab:gridsearch-arguments} lists its arguments.

\begin{figure}[!h]
\centering
\includegraphics[width=\textwidth]{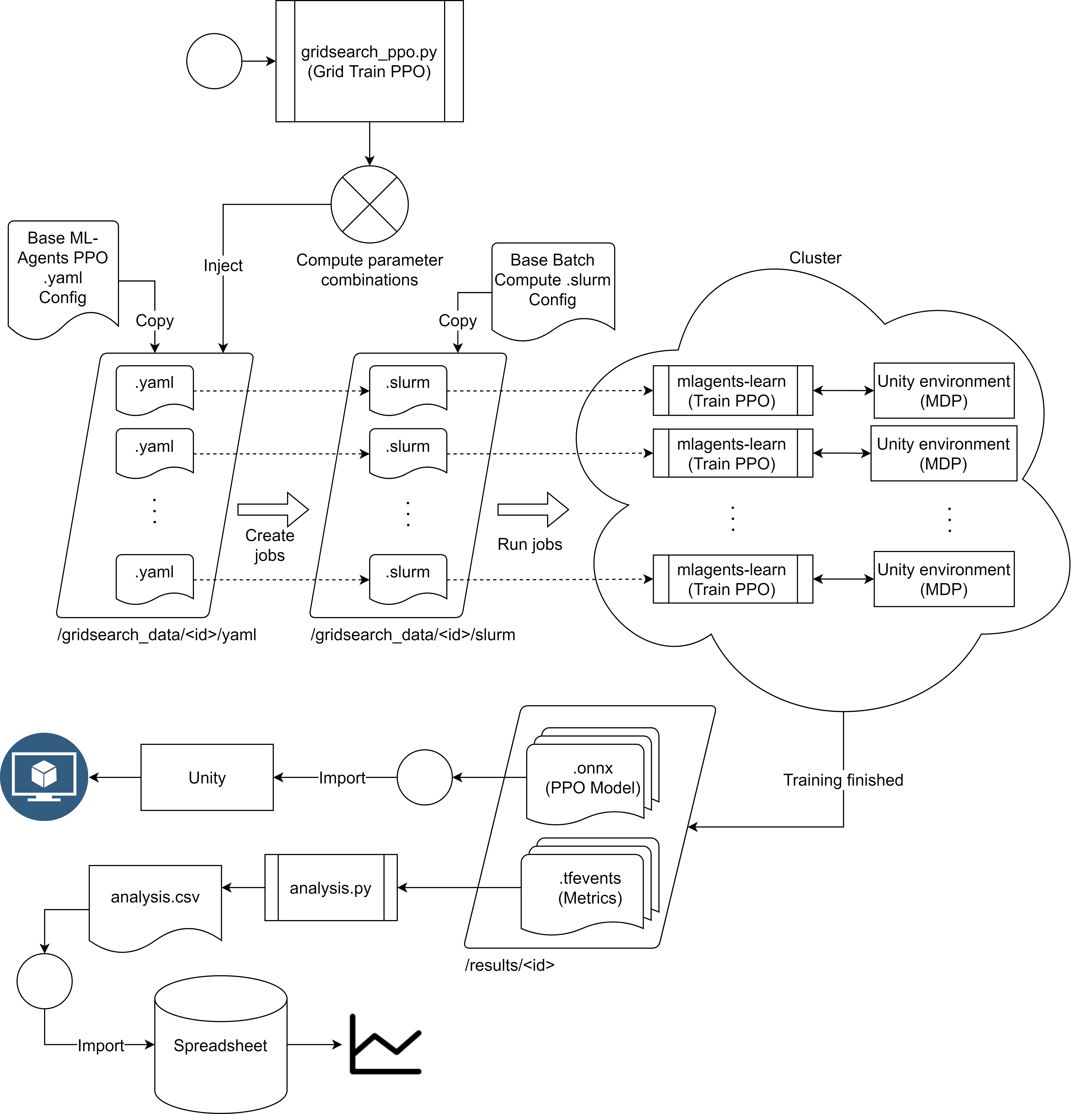}
\caption{PPO grid-searching implementation.}
\label{fig:grid-search-impl}
\end{figure}

\subsubsection{Analysis Mutex Lock}

It should be noted that our implementation of \texttt{gridsearch\_ppo.py} only runs \texttt{analysis.py} after all of the models have finished training, despite it being called after every call to \texttt{mlagents-learn} in step 5. We only run \texttt{analysis.py} once so the work isn't repeated, and to avoid potential race conditions corrupting the output CSV file when multiple runs of \texttt{analysis.py} occur at once. This is achieved by checking the number of models that had finished training at the start of \texttt{analysis.py}, and if not all of the models have finished training then we abort the script. Otherwise, we create a mutex lock file indiciating the analysis is being performed, and other calls of \texttt{analysis.py} abort their execution if that file exists, preventing race-conditions on the execution of \texttt{analysis.py}.

\subsubsection{Redundancy Tests}
We also harness \texttt{gridsearch\_ppo.py} to perform redundancy tests, whereby one trains their models multiple times and averages the results to increase their accuracy. Such a technique is paritcularly useful when runs often fail (known as flaking), with multiple runs having a higher chance of at least one succeeding. To perform a redundancy test with $n$ repetitions, the \texttt{try} parameter is added to the list of environment parameters in the base PPO YAML configuration file, and the call to \texttt{gridsearch\_ppo.py} contains the argument \\\texttt{environment\_parameters.try=i[$1,...,n$]} to grid-search over the \texttt{try} parameter with values $1$ to $n$, for each respective repitition.

\subsection{Environment Configurations}
\label{sec:env-configs}

In our PPO experiments with fixed goals, the \texttt{EnvironmentConfig} uses the values listed in \Cref{tab:ppo-fixed-goals-env-config}, with the default values listed in \Cref{tab:env-params} and \Cref{tab:env-params-2} for the remaining parameters. For PPO experiments with dynamic goals, we use the values listed in \Cref{tab:ppo-dynamic-goals-env-config}, with the remaining values being the same as the configuration for PPO with fixed goals. We arrived at such configurations in light of our results (\Cref{ch:results}), and the commonalities among their experiments.

Some significant differences in the \texttt{EnvironmentConfig} for PPO is that the granularity of the positions, velocities and rotations are larger than with Q-Learning, since PPO encodes the elements of the state as real numbers, resolving the scalability issues in which Q-Learning suffers. In addition, since we apply PPO in multi-agent scenarios with independent learners, the \texttt{EnvironmentConfig} contains multi-agent information with respect to the tracked cars, and with dynamic goals also the tracked parking spaces. Since $n_{track} > 0$ in our PPO \texttt{EnvironmentConfig}, we no longer use rings to detect nearby cars since we have full information of their pose, and we only use the rings to avoid walls. By default, our \texttt{EnvironmentConfig} for PPO has velocity and goal sharing enabled among the agents, but it does not contain global information.


\chapter{Results}
\label{ch:results}

In this chapter, we describe the results\footnote{The datasets for the graphs in our results can be found in the provided code under the \texttt{results} directory.} of implementing of our environment designed in \Cref{ch:design}, following the implementation specified in \Cref{ch:implementation}. We begin by applying Q-Learning on the single-agent environment described in \Cref{sec:single-agent-mdp}, verifying the construction of our simple environment. Then, after facing limitations with Q-Learning, we apply IPPO to our environment with multiple agents, both with fixed and dynamic goals. We analyse the performance of our models quantitatively using the metrics obtained from our analysis tools, as explained in \Cref{sec:mdp-metrics} and \Cref{sec:ppo-analysis}.

\section{Q-Learning}

Our experiments begin with training Q-Learning in our single-agent discrete environment, verifying the environment's basic construction. Then, we add obstacles to the environment to increase its complexity, and Q-Learning soon becomes intractable.

\subsection{Basic Environment}

Our basic environment uses the default environment parameters listed in \Cref{tab:env-params} and \Cref{tab:env-params-2}, with a single agent, no parked cars, no distance to goal information, and fixed goals. Hence, the agent's state consists of its velocity and its localised angle to it's goal parking space, and its actions consist of its acceleration and angular velocity. Thus, our Q-Table has $|\mathbb{S}\times\mathbb{A}| = 3 * 8 * 3 * 3 = 216$ values, which is tractable for Q-Learning.

After tuning Q-Learning, we found that the hyper-parameters $\alpha = 0.1$, $\gamma = 0.9$ and $\epsilon = 0.3$ yield good performance, and $3000$ training episodes were necessary for Q-Learning to converge to an optimal model. It was necessary to decay $\epsilon$ from $0.3$ to $0$ over the first $2000$ episodes, then fix $\epsilon = 0$ for the remaining $1000$ episodes (known as the `evaluation' zone), to reinforce the values in the Q-Table by following the learned policy without any randomness. Stochasticity from the $\epsilon$ parameter seemed to have a large effect on the convergence of our models, since the agents' goal state is near many terminal states (walls), amplifying the effect of even a small amount of stochasticity in the agent's policy. In addition, training for a large number of episodes was necessary since the underlying MDP has $74 * 74 = 5476$ possible $(x,y)$ positions, and each action only moves the agent one position at a time, making our environment's state distribution quite sparse.

\begin{figure}[!h]
\centering
\includegraphics[width=\textwidth]{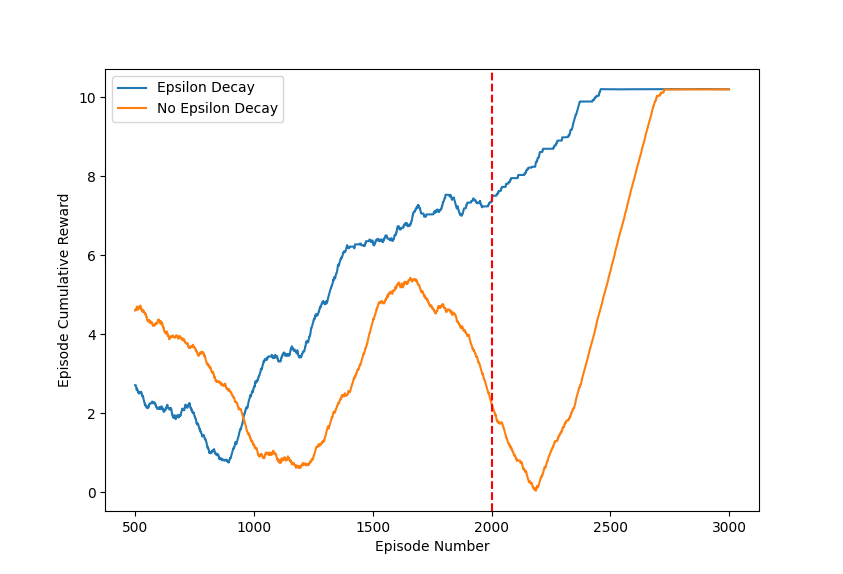}
\caption{Results of Q-Learning in the basic environment.}
\label{fig:q-learning-basic-results}
\end{figure}

\Cref{fig:q-learning-basic-results} shows the results of applying Q-Learning in our basic environment, in which we measure the agents' cumulative reward per episode, with a rolling mean of $500$ episodes for a clearer graph. The agent's cumulative reward must be close to $10$ if they learned to reach their goal, since the agent receives a reward of $10$ for reaching it. In \Cref{fig:q-learning-basic-results} the blue line corresponds to the training Q-Learning by decaying $\epsilon$ linearly to $0$ in the first $2000$ episodes, and the orange line corresponds to no such decay, and the red line indicates the start of the evaluation zone. Indeed, we see that both runs converged to an optimal policy successfully parking the agent in every episode, but the run with decaying $\epsilon$ converged faster.

\subsection{Environment With Parked Cars}

With Q-Learning shown to converge to optimal policies for our basic environment, we next add complexity to the environment, since the basic environment is unrealistic. We begin by introducing parked cars into the environment, and give the agent rings such that they can be detected and avoided. Now, our environment has $35$ parked cars, leaving only the agent's goal parking space uninhabited, putting the agent's parking space in the gap between two parked cars on either side which it must avoid.

\subsubsection{Environment Configurations}

Since the agent may have a varying number of rings with a varying number of historical ring states, and each ring may count a variable number of objects, we train multiple models with varying combinations of the ring parameters. Through observation, we found that 5 rings with diameters $14$, $11$, $10$, $6$, $7$ lie on the threshold of its nearby obstacles, providing the required distance information. We thus train 4 models counting $1$ to $3$ obstacles in each ring with no historical ring states, and another model counting $1$ obstacle in each ring with $1$ historical ring state to test the addition of historical ring states. In addition, we train a model with all possible ring diameters from $14$ to $7$ counting $1$ obstacle and no historical ring states, to fully test the information provided by the rings without a specific configuration of rings chosen. We use the same Q-Learning hyperparameters as in the basic environment.

\begin{figure}[!h]
\centering
\includegraphics[width=\textwidth]{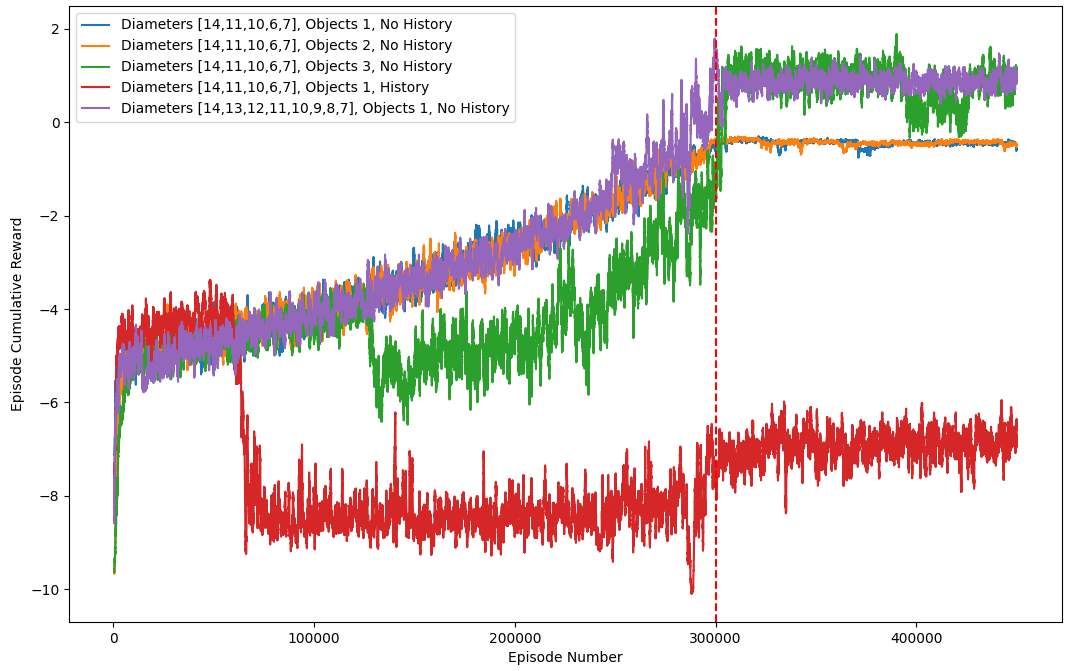}
\caption{Results of Q-Learning in the environment with parked cars and rings.}
\label{fig:q-learning-rings-results}
\end{figure}

\subsubsection{Results}

\Cref{fig:q-learning-rings-results} shows the agent's cumulative reward in our environment with parked cars when training Q-Learning with our varying ring configurations, with a rolling mean over $500$ episodes for a clearer graph. Now, it takes around $450,000$ episodes for Q-Learning to converge, which is much longer than the $3000$ episodes in our basic environment. Since the cumulative reward never converges to a value near $10$, and it instead converges to between $0$ and $2$ for most models, the agent is unable to learn to park successfully in their parking space, but they also learn not to crash with obstacles since the cumulative reward isn't near $-10$, which is the reward received for crashing (except when we include historical ring states shown by the red line). Hence, unfortunately it was found that the rings do not provide the agent enough information to avoid the obstacles, and the agent's learned behaviour is to halt near an obstacle (with a velocity of $0$) instead of crashing.

\begin{figure}[!h]
\centering
\includegraphics[width=0.8\textwidth]{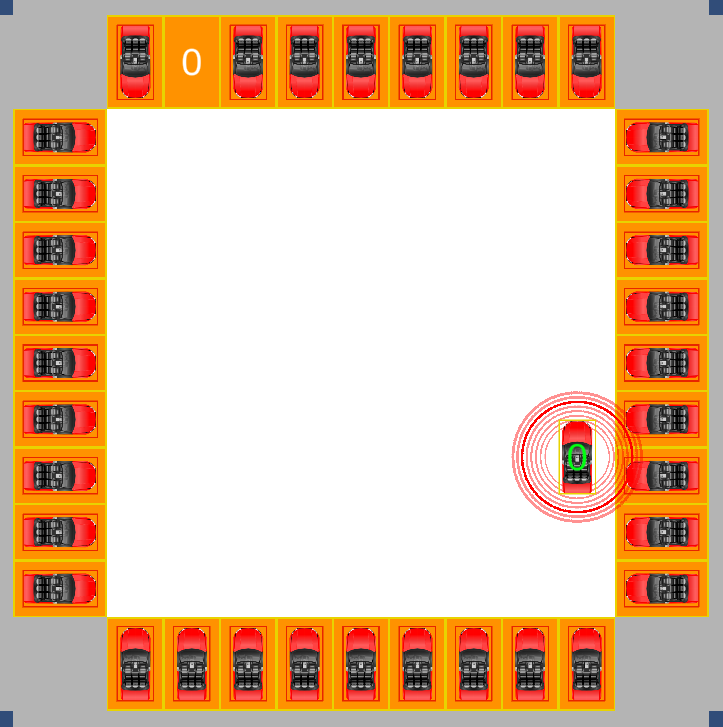}
\caption{A particular position of the agent where its learned model causes it to halt.}
\label{fig:q-learning-halt-pos}
\end{figure}

After manually inspecting the agent's learned policies, the agent was found to halt near obstacles, even when travelling in a direction away from the obstacles, as shown in \Cref{fig:q-learning-halt-pos}. Despite being sub-optimal, such a policy is optimal with respect to the information available in the MDP, since if the agent were rotated $90$ degrees towards the parked cars then its motion may cause it to crash, but the agent doesn't know its rotation with respect to the parked cars, so it makes a pessimistic assumption to avoid the $-10$ reward. Thus, the agent requires angular information in addition to distance information of its nearby obstacles, however encoding such information in the MDP significantly increases the size of the Q-Table, making the application of Q-Learning become intractable.

\subsubsection{Limitations}

We have identified a limitation with the information provided by the rings in the MDP, and require more information about the nearby cars to the agent in order to avoid them. Since the localised pose of a nearby car provides sufficient information for the agent to determine its pose, as explained in \Cref{sec:localised-pose}, we next encode such information in the state of our MDP. However, since such additional state significantly increases the dimensionality of the state-space and the Q-Table, we next apply PPO instead of Q-Learning, due to its superior scalability. Indeed, one may have further discretised the environment, however the additional effort of discretisation isn't necessary for PPO, and deep learning may discover such feature transformations for us \cite{dqn}. In addition, we're unable to train Q-Learning for longer, as the limit in duration of a job on the department's compute cluster was neared when when training for $450,000$ episodes, making training for longer difficult (as it would have to occur over multiple jobs, requiring orchestration).

\section{PPO with Fixed Goals}
\label{sec:ppo-fixed-goals-results}

Having faced scalability issues when applying Q-Learning in the environment with parked cars present, we use deep RL to overcome the issues, now approximating the policy using neural networks rather than tabulating every possible combination like Q-Learning. In this section, we describe the results of our various experiments applying IPPO in our environment with fixed goals and parked cars, obtaining a $98.1\%$ park rate with $7$ agents, and a maximal $99.3\%$ park rate with $2$ agents. In what follows, we refer to `X redundancy' as repeating an experiment X times and averaging the results, to improve their accuracy.

\subsection{Environment Configuration}

Since elements of the state in PPO are real numbers between $0$ and $1$, and they each correspond to an input to a neural network, state spaces with much higher granularity can be used. This is because an element of the state with domain $\mathbb{N}_n$ corresponds to one neuron, where the environment maps every $i \in \mathbb{N}_n$ to an input value $i / n$ to the neuron, but such an element of the state increases the size of the Q-Table by a factor of $n$ in Q-Learning. Note how the increase in neural network complexity remains one neuron irrespective of $n$, however the Q-Table's increase in size is proportional to $n$.

With this increase in scalability, we are able to increase the granularity of the positions, velocities and rotations in the environment, such that they're more realistic. In addition, we are able to track nearby cars, encoding their localised pose in the state. \Cref{tab:ppo-fixed-goals-env-config} lists the environment parameters used in our experiments of applying IPPO with fixed goals. Note that the granularity of the positions, rotations and velocities have increased to more realistic values, and the agents now track one nearby car (which includes other agents, making the configuration applicable to multi-agent settings). Since the agents now know the pose of its nearest car obstacle, we only use one ring to detect and avoid walls. In addition, we now have 16 parked cars in the environment, which is approximately $50\%$ of the maximum. By default, we share the agents' goals and velocities in multi-agent experiments.

\subsection{Setup}

After tuning PPO's hyperparameters using grid-searching on the suggested ranges provided by ML-Agents \cite{ml-agents} and harnessing the hyperparameters provided by \citet{ppo-self-parking} as a baseline, we obtain a set of stable hyperparameters for IPPO in our higher granularity environment, listed in \Cref{tab:ppo-hyperparams}. We also find that a total of 16 million steps are required to converge IPPO to global optima with 7 agents, and a smaller total of 10 million are required with a single agent. Training for such a number of time-steps takes approximately one day of compute time with two CPUs and 6 gigabytes of RAM on the department's CPU cluster. As explained in \Cref{sec:ppo-train}, we reserve $80\%$ of those episodes for training, and the last $20\%$ for evaluation. In our subsequent applications of IPPO, we measure the performance of our models using the various metrics recorded as explained in \Cref{sec:ppo-analysis}, since manual analysis of the agents' cumulative reward is less informative and time-consuming. Some key metrics used are the agents' park, crash and halt rates, which are the ratio of their episodes in which they successfully parked, crashed, or halted, respectively (here, the agent `halted' if they neither parked nor crashed).

\begin{table}[!h]
\centering
\begin{tabular}{|l|r|}
\hline
\textbf{Parameter Name}                                                   & \multicolumn{1}{l|}{\textbf{Parameter Value}} \\ \hline
Batch Size                                                                & 32                                            \\ \hline
Buffer Size                                                               & 4096                                          \\ \hline
Learning Rate                                                             & 0.0001                                        \\ \hline
Beta                                                                      & 0.002                                         \\ \hline
Epsilon                                                                   & 0.25                                          \\ \hline
Lambda                                                                    & 0.925                                         \\ \hline
\# Epoch                                                                   & 10                                            \\ \hline
\begin{tabular}[c]{@{}l@{}}\# Hidden Units\\ (Neural Network)\end{tabular} & 256                                           \\ \hline
\begin{tabular}[c]{@{}l@{}}\# Layers\\ (Neural Network)\end{tabular}       & 3                                             \\ \hline
Gamma                                                                     & 0.999                                         \\ \hline
Time Horizon                                                              & 128                                           \\ \hline
\end{tabular}
\caption{PPO hyperparameters.}
\label{tab:ppo-hyperparams}
\end{table}

One particular hyperparameter with a large effect on the quality our learned models is the discount factor, $\gamma$. When training with 7 agents, we find that their crash rate varies between $4.41\%$ and $2.04\%$ when $\gamma$ varies between $0.92$ and $0.9998$, and the crash rate decreases the larger the value of $\gamma$. \Cref{fig:vary-gamma} show how the agents' crash rate varies with $\gamma$, where the bottom graph shows the larger values of $\gamma$ in more detail. Since $\gamma = 0.999$ yields a low crash rate and is outside the region of marginal instability beyond $0.999$, we use $\gamma = 0.999$ in our other experiments. We tuned the the other hyperparameters in a similar way, seeking values that yield low crash rates.

\begin{figure}[!h]
\centering
\includegraphics[width=.8\textwidth]{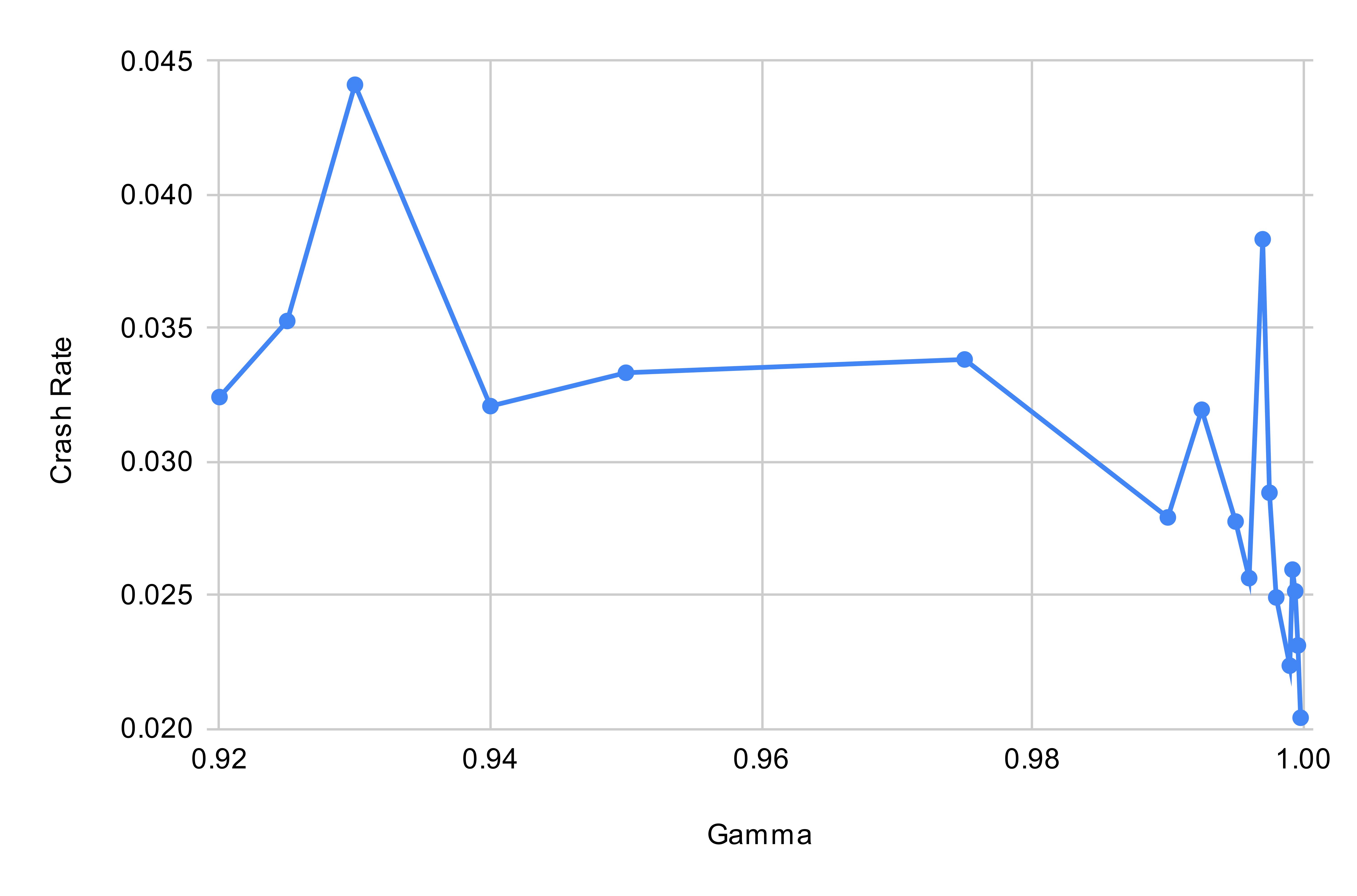}
\includegraphics[width=.8\textwidth]{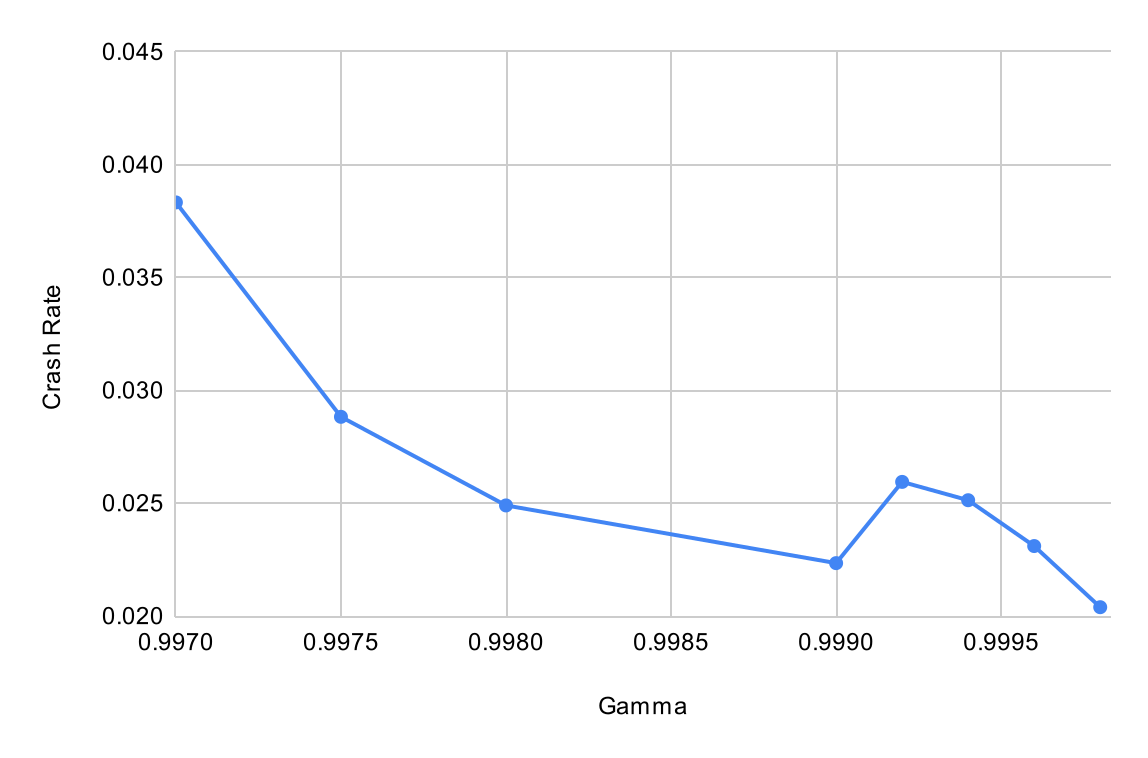}
\caption{Crash rate when varying the discount factor $\gamma$ with 7 agents.}
\label{fig:vary-gamma}
\end{figure}

\subsection{Density and Dimensionality}

Having obtained a stable set of hyperparameters for IPPO in our environment, we vary some of the environment's key parameters to identify trends relating to its density and dimensionality. In this section, we describe the results from our experiments of applying IPPO and varying the number of agents, the number of parked cars, and the number of tracked cars ($n_{track}$) in our environment, since such parameters are critical to the environment's density and dimensionality.

\subsubsection{Number of Agents}

When varying the number of agents from $1$ to $7$ and training for an equal number of $14$ million steps shared among the agents (with triple redundancy), we find that the agents' park rate decreases from $99.3\%$ with 2 agents to $97.4\%$ with $7$ agents, and a significantly lower $30.9\%$ with $1$ agent. \Cref{fig:vary-num-agents} shows a graph of the agents' park rate when varying the number of agents from $2$ to $7$ (with the single-agent outlier excluded). Thus, $2$ agents yields the optimal park rate in our environment, and increases beyond $2$ decrease the park rate by a small amount, however decreases below $2$ decrease the park rate by a prohibitively large amount.

\begin{figure}[!h]
\centering
\includegraphics[width=.8\textwidth]{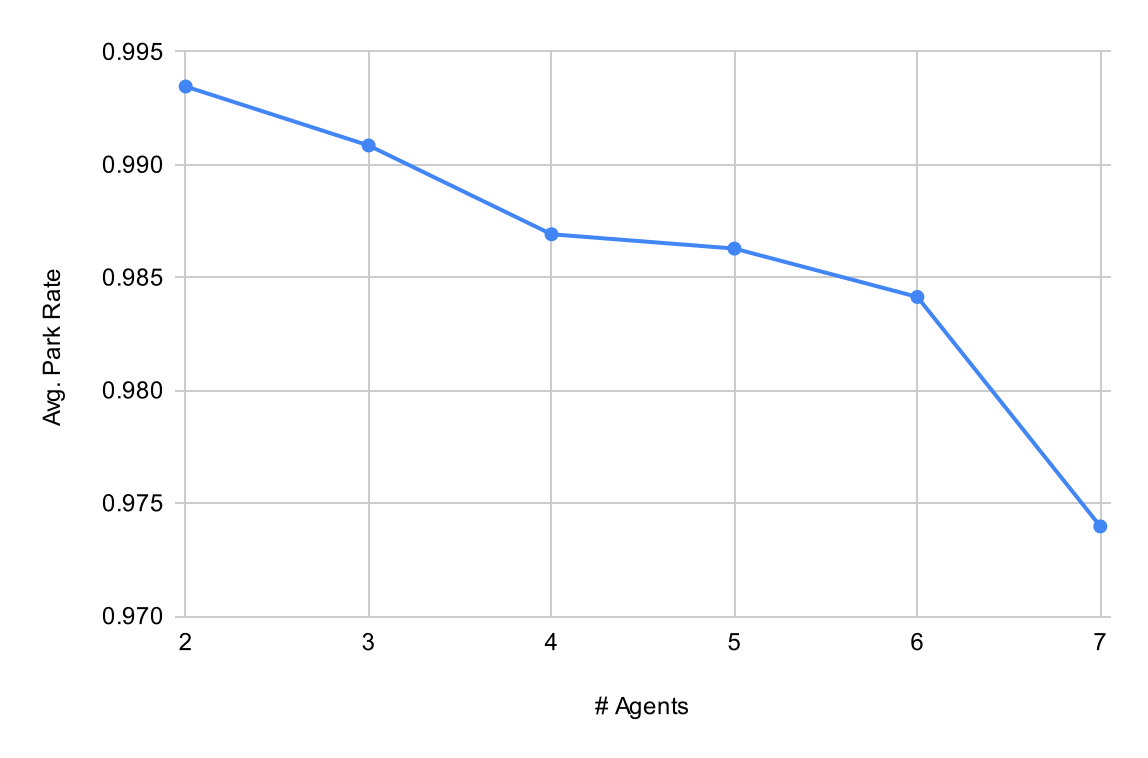}
\caption{Park rate when varying the number of agents.}
\label{fig:vary-num-agents}
\end{figure}

While the non-stationarity of our multi-agent setting makes the decreasing park rate beyond $2$ agents expected, the significant increase in park rate between $1$ and $2$ agents is an unexpected result. We believe such an increase is a result of an increase in the density of the environment with $2$ agents, since there are more states in which the agent has to avoid a car (since agents have to avoid other agents as well as parked cars). Hence, the `training data' (i.e. the situations that arise in the environment) is denser with more agents, giving the agents more situations to learn from. We arrive at such a conclusion by manually inspecting the cumulative reward graphs for the models, observing that the agents get stuck in halting local optima for a long time during training (in which the agents learn to halt rather than to park or crash), and the agents escape such local optima more quickly when there are more agents present in the environment.

Since our experiments reveal a relatively high parking rate with $7$ agents, our other multi-agent experiments use $7$ agents, due to the attractive density they provide to our training data (as shown in \Cref{fig:env-impl}).

\subsubsection{Number of Parked Cars}

Another parameter varying the density of our environment is the number of parked cars within. With more parked cars, the agents' parking spaces are increasingly likely to be surrounded by more parked cars, making it more difficult for them to successfully park without colliding. Hence, we desire a large number of parked cars in our environment, such that the agents learn to park in challenging scenarios.

To assess the effect of the number of parked cars on the agents' learned models, we vary the number of parked cars between $0$ (one) and $35$ (all) with $1$ agent, and between $0$ (none) and $29$ (all) with $7$ agents. We train for $6$ million steps with $1$ agent, and $16$ million steps among the $7$ agents, with triple redundancy. \Cref{fig:vary-num-parked-cars-1ag} and \Cref{fig:vary-num-parked-cars-7ag} show how the agents' park and halt rates vary with the number of parked cars in the environment, for the experiments with $1$ and $7$ agents respectively.

\begin{figure}[!h]
\centering
\includegraphics[width=.8\textwidth]{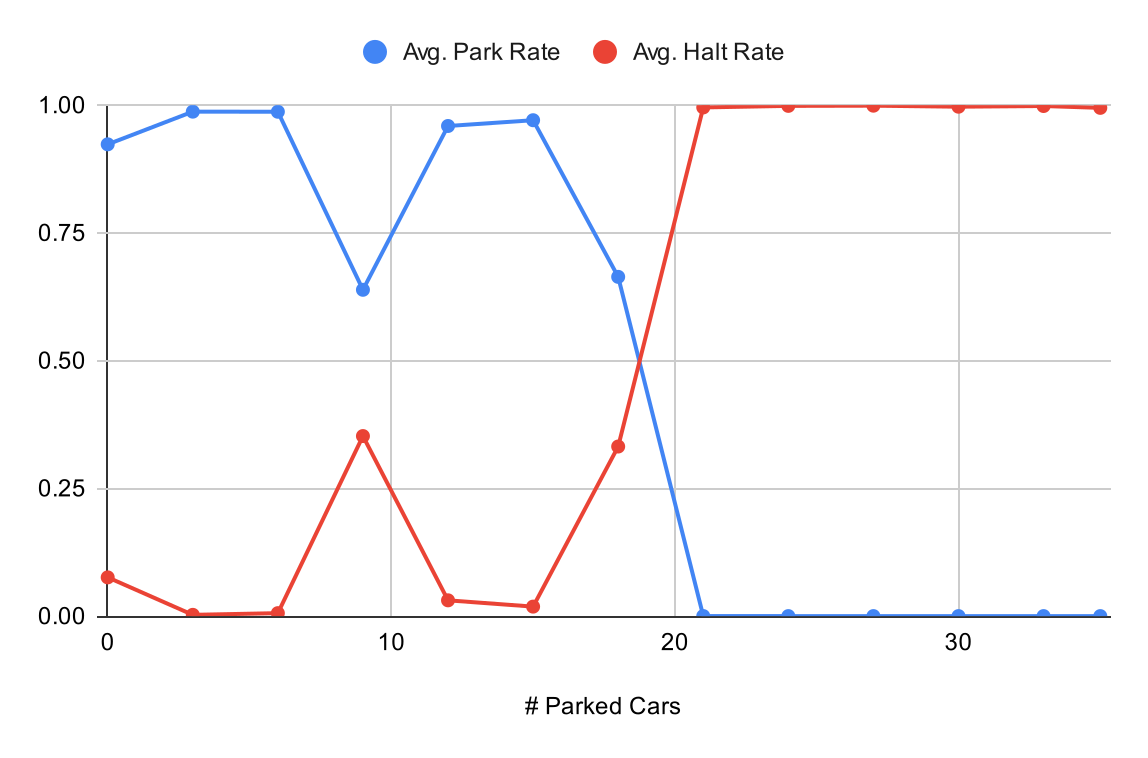}
\caption{Park rate when varying the number of parked cars, with $1$ agent.}
\label{fig:vary-num-parked-cars-1ag}
\end{figure}

\begin{figure}[!h]
\centering
\includegraphics[width=.8\textwidth]{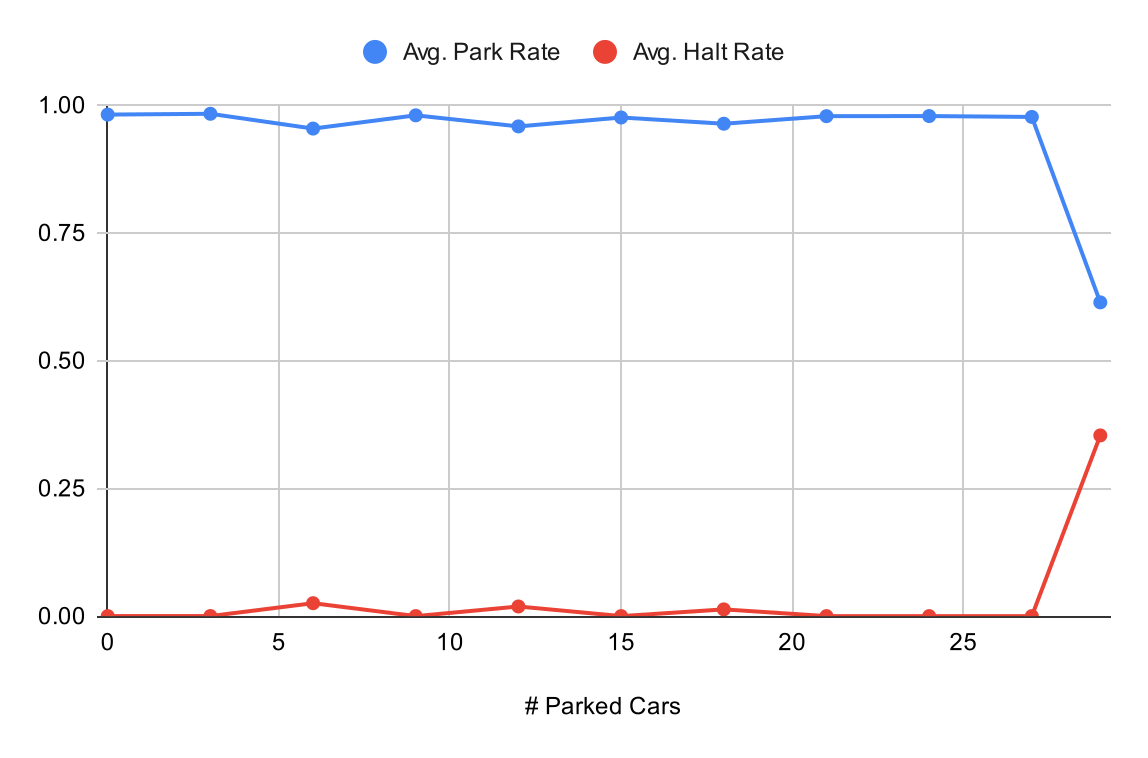}
\caption{Park rate when varying the number of parked cars, with $7$ agents.}
\label{fig:vary-num-parked-cars-7ag}
\end{figure}

It can be seen from \Cref{fig:vary-num-parked-cars-1ag} that with a single agent in the environment, they're unable to learn to successfully park when there are more than approximately 16 parked cars in the environment, since their park rate drops to $0\%$. Instead, they learn to halt, getting stuck in local optima, after manual analysis of their cumulative reward graphs. Intuitively, this happens because when the agent begins training it does not yet know how to reach its parking space, and when its parking space is more likely to be surrounded by parked cars, it's more likely to collide with a neighbouring parked car (as more exist), making it less likely to reach its parking space during training. Since the agent reaches its parking space less often, it's less able to reinforce the behaviour of reaching it. This result shows that there is a trade-off that must be made between the density of our environment and the convergence of our models, where too high a density results in poor convergence, and vice-versa.

On the contrary, \Cref{fig:vary-num-parked-cars-1ag} shows that when training with a large number of agents, the agents learn to park successfully even with a large number of parked cars. Hence, we have another result showing multiple agents aiding training. The reason we obtain higher parking rates here is again due to density, since the agents are able to learn to avoid cars also from each-other - not only with the parked cars in the single-agent case.

Thus, in light of our findings, we use $16$ parked cars in our environment for other experiments (populating approximately half of the parking spaces with parked cars). Such a value is large enough to arise sufficiently difficult parking scenarios, yet it's small enough to yield a high parking rate irrespective of the number of agents.

\subsubsection{Number of Tracked Cars}

Since the agents use the localised pose of its nearby cars to avoid them, the number of nearby cars they're tracking ($n_{track}$) may vary their park rate. Hence, we measure the agents' park rate with $n_{track}$ varying between $0$ and $4$, from no knowledge of nearby cars to full knowledge of $4$ nearby cars. \Cref{fig:vary-num-tracked-cars} shows the resulting park rate for different values of $n_{track}$, where we trained $7$ agents for $16$ million steps and averaged the results over $3$ repetitions.

\begin{figure}[!h]
\centering
\includegraphics[width=.8\textwidth]{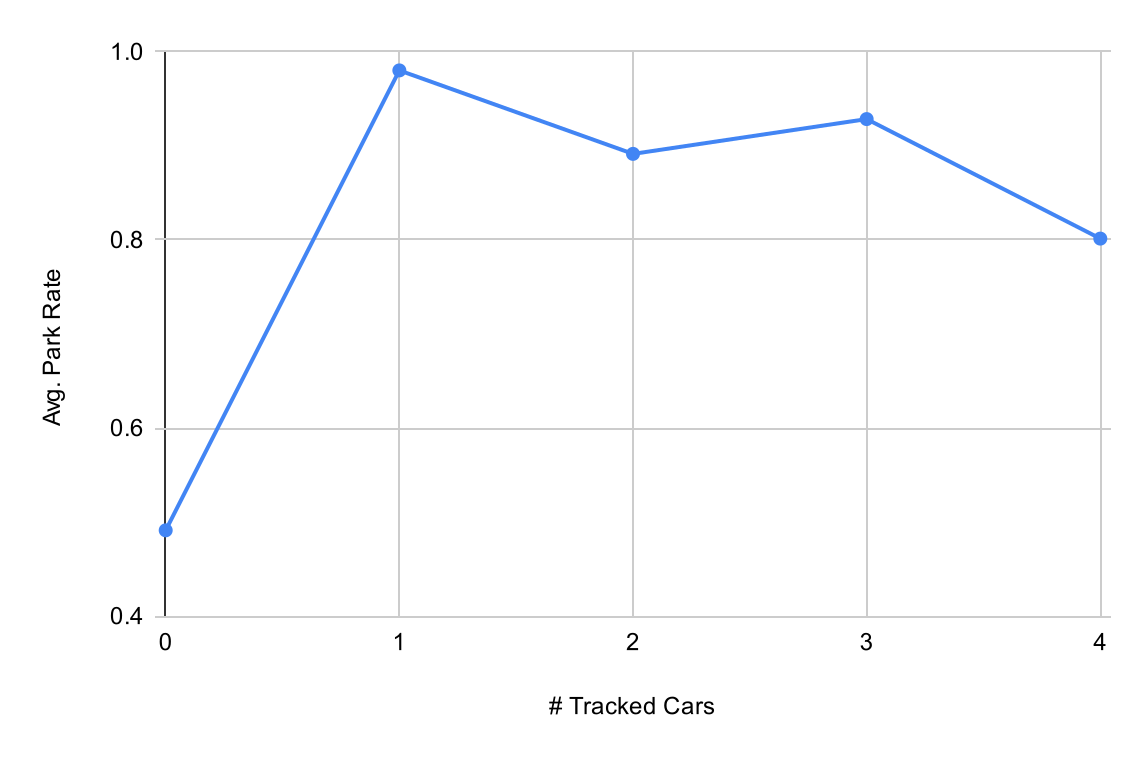}
\caption{Park rate when varying the number of tracked cars.}
\label{fig:vary-num-tracked-cars}
\end{figure}

It can be seen from the graph in \Cref{fig:vary-num-tracked-cars} that $n_{track} = 1$ is optimal, and that tracking more than $1$ agent reduces the park rate. Hence, the agents' optimal learned policy greedily avoids the closest car to it, and the inclusion of the information of other cars only hinders the agents' learned policy. Such a result is unexpected, because it's natural to assume that giving the agent more information will enable it to make better decisions. However it appears that the information of more than one nearby car is not utilised, and it increases the dimensionality of the state-space without much benefit. 

In addition, it can be observed from the graph that the park rate is much larger when tracking agents than not tracking them. Hence, there are a significant number of scenarios that arise in our environment where the agents require the localised pose of nearby cars to avoid them, which validates our environment's construction since the agents are unable to yield a high park rate by following a trivial policy within.

\subsubsection{Shared State}

In addition to tracking nearby agents, agents may share their goals and velocities with each other to enable better path planning, as explained in \Cref{continuous-mdp-fixedgoals}. Hence, we assess the agents' improvement in park rate when sharing their goals and velocities, to quantify the benefit of such shared state in our MDP.

We train 7 agents for $16$ million steps in our environment under four configurations:
\begin{itemize}
    \item Without shared goals or velocities;
    \item With shared goals but without shared velocities;
    \item With shared velocities but without shared goals;
    \item With both shared velocities and shared goals.
\end{itemize}
In each configuration, we measured the agents' average park rate, with a redundancy of 6 repititions. \Cref{fig:vary-shared-info} shows the agents' average park rate in each of the four configurations.

\begin{figure}[!h]
\centering
\includegraphics[width=.8\textwidth]{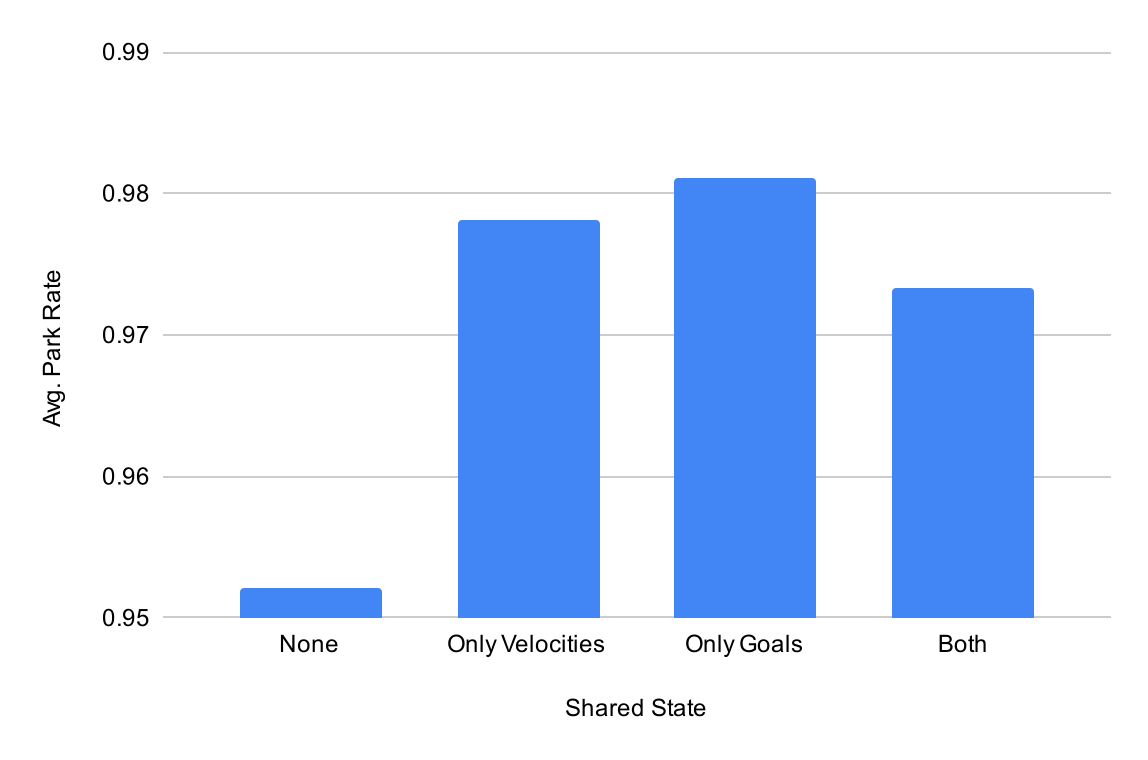}
\caption{Park rate when varying the shared information between the agents.}
\label{fig:vary-shared-info}
\end{figure}

It can be seen from the graph in \Cref{fig:vary-shared-info} that  the agents have a higher park rate with shared state, and the highest park rate occurs when the agents share only their goals, increasing the park rate from $95.2\%$ with no shared state to $98.1\%$ with only shared goals, a $2.9\%$ increase. Interestingly, when the agents share both their velocity and their goals, the park rate decreases from when only one of them are shared. This may be due to sharing both increasing the dimensionality of the MDP too much, with the combination of the two data being less useful than one datum alone. However, since sharing both doesn't decrease the park rate by much, and the agents have access to more information which may be useful as the complexity of the MDP grows, we share both as a baseline.

\subsection{Safety}

Having assessed the agents' park rate in our environment with varying density and dimensionality, we next seek to assess the safety of their driving, since it's also important \textit{how} the agents reach their goals. In this section, we attempt two techniques to improve the safety of the agents' driving. Our first technique is to punish the agents proportional to their forward and angular velocities to increase the smoothness of their driving, as explained in  \Cref{sec:mdp-rewards}. Secondly, we train the agents with larger hitboxes so they learn to keep a threshold distance away from the other cars, as explained in and \Cref{sec:mdp-optimisations}.

\subsubsection{Smoothness}

To measure the effect of constraining the agents' forward velocities ($v$) and angular velocities ($\omega$) such that they drive smoothly, we reward each agent $\frac{-k}{\tau}|v||\omega|$ at each time-step, varying the punishment sum $k \geq 0$, with larger values enforcing the punishment more heavily. We measure the agents' average forward and angular velocities, as well as their crash rates, to determine the effect of enforcing the punishment.

After training $7$ agents for $16$ million steps with varying values of $k$, we obtain the results shown in \Cref{fig:vary-vdtheta-rew-vel} and \Cref{fig:vary-vdtheta-rew-crash}. In \Cref{fig:vary-vdtheta-rew-vel}, we see that as the punishment is enforced, both the agents' forward and angular velocities decrease (we normalise their values in the graph). However, the angular velocities decrease to an asymptote much more quickly than the forward velocities. Hence, the agents conform to driving more smoothly when enforcing the punishment.

\begin{figure}[!h]
\centering
\includegraphics[width=.8\textwidth]{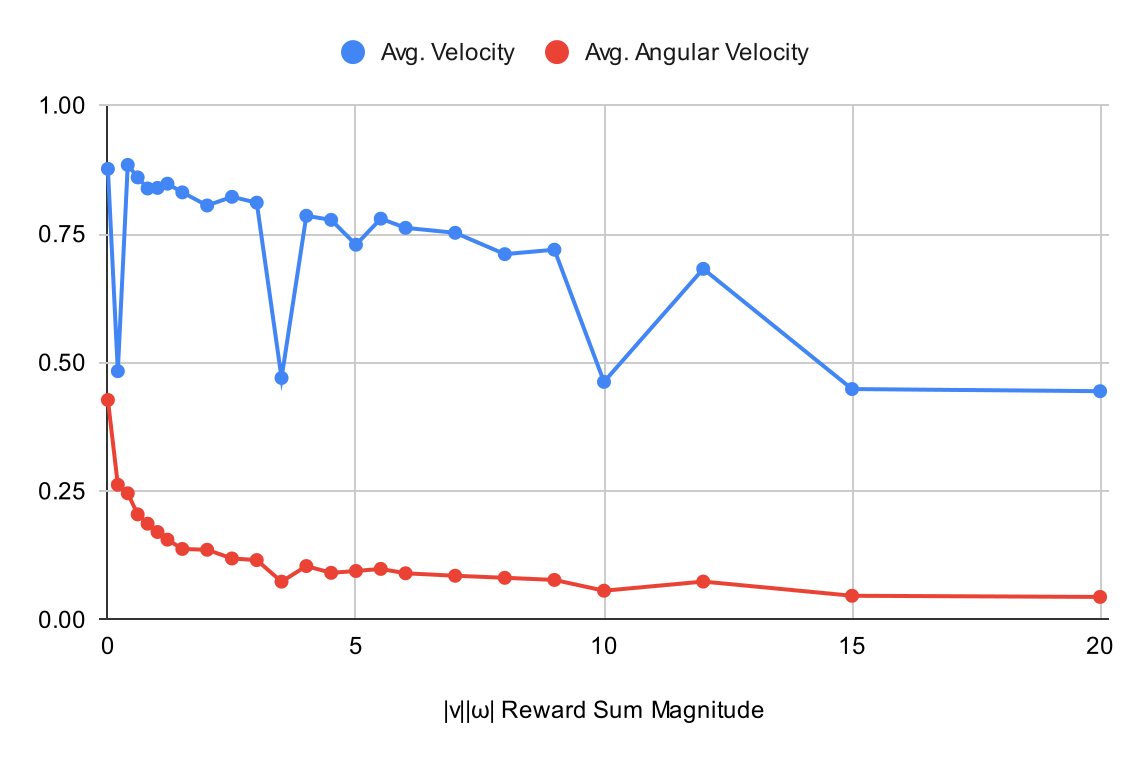}
\caption{Average forward and angular velocity when varying the magnitude of the smoothness punishment sum.}
\label{fig:vary-vdtheta-rew-vel}
\end{figure}

\begin{figure}[!h]
\centering
\includegraphics[width=.8\textwidth]{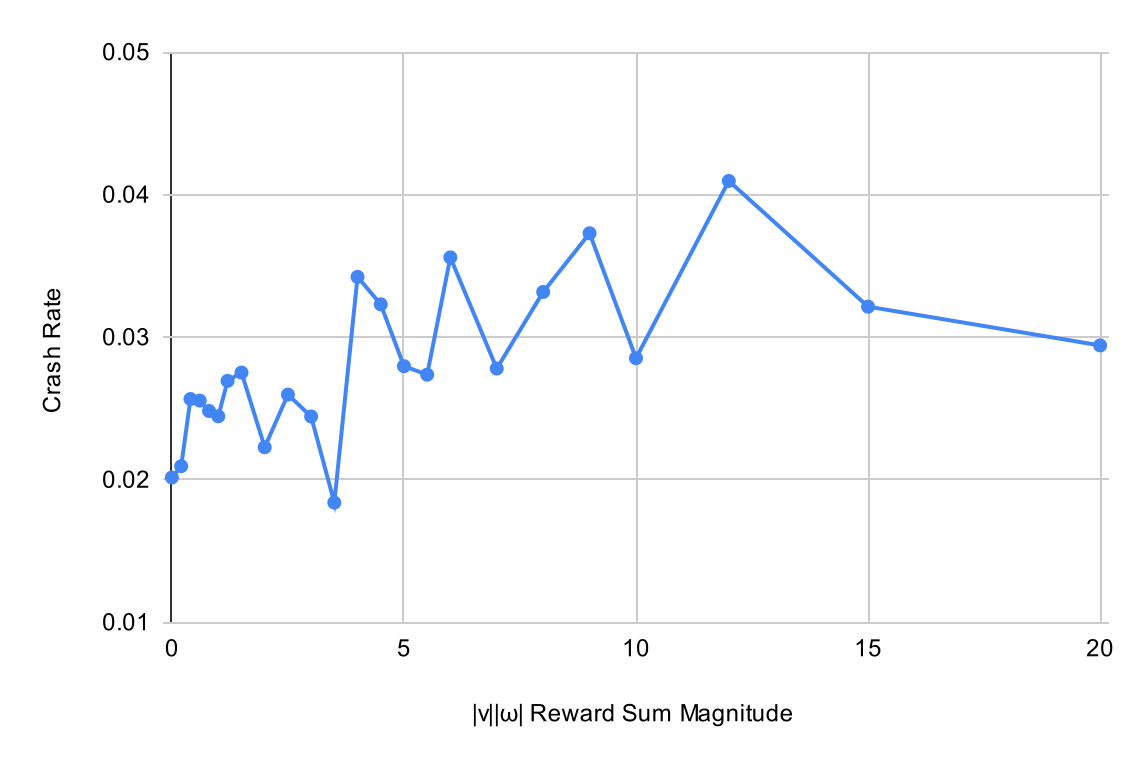}
\caption{Crash rate when varying the magnitude of the smoothness punishment sum.}
\label{fig:vary-vdtheta-rew-crash}
\end{figure}

However, as can be seen in \Cref{fig:vary-vdtheta-rew-crash}, the agents' crash rate increases when the punishment is enforced. This is expected, because the punishment limits the agents' choice of their rotations, making them avoid sharp turns, in-tern reducing their ability to escape dangerous scenarios. Hence, we do not get an increase in driving safety for free - it comes at the expense of an increased crash rate. Thus, a compromise must be made between driving smoothness and park rate in our models.

\subsubsection{Larger Hitboxes During Training}

To attempt to further increase the safety of the agents in our environment, we increase the cars' hitboxes during training, such that the agents learn to keep a threshold distance away from them. We measure the agents' average distance to their closest car and their crash rates to determine the effect of increasing the cars' hitboxes during training.

After training $7$ agents for $16$ million steps with varying scale factors of the cars' hitboxes in the environment (with double redundancy), we obtain the results shown in \Cref{fig:vary-car-scale-park-rate-nearbydist} and \Cref{fig:vary-car-scale-park-rate}.

\begin{figure}[!h]
\centering
\includegraphics[width=.8\textwidth]{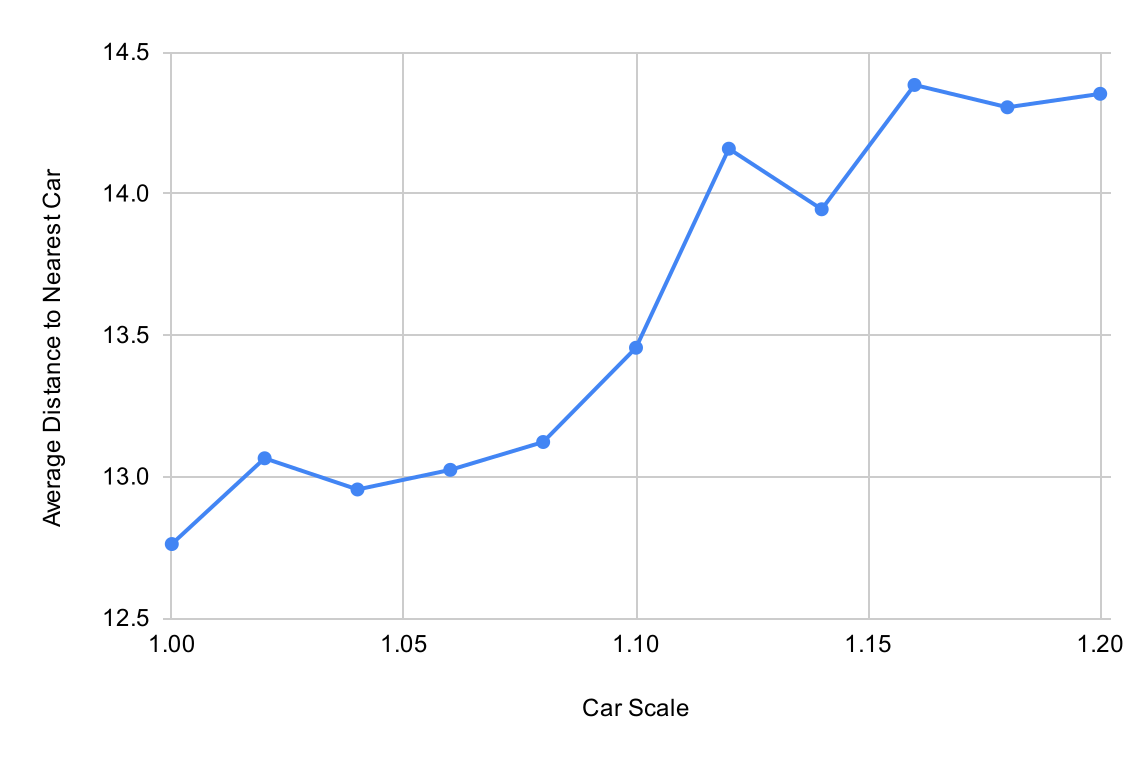}
\caption{Average distance to the agents' nearest car when varying the scale factor of the cars' hitboxes during training.}
\label{fig:vary-car-scale-park-rate-nearbydist}
\end{figure}

\begin{figure}[!h]
\centering
\includegraphics[width=.8\textwidth]{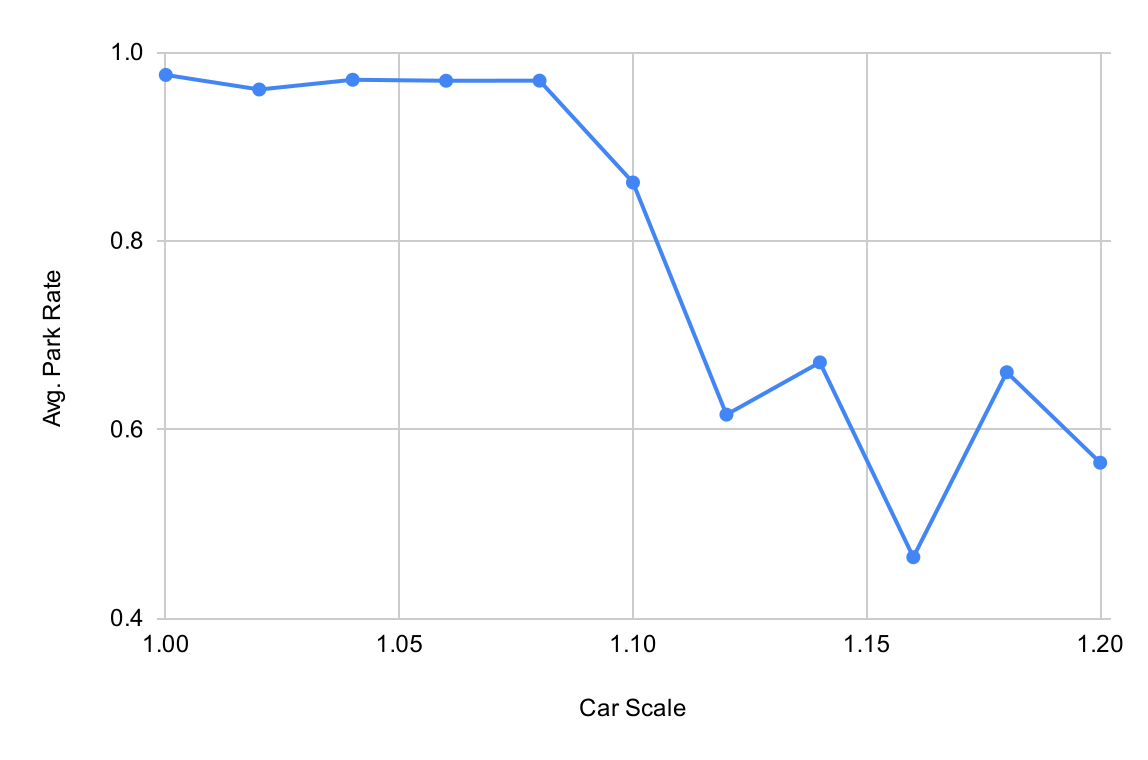}
\caption{Park rate when varying the scale factor of the cars' hitboxes during training.}
\label{fig:vary-car-scale-park-rate}
\end{figure}

We can see from \Cref{fig:vary-car-scale-park-rate-nearbydist} that increasing the scale of the cars' hitboxes during training increases the agents' average distance to their closest car, as desired. However, \Cref{fig:vary-car-scale-park-rate} shows that such an increase in safety doesn't come for free, with the park rate decreasing as the scales increase, due to the environment having a higher density during training with larger hitboxes. Hence, similar to the smoothness punishment, a compromise must be made between the distance the agents leave with the cars and the park rate in our models.

\section{PPO with Dynamic Goals}

Having assessed the behaviour of the agents in our environment with fixed goals under various configurations, we next assess their behaviours with dynamic goals. In this section, we describe the results of applying IPPO in our environment with dynamic goals, quantifying properties relating to its dimensionality as well as competitive and collaborative behaviours that arise in the environment.

\subsection{Environment Configuration}

As explained in \Cref{sec:env-configs}, our environment with dynamic goals uses the same environment parameters as our environment for IPPO with fixed goals in \Cref{sec:ppo-fixed-goals-results}, but with some additional modifications listed in \Cref{tab:ppo-dynamic-goals-env-config}. By default, the agents track only $1$ parking space (i.e. $n_{space} = 1$), $r_{\delta g}(e \rightarrow e) = -0.002$ to discourage exploration, and $r_{\delta g}(p \rightarrow p') = -0.05$ to make the agents commit to a goal. We also enforce $\Cref{eq:rew-change-goal-equiv-simplified}$ so $r_{\delta g}(p \rightarrow e)$ is always equal to $r_{\delta g}(p \rightarrow p')$, hence when we tune $r_{\delta g}(p \rightarrow p')$ we also tune $r_{\delta g}(p \rightarrow e)$.

\subsection{Density and Dimensionality}

As in our environment with fixed goals, we vary the environment's parameters relating to density and dimensionality, and describe the results of applying IPPO with the environment in its different configurations. We first vary the density of the environment by varying the number of agents, then we vary the MDP's dimensionality by varying the number of tracked cars ($n_{track}$) and parking spaces ($n_{space}$).

\subsubsection{Number of Agents}

When varying the number of agents from $1$ to $7$ and training for $14$ million steps among the agents (and averaging the results over $3$ repetitions), the agents' park rate decreases from $99.9\%$ with $1$ agent to $99.5\%$ with $7$ agents, as shown in \Cref{fig:vary-num-agents-dynamicgoals}. Hence, we have much higher park rates in our environment with dynamic goals than with fixed goals, and this is likely due to the fact that the agent can always choose the closest parking space to them as their goal, which reduces the distance the agents have to travel and thus reduces the chance of a collision. In addition, now the optimal density of the environment contains $1$ agent rather than $2$ in the environment with fixed goals. As with fixed goals, our other experiments in this section use $7$ agents due to the attractive density they provide to the environment.

\begin{figure}[!h]
\centering
\includegraphics[width=.8\textwidth]{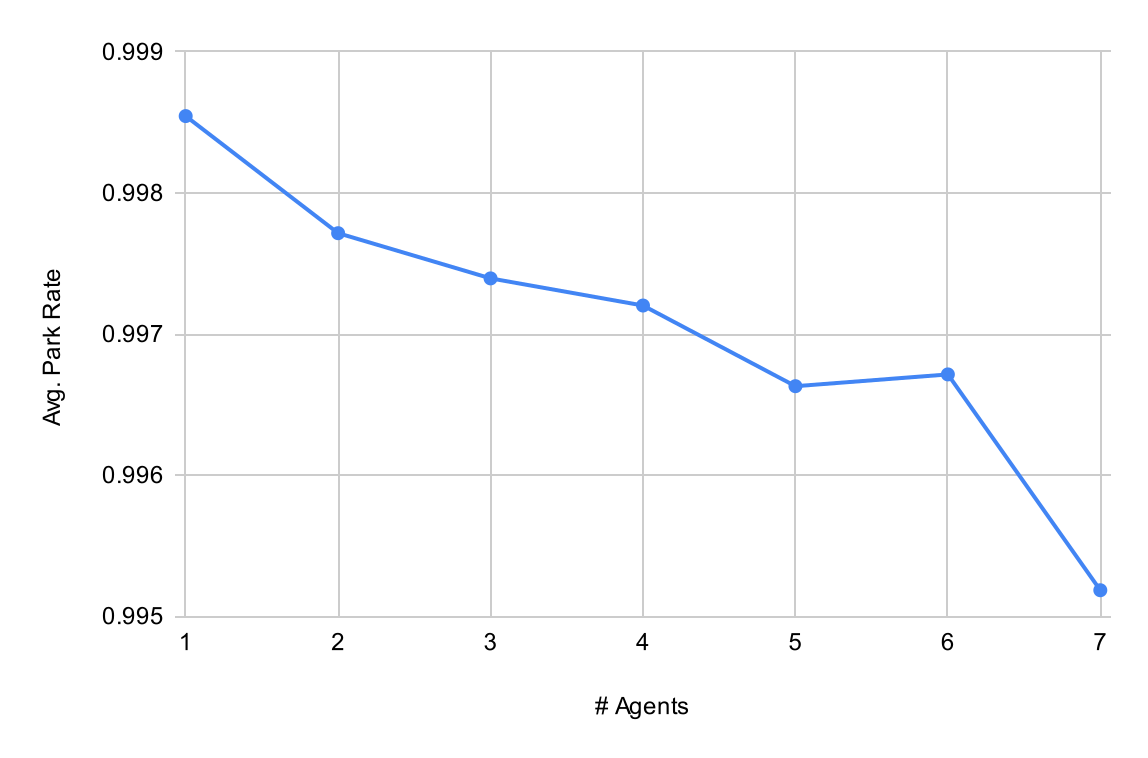}
\caption{Park rate when varying the number of agents in the environment with dynamic goals.}
\label{fig:vary-num-agents-dynamicgoals}
\end{figure}

\subsubsection{Number of Tracked Cars and Parking Spaces}

Since $n_{space}$ is now also a parameter relating to the dimensionality of our MDP (along with $n_{track}$), we vary both $n_{space}$ and $n_{track}$ to identify the optimal values of both. We train $7$ agents for $16$ million steps, and vary both $n_{space}$ and $n_{track}$ between $1$ and $4$ respectively (obtaining $4 * 4 = 16$ total combinations), assessing the agents' park rate for each combination (with triple redundancy).

\begin{figure}[!h]
\centering
\includegraphics[width=.8\textwidth]{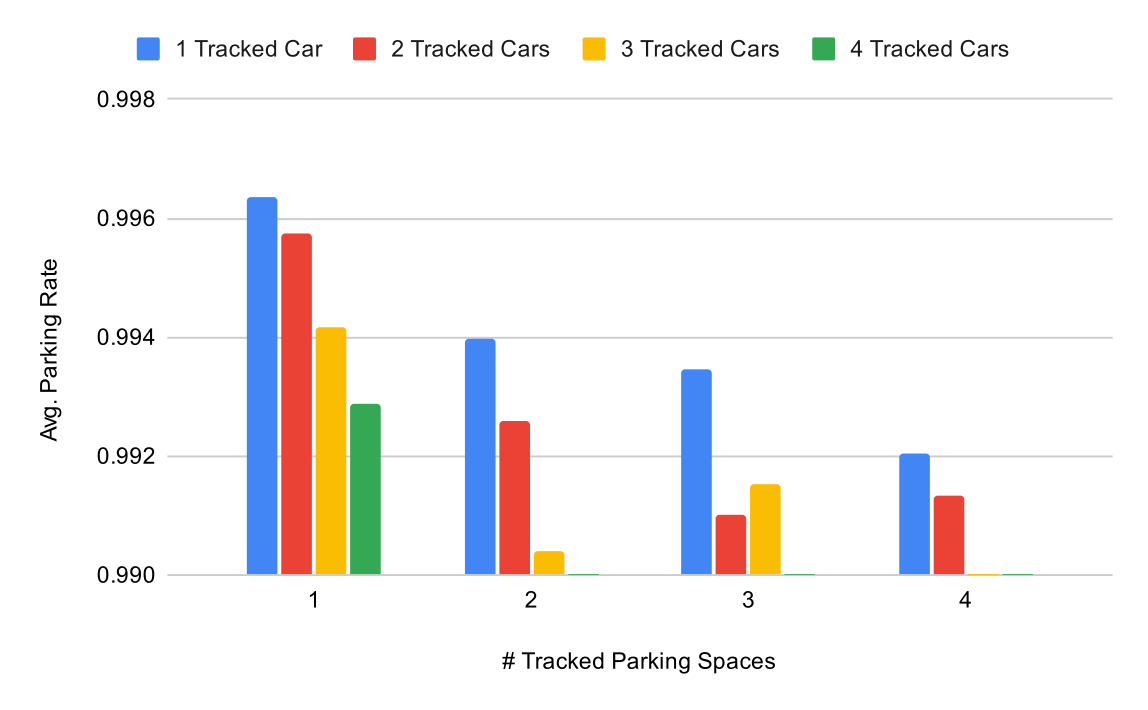}
\caption{Park rate when varying the number of tracked cars and parking spaces.}
\label{fig:vary-num-tracked-cars-parkingspaces}
\end{figure}

\Cref{fig:vary-num-tracked-cars-parkingspaces} shows our results, showing that the agents' park rate is highest when both $n_{track} = n_{space} = 1$, and the park rate decreases when either $n_{track}$ or $n_{space}$ are increased. Hence, similar to the environment with fixed goals, the agents do not benefit from the extra dimensionality of tracking more than one car or parking space, and have the highest park rate when acting greedily and only observing their closest car and parking space. Hence, we use $n_{track} = n_{space} = 1$ as a benchmark in our other experiments. It should be noticed that $n_{space} = 1$ does not prevent the agents from changing their goal parking space in our MDP, since they're still able to explore and find a new one.

\subsection{Competition}
\label{sec:results-competition}

Since competition is an important aspect of the agents' behaviour, in this section we assess competitive behaviours exhibited by the agents in our environment with dynamic goals.

\subsubsection{Method}

One parameter controlling the amount of competition exhibited by the agents in our environment is the magnitude of the change goal punishment, $r_{\delta g}(p \rightarrow p')$, since larger punishments cause the agents to be more reluctant to change their goal parking spaces, causing them to pursue their goal parking space in a more competitive manner. Hence, we vary $r_{\delta g}(p \rightarrow p')$ and measure various metrics relating to the agents' goal seeking behaviour. We measure the following metrics:
\begin{itemize}
    \item The ratio of the time-steps in which the agents are exploring. This metric enables us to quantify the `decisiveness' of the agents, since a high exploration ratio means the agents are less willing to dedicate to a goal parking space, exhibiting indecisive behaviour.
    \item The ratio of the time-steps in which the agents move towards their goal parking space when they have a goal parking space, and the ratio of time-steps in which they move towards their tracked parking space when they're exploring. These metrics enable us to quantify the agents' accuracy towards their goal parking space, both when they have a goal parking space and also when they do not. In the case that they do not have a goal parking space but have have a high accuracy towards their tracked parking space, the agent is moving towards an uncommunicated goal (since to the other agents it appears the agent is exploring) which can considered selfish as the agent is pursuing its goal without telling the other agents.
    \item The average number of stopped goals ($p \rightarrow e$ goal transitions) per episode. This metric also enables us to measure the decisiveness of the agents, since a high number of stopped goals per episode means the agents are willing to change their mind quickly and stop pursuing a goal, which can be considered indecisive (and also uncompetitive) behaviour.
    \item The average number of lost goals (times when the agent is no longer able to track their goal parking space) per episode. This metric enables us to measure the accuracy of their agents towards their goal, as well as the likelihood that their chosen goals are likely to not be taken by other agents, since the agent may lose its goal if it moves away from it or if another agent takes it. Hence, a low number of lost goals indicates the agent is picking good goals and that they travel accurately towards them.
\end{itemize}

\subsubsection{Results}

With our measured metrics defined, we train $7$ agents for $16$ million episodes with $r_{\delta g}(p \rightarrow p')$ varying between $0$ and $-0.3$. \Cref{fig:vary-changegoalrew-exploremovetowards} shows the results of the ratio of the agents' exploration and the ratio of their movement towards their goal while both having a goal and exploring, and \Cref{fig:vary-changegoalrew-loststoppedgoals} shows the results of the average number of stopped and lost goals.

\begin{figure}[!h]
\centering
\includegraphics[width=.8\textwidth]{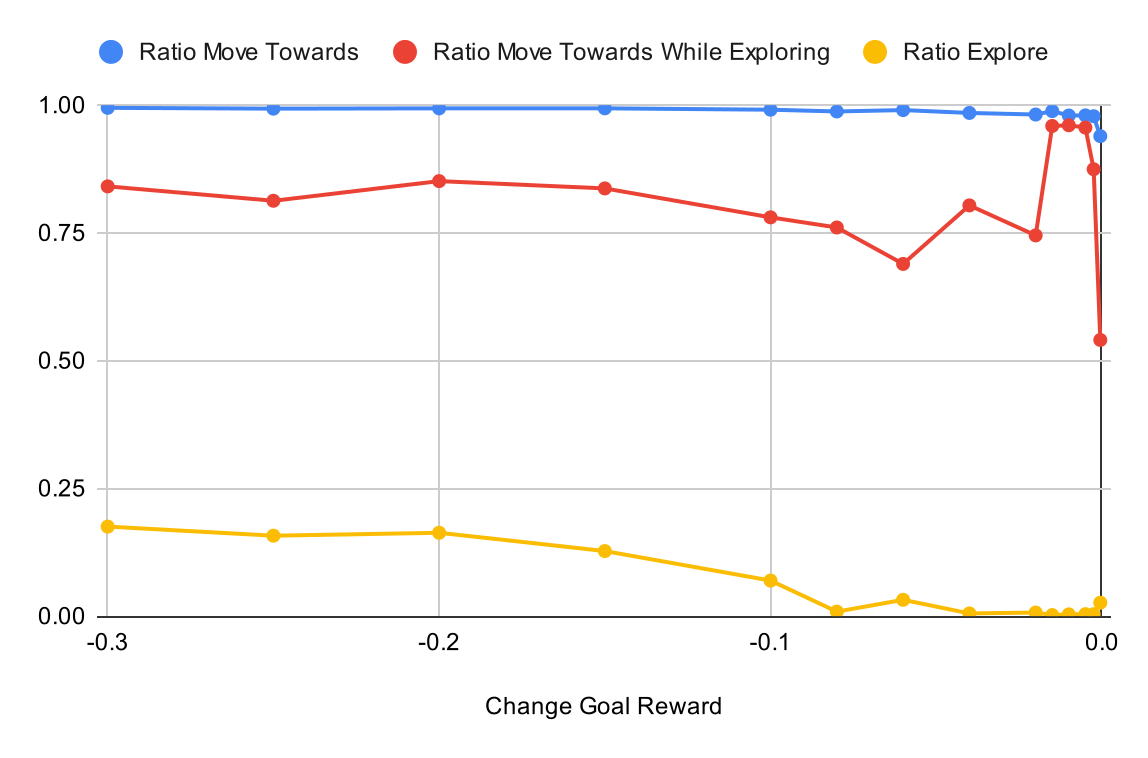}
\caption{Ratio of exploration and movement towards parking spaces when varying the change goal punishment.}
\label{fig:vary-changegoalrew-exploremovetowards}
\end{figure}

\begin{figure}[!h]
\centering
\includegraphics[width=.8\textwidth]{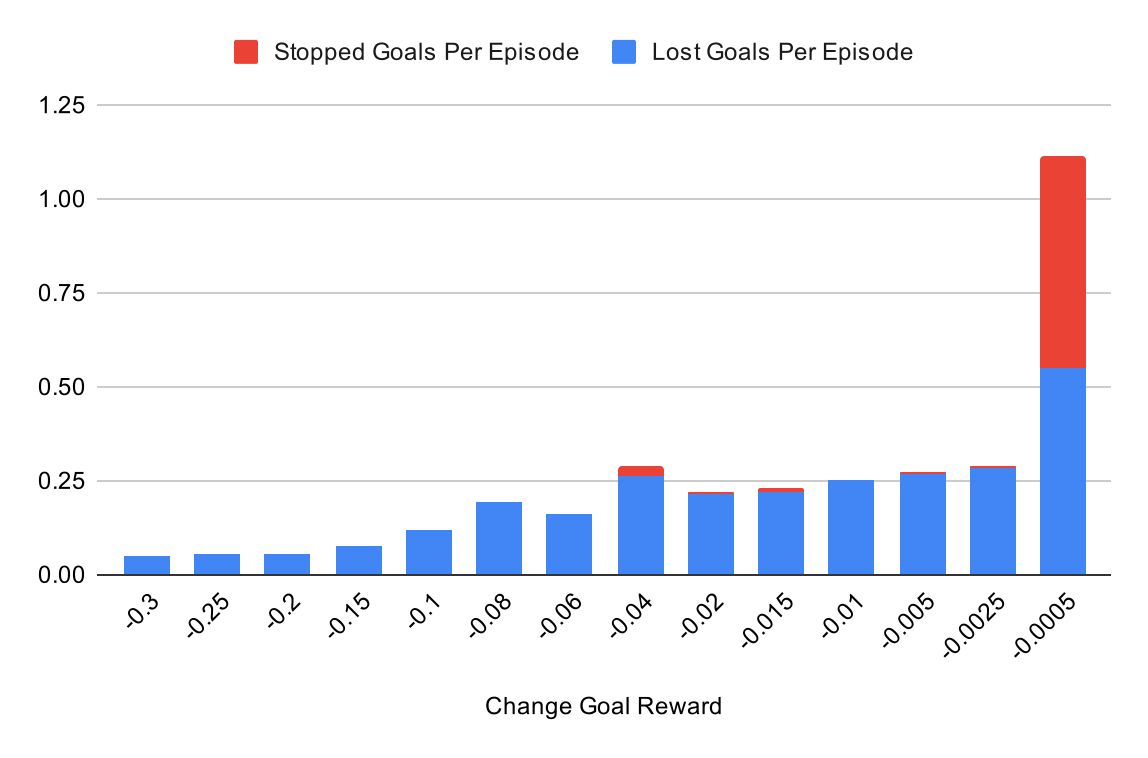}
\caption{Number of lost and stopped goals when varying the change goal punishment.}
\label{fig:vary-changegoalrew-loststoppedgoals}
\end{figure}

Reading the graphs from right to left, it can be seen from \Cref{fig:vary-changegoalrew-exploremovetowards} that as the change goal punishment increases beyond approximately $-0.8$, the agents explore more and travel towards an uncommunicated goal more often. In addition, from \Cref{fig:vary-changegoalrew-loststoppedgoals} it can be seen that the agents lose and change their goals less often. As a result, their exhibited behaviour becomes increasingly selfish, since they spend more time exploring and travelling towards an uncommunicated goal. The agents' goal parking space selections also become better because as $|r_{\delta g}(p \rightarrow p')|$ increases, we see the number of lost goals decrease significantly but the ratio of movement towards chosen goals remain very close to $1$, so the reason for the reduction must be due to their goals being taken by other agents less often.

This is an interesting result, because the learned behaviour may have instead been that the agents explore less and pursue their goal more competitively, however our models exhibit another kind of competitive behaviour whereby the agents are unlikely to choose their goal until they're certain it's a good one.

In our experiments we also find that a too small value of $|r_{\delta g}(p \rightarrow p')|$ causes the agents to get stuck in a halting local optima, with a very low park rate, and that $r_{\delta g}(p \rightarrow p') \leq -0.005$ is required for convergence to a model with a high park rate. However, within that threshold, increases in the punishment do not cause a clear trend to the park rate. As a result, in light of our results, it's important that a sufficiently large value of $|r_{\delta g}(p \rightarrow p')|$ is chosen so the agents learn to park successfully, but the value should not be too large that the agents become too selfish.

\subsection{Collaboration}
\label{sec:results-collaboration}

In addition to competition being an important aspect of the agents' behaviour, collaboration is a dually important aspect. Hence, in this section, we assess the collaborative behaviours exhibited by the agents in our environment with dynamic goals, classifying their collaborative behaviour with our giving way framework defined in \Cref{sec:give-way-schemes}.

\subsubsection{Method}

To assess collaborative behaviour among the agents, we measure their conformity to the way give-way contexts, where conformity to a context is defined in \Cref{eq:context-conformance}. Likewise, we enforce collaborative behaviour among the agents by enforcing different give-way contexts with different strengths, specifically by punishing the agent on a $p \rightarrow p$ goal transition when $G_{\downarrow}(s^*)$ is true (i.e. the agent is pursuing a bad goal, as determined by $\Delta_C$ for the context $C$ being enforced). 

Hence, we vary the bad goal punishment across all of the give-way schemes as defined in \Cref{sec:give-way-schemes}, and measure the agents' conformity to to each context when each give-way scheme is used. Specifically, we use the $(L,L,S)$, $(L,L,A)$, $(G,L,S)$, $(G,L,A)$, $(G,G,S)$ and $(G,G,A)$ give-way schemes, and measure the agents' conformity to the $(L,S)$, $(L,A)$, $(G^+,S)$ and $(G^+,A)$ give-way contexts. We enforce every scheme for completion, and we measure the $G^+$ contexts instead of $G$ contexts because $G^+$ contexts are fully disjoint from the $L$ contexts, preventing the obtained conformities from potentially being misleading (conformity to $G$ contexts may occur as a result of conformity to $L$ contexts, but the agents may not actually be conforming to the $G$ contexts, but rather only their $L$ subset).

In addition to measuring the agents' conformity to the give-way contexts, we also measure the same metrics as in our competition experiments (\Cref{sec:give-way-schemes}) to measure competitive behaviour that may also arise.

\subsubsection{Results}

With our method defined, we train $7$ agents for $16$ million episodes under each give-way scheme, where in each we vary the bad goal punishment between $0$ and $-0.2$, and we measure the conformity to the give-way schemes in each configuration. \Cref{fig:givewaycontext-conformance} shows the results of the agents' conformity to each give-way context (y axis) under each give-way scheme (each individual graph) as the bad goal reward varies (x axis). Hence, when reading the graphs from right to left, the give-way schemes are more greatly enforced. For example, the top-left graph shows the conformities to the give-way contexts as the $(L,L,S)$ give-way scheme is enforced (from right to left), and the bottom-right graph shows their conformities as the $(G,G,A)$ give-way scheme is enforced.

\begin{figure}[!h]
\centering
\includegraphics[width=\textwidth]{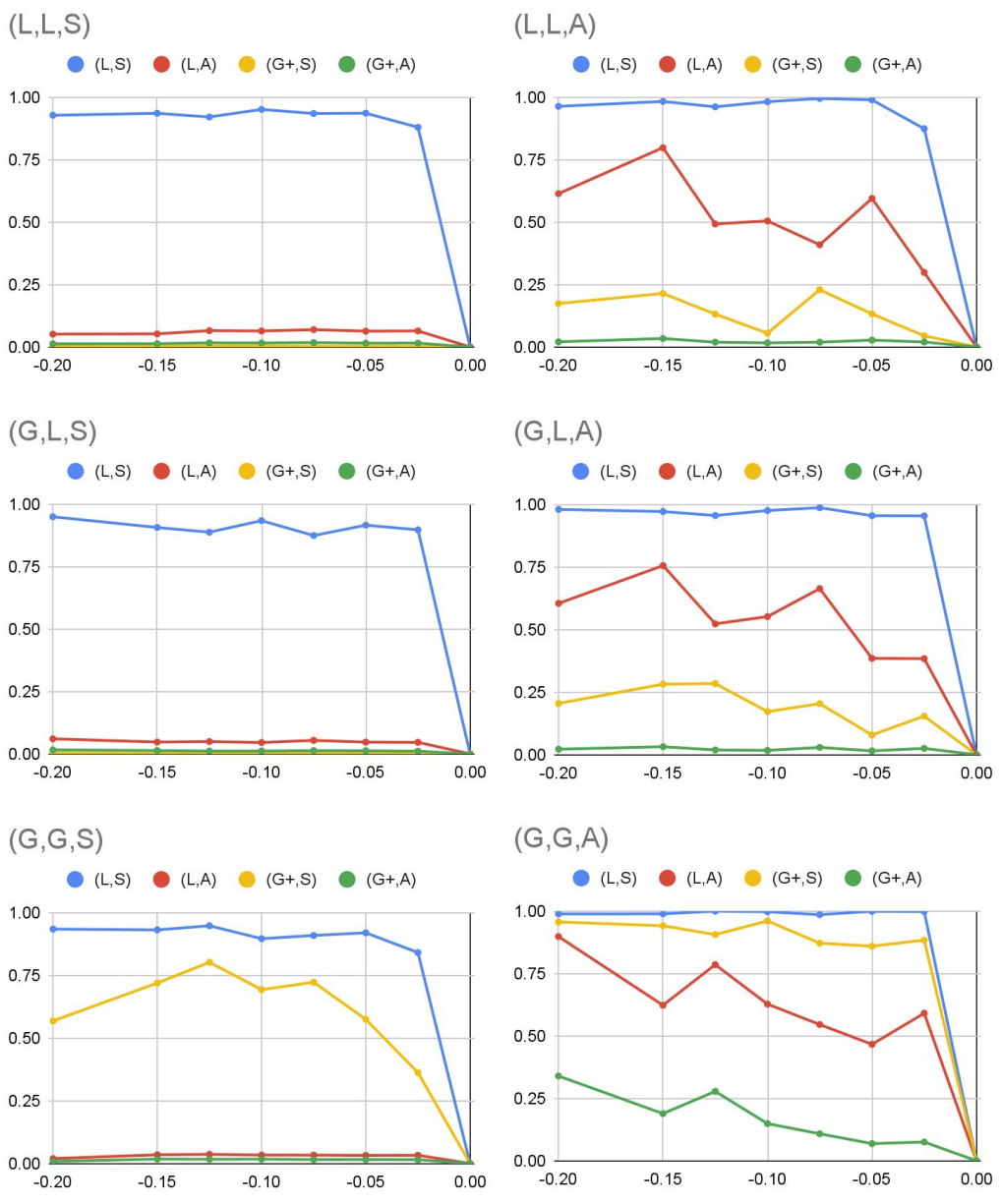}
\caption{Conformity to each give-way context under each give-way scheme when varying the bad goal reward.}
\label{fig:givewaycontext-conformance}
\end{figure}

While we obtain the expected result of the agents conforming to all of the contexts that are a subset of each give-way scheme being enforced, \Cref{fig:givewaycontext-conformance} entails other interesting results.

One mysterious result is that the $(G^+,S)$ context is being conformed to in the $(L,L,A)$ scheme (since its conformity is noticeably larger than in the $(L,L,S)$ scheme). However, the agents do not have access to global information in the $(L,L,A)$ scheme, making the conformity to the non-local context non-trivial, since the agents do not have access to the global information required to conform to it. The author believes such a result is due to to the collective behaviour of the agents under the $(L,L,A)$ scheme, and the author conjectures that such `leaking' conformity may be due to the fact that $(L,A) \not \subseteq (G,S)$ and $(G,S) \not \subseteq (L,A)$ in \Cref{fig:give-way-context-subsets}. However, the reason for such leaking conformity is an open question, and is future work (\Cref{sec:future-work}). 

In addition, the addition of global information to the $(L,L,A)$ scheme does not reduce the conformity to  $(G^+,S)$, despite the agents having sufficient information to determine the scenarios under the $(G^+,S)$ context in the $(G,L,A)$ scheme. Hence, it appears that schemes enforcing giving way to other agents with any goal causes more general behaviour than just conforming to the contexts being enforced. This is further supported by the increase in conformity to $(L,S)$ contexts in the $(L,L,A)$ scheme rather than the $(L,L,S)$ scheme, despite the $(L,L,S)$ scheme specifically enforcing giving way in the context. This may be due to there being more situations to give way under the any goal contexts than the same goal contexts, hence when being enforced the agent receives more punishments and thus has more experiences to reinforce the giving way behaviour, resulting in generally more collaborative behaviour.

Since the $(G,G,A)$ scheme enforces giving way in the $(G,A)$ context which is a superset of all of the other contexts, enforcing the $(G,G,A)$ enables us to compare the conformity of every context when they are all simultaneously enforced. As mentioned in \Cref{sec:context-conformance}, we can thus devise an alternate hierarchy of strength among the contexts, where those that are less likely to be conformed to can be considered `stronger' than those that are not. Hence, from the $(G,G,A)$ scheme in \Cref{fig:givewaycontext-conformance} we see that $(G^+,A) \succ (L,A) \succ (G^+,S) \succ (L,S)$ (where $C_1 \succ C_2$ denotes that $C_1$ is `stronger' than $C_2$). It's important to note here that $(L,A) \succ (G^+,S)$, indicating that $(L,A) \succ (G,S)$ and providing one potential answer to the question of which context between $(L,A)$ and $(G,S)$ is stronger in \Cref{sec:context-heirarchy}. Hence, in terms of conformity, the agents are more likely to conform to same goal contexts than any goal contexts, irrespective of whether the context is non-local or not.

Finally, in all of the graphs one can observe that the agents do not conform any give-way context when the bad goal punishment is equal to $0$. Hence, without enforcing giving way in particular contexts, the agents do not learn to give way to eachother in our environment. While interesting, this result isn't very surprising, because in IPPO each agent is independent and thus self-interested.

\subsubsection{Competition}

In addition to collaborative behaviours being exhibited by the agents under the give-way schemes, they also exhibit competitive behaviours. Specifically, under the $(G,G,A)$ scheme, we observe that the agents learn to exhibit increasingly selfish behaviour as the bad goal punishment increases, as shown by the graphs in \Cref{fig:vary-givewayrew-exploremovetowards} and \Cref{fig:vary-givewayrew-loststoppedgoals}, using the same reasoning as in \Cref{sec:results-competition}. Hence, aspects of both collaborative and competitive behaviours are exhibited when we enforce the give-way schemes, but unfortunately collaboration does not arise in our models without selfishness.

\begin{figure}[!h]
\centering
\includegraphics[width=.8\textwidth]{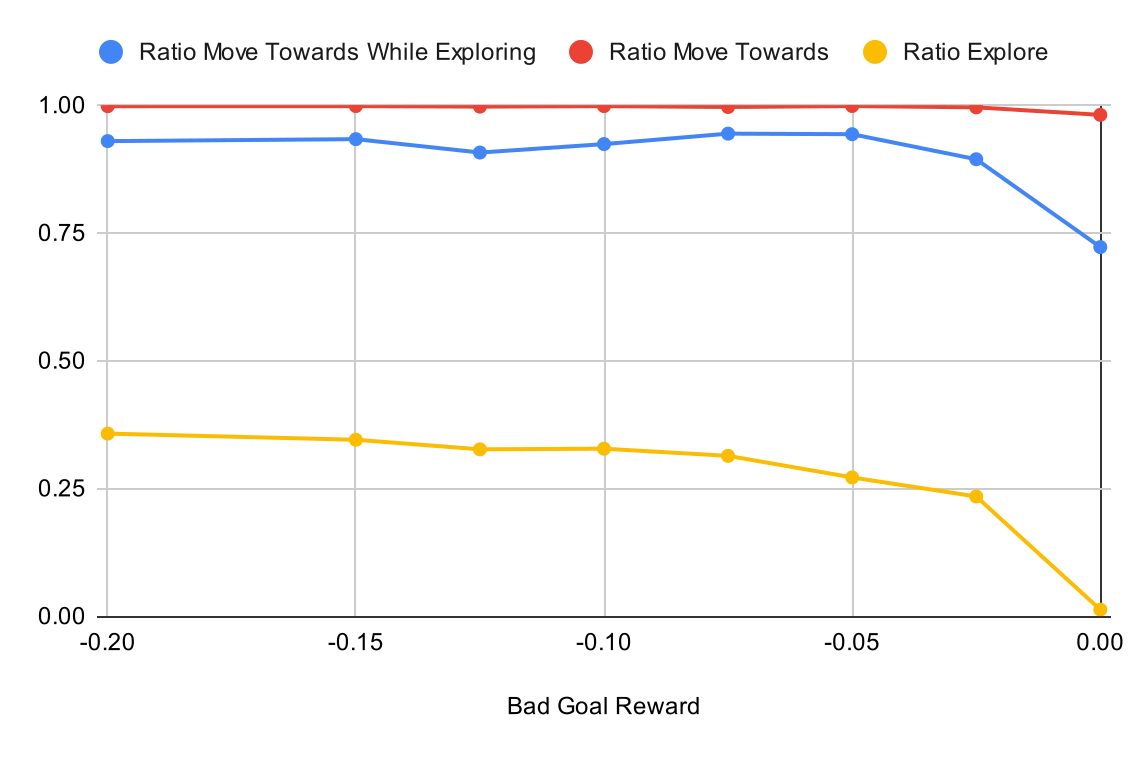}
\caption{Ratio of exploration and movement towards parking spaces when varying the bad goal punishment in the $(G,G,A)$ give-way scheme.}
\label{fig:vary-givewayrew-exploremovetowards}
\end{figure}

\begin{figure}[!h]
\centering
\includegraphics[width=.8\textwidth]{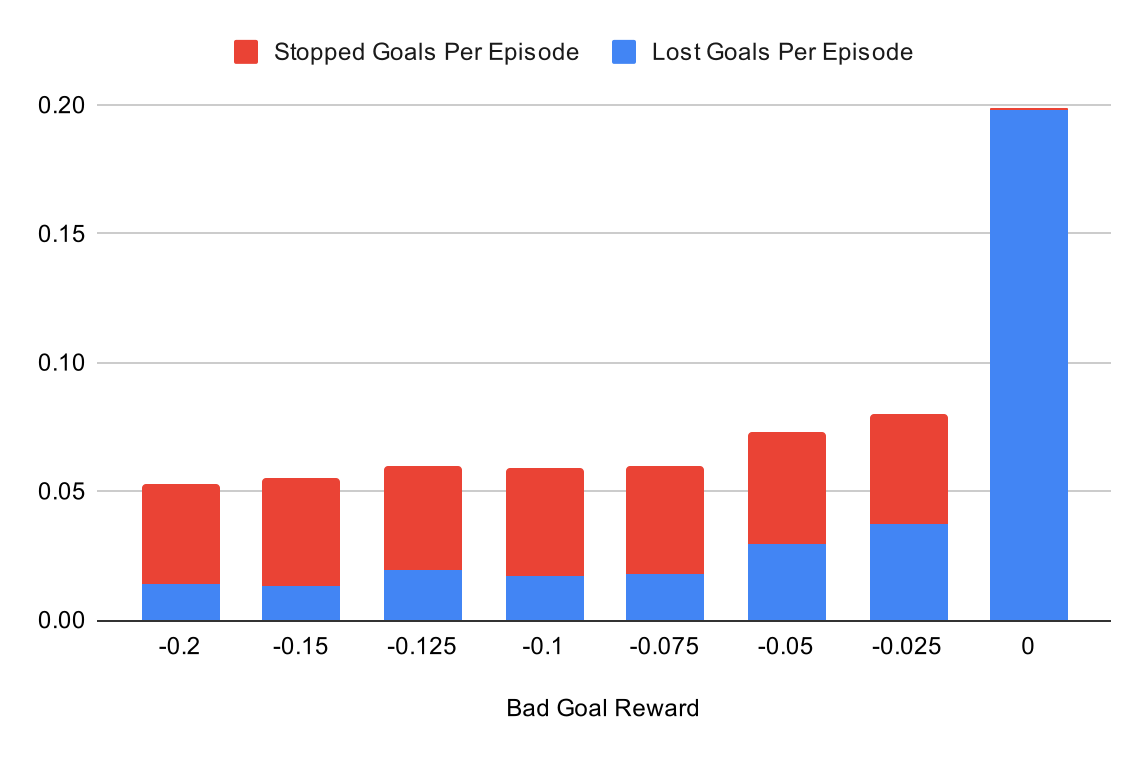}
\caption{Number of lost and stopped goals when varying the bad goal punishment in the $(G,G,A)$ give-way scheme.}
\label{fig:vary-givewayrew-loststoppedgoals}
\end{figure}
\chapter{Conclusions}
\label{ch:conclusions}

In conclusion, our project is a success, as it fully achieves its objectives defined in \Cref{sec:objectives}. In this chapter, we summarise our project's achievements in relation to its objectives, discuss some difficulties and lessons learned, and discuss any of the project's limitations and future work that may follow from it. In what follows, we refer to the objectives by their indexes in \Cref{sec:objectives}.

\section{Summary}

To summarise our project, we designed and and implemented flexible abstractions of a car parking environment suitable for the application of Q-Learning and IPPO, achieving Objective \ref{obj-1}. In our abstraction, the agents may have fixed or dynamic goals, and several properties of the environment can be varied such as its density and dimensionality.

With the environment implemented, we implement and apply Q-Learning to our environment, discovering several limitations in the obstacle avoidance behaviours that can be learned, achieving Objective \ref{obj-2}.

With the limitations of Q-Learning in our environment known, we apply IPPO to our environment, harnessing the superior scalability of deep RL methods. We implement an automated grid-searching and model analysis suite to train and evaluate models at scale, enabling the achievement of Objectives \ref{obj-2}, \ref{obj-3} and \ref{obj-4}.

After tuning PPO's hyperparameters in our environment with fixed goals and obstacles present, we learn models using IPPO parking up to $7$ agents with a very high success rate of at least $98.1\%$, and $2$ agents with a success rate of at least $99.3\%$, achieving Objective \ref{obj-2}. Such an achievement significantly outperforms \citet{ppo-self-parking}'s maximal $96.15\%$ success rate with single agents in their comparable environment, even with the non-stationarity present from multiple agents in our environment. The main difference between our studies is that we share the localised pose information between the cars for collision avoidance, whereas they use several lasers around their agent to detect nearby obstacles, hence we see a significant advantage with our information-sharing approach. Various parameters of the environment are also varied to obtain results relating to its density and dimensionality, shared state (communication) among the agents, as well as the agents' safety within, achieving Objectives \ref{obj-2} and \ref{obj-3}.

Having successfully applied IPPO to our environment with fixed goals, we apply IPPO to our environment with dynamic goals, obtaining a very high success rate of at least $99.5\%$ with 7 agents and obstacles present, achieving Objective \ref{obj-2}. By varying the environment's parameters, we obtain several results relating to competitive and collaborative behaviours exhibited the agents in our environment, achieving Objective \ref{obj-4}. We find that competition arises in the form of selfishness, and that collaboration only arises when explicitly enforced, in the form of giving way to other agents. In addition, we find that collaboration does not arise in our models without selfishness, requiring a trade-off between the two behaviours, and that enforcing giving way in certain contexts causes the agents to give way in other unexpected contexts.

\section{Limitations and Future Work}
\label{sec:future-work}

Our project has several limitations, leading to a variety of potential future work.

\subsubsection{Realism}

Several of the limitations of our project relate to the design of the MDP and its realism. First, the basic motion model employed by the environment's dynamics may not be fully realistic, since there exist more realistic models, such as Ackermann-type steering \cite{ackermann-steering} and the physics implemented in the Gazebo simulator \cite{gazebo}. Thus, future work could potentially enhance the motion models of the agents. In addition, we find from our experiments that the agents do not utilise their full velocity range much, often travelling forwards with full velocity. Such behaviour doesn't align with realistic driving, which often utilises a range of velocities and alternates between forwards and reversing motion while parking for fine-grained motion. Hence, future work may also investigate improving the velocity ranges utilised by the agents. Despite the increase in realism such an extension would provide, the specifics of \textit{how} the agents move themselves from position A to position B in our environment may not be crucial to the investigation of the agents' higher level behaviours relating to goal selection and planning, thus such an extension may not change our higher-level results much.

Another limitation with respect to the realism of our environment is that it lacks road rules and the mechanisms to enforce them. In realistic car parking scenarios, the cars would have to conform to road rules, which may potentially modify the agents' selection of their goals, potentially yielding different results with respect to goal selection and path planning. Thus, extending the MDP to support dynamic road rules and investigating the effect of such road rules on the agents' goal planning is another area of future work of the project. However, the pursuit of such future work may be negated if one assumes that road rules may be relaxed when implementing our model in practice, hence the pursuit of such work again depends on the objectives. 

On a related vein of constraining the agents, additional future work may also entail varying the field of view of the agents in our environment, reducing the range in which they can communicate. Such reductions may entail interesting results with respect to the horizons of the agents' goal planning behaviour. Since our environment supports a dynamic field of view via the \texttt{\_obsNearbyCarsDiameter} parameter, such work wouldn't require much additional implementation.

Finally, our MDP is fully deterministic. Future work may entail relaxing that assumption by introducing stochasticity into the environment's dynamics, and investigating its effect.

\subsubsection{Safety}

There are several limitations relating to the safety of the agents' learned models. First, we assume that each car has the same size hitbox, which may be dangerous for cars with larger hitboxes. Hence, future work may entail making the hitboxes of the cars dynamic and encoding hitbox information into the agents' state. In addition, the agents' obstacle avoidance behaviour is extremely limited with respect to non-car obstacles. Thus, if non-car obstacles are deemed important in our environment, then future work may entail additionally improving obstacle avoidance for non-car obstacles.

Since the agents may make mistakes, additional future work may entail enabling real drivers to override the behaviour, for example if they get dangerously close to a car. So far we have investigated techniques to improve the agents' driving safety by making them keep a threshold distance from other cars, however our MDP does not provide the tools for manual intervention of the agents' behaviours, which may be a useful extension.


\subsubsection{Infrastructure}

Our suite of tools for training models may also be improved. Currently, models can only be trained for a maximum duration as determined by the maximum job length of the department's compute cluster (two days), limiting the training of larger models that take longer to train. Training time may also be reduced by splitting the models into different re-usable parts, so re-learning e.g. locomotion and basic collision avoidance behaviour isn't necessary every time a model is trained. Such re-use may come with the additional benefit of successfully training in environments with many parked cars present, increasing the difficulty of the scenarios the agents are exposed to during training, and potentially increasing the safety of the learned models.

Since our environment has became fairly complex with lots of different parameters, future work may also wish to use more sophisticated techniques to optimise the environment's parameters (rather than the basic technique of grid-searching used in our study). In particular, with respect to optimising rewards, several techniques in the literature may be utilised. One such technique is the use of genetic algorithms, as proposed by \citet{genetic-algos} (2019). Another is the optimisation of dynamic rewards which change as training progresses, as proposed by \citet{dynamic-rewards} (2020). Finally, \citet{multi-reward-opt} (2018) propose a novel technique of splitting the loss function optimised in deep RL into several individual ones for each reward, and optimising each loss function simultaneously.


\subsubsection{Collaboration}

Finally, there is a variety of potential future work in resolving the questions that arised from our results relating to the agents' collaborative behaviours in \Cref{sec:results-collaboration}. In particular, we desire a rigorous explanation of the `leaking' conformity to the $(G^+,S)$ give-way context under the $(L,L,A)$ scheme, providing an answer to our conjecture proposed in \Cref{sec:results-collaboration}. One may also wish to extend the give-way contexts to harness additional information available in our MDP, such as angular information, to increase the use of our MDP's information. Alternatively, deep learning may be exploited to identify new give-way contexts without them being user-defined.

As a large extension, it may be interesting to investigate whether other MARL algorithms learn to conform to give-way contexts without them being explicitly enforced, unlike IPPO as discovered in our study. Some candidate algorithms for such investigations may be MAPPO \cite{ppo-in-ma-games} due to its close ties to IPPO but the agents are not fully independent, and MA-POCA \cite{ml-agents}: a true multi-agent algorithm invented and implemented in ML-Agents, which may be harnessed without significant additional implementation.



\bibliographystyle{common/plainnat}
\bibliography{dissertation}

\appendix

\chapter{Spatial Symmetry of Local Object Pose Encoding}
\label{appendix:local-object-pose-encoding}

This appendix shows how our encoding of the pose of nearby objects as described in \Cref{sec:local-object-pose-encoding} preserves $d$ and $\theta$ after both cars change their position by the same constants $a$ and $b$ on either axis, and preserves $\delta \theta$ after both cars change their rotations by the same constant $c$.

$d$ is preserved since if the cars are at a different position only differing by a constant $a$ and $b$ on either axis, the new positions of the green and grey car are $(x+a,y+b)$ and $(x'+a,y'+b)$ respectively, thus the L2 distance between them is:
$$\sqrt{(x'+a-(x+a))^2+(y'+b - (y+b))^2)} = \sqrt{(x'-x)^2+(y'-y)^2}$$, which is the same as the L2 distance between the original positions of the cars, yielding the same value of $d$. $\theta$ is also preserved, since the global angle between them is $$\arctan(\frac{(y'+b)-(y+b)}{(x'+a)-(x+a)}) = \arctan(\frac{y'-y}{x'-x})$$, which is equal to $\theta$ after taking the angle relative to the green car's global rotation. 

Likewise, $\delta \theta$ is preserved, since if both cars both rotate by a constant $c$, and their global rotations are $\theta_1$ and $\theta_2$ respectively, then their new global rotations are $\theta_1 + c$ and $\theta_2 + c$ respectively, thus the difference in their global rotation is $(\theta_1 + c) - (\theta_2 + c) = \theta_1 - \theta_2 = \delta \theta$, preserving $\delta \theta$.

\chapter{Deadlock Free Evolution Proof}
\label{appendix:proof-no-deadlock}

This appendix informally proves \Cref{th:no-deadlock}.

\begin{proof}
First, observe that the position of each agent is randomly distributed, since they're spawned in random positions and travel towards their randomly chosen parking space to park in it.

Next, observe that every parked car in the environment is eventually moved, because we move the furthest parked car from all of the agents in the environment. The positions of the agents in our environment are randomly distributed, thus each parked car has a chance to be the furthest parked car from all of the agents.

When a parked car is moved, it is moved into an agent's parking space when the agent successfully parked in it. Thus, there is a chance a parked car is moved to a parking space that is outside the cluster, because the agent may have a goal parking space outside of the cluster, since we are guaranteed that such a space exists.

Thus, every parked car in the cluster has a chance to be moved to a parking space that is outside the cluster, eventually breaking it.
\end{proof}

\chapter{Implementation Tables}
\label{appendix:impl-tables}

This appendix contains the tables for \Cref{ch:implementation}.

\begin{table}[]
\resizebox{\linewidth}{!}{%
\begin{tabular}{|l|l|l|l|}
\hline
\textbf{Parameter}                                                                                    & \textbf{Type} & \textbf{\begin{tabular}[c]{@{}l@{}}Default \\ Value\end{tabular}} & \textbf{Description}                                                                                                                                                                                                                                   \\ \hline
\texttt{\_positionGranularity}                                                                        & int           & 1                                                                 & $G_p$                                                                                                                                                                                                                                                  \\ \hline
\texttt{\_velocityGranularity}                                                                        & int           & 1                                                                 & $G_v$                                                                                                                                                                                                                                                  \\ \hline
\texttt{\_thetaGranularity}                                                                           & int           & 8                                                                 & $G_{\theta}$                                                                                                                                                                                                                                           \\ \hline
\texttt{\_maxVelocityMagnitude}                                                                       & int           & 1                                                                 & ${v_{max}}^+$                                                                                                                                                                                                                                          \\ \hline
\texttt{\_minVelocityMagnitude}                                                                       & int           & 1                                                                 & ${v_{max}}^-$                                                                                                                                                                                                                                          \\ \hline
\texttt{\_maxDeltaVMagnitude}                                                                         & int           & 1                                                                 & $a_{max}$                                                                                                                                                                                                                                              \\ \hline
\texttt{\_maxDeltaThetaMagnitude}                                                                     & int           & 1                                                                 & $\omega_{max}$                                                                                                                                                                                                                                         \\ \hline
\texttt{\_numAgents}                                                                                  & int           & 1                                                                 & Number of agents in the environment.                                                                                                                                                                                                                   \\ \hline
\texttt{\_normalizeObs}                                                                               & bool          & false                                                             & Whether the state values are normalised to the range $[0,1]$.                                                                                                                                                                                          \\ \hline
\texttt{\_debugObs}                                                                                   & bool          & false                                                             & Whether to print the values of the observations to the console.                                                                                                                                                                                        \\ \hline
\texttt{\_visualiseNearbyCars}                                                                        & bool          & false                                                             & Show lines of the agents' tracked cars and parking spaces.                                                                                                                                                                                             \\ \hline
\texttt{\_numParkedCars}                                                                              & int           & 0                                                                 & Number of parked cars.                                                                                                                                                                                                                                 \\ \hline
\texttt{\_obsDist}                                                                                    & bool          & false                                                             & Whether to encode the agents' distance to their goal in the state.                                                                                                                                                                                     \\ \hline
\texttt{\_obsAngle}                                                                                   & bool          & true                                                              & Whether to encode the agents' angle to their goal in the state.                                                                                                                                                                                        \\ \hline
\texttt{\_obsRings}                                                                                   & bool          & false                                                             & Whether the agents have rings around them.                                                                                                                                                                                                             \\ \hline
\texttt{\_ringMaxNumObjTrack}                                                                         & int           & 0                                                                 & Maximum number of objects each ring can count.                                                                                                                                                                                                         \\ \hline
\texttt{ringDiams}                                                                                    & int[]         & []                                                                & \begin{tabular}[c]{@{}l@{}}Diameters of the rings (in world units). \\ Also can be populated via \texttt{\_rd<...>} below.\end{tabular}                                                                                                                \\ \hline
\texttt{\_rd<...>}                                                                                    & int           & 0                                                                 & \begin{tabular}[c]{@{}l@{}}Diameters of each ring in separate \texttt{\_rd<...>} parameters \\ (with different \texttt{<...>} suffixes). \\ \texttt{ringDiams} is populated with these values at run-time.\end{tabular}                                \\ \hline
\texttt{\_obsGoalDeltaPose}                                                                           & bool          & false                                                             & \begin{tabular}[c]{@{}l@{}}Whether to encode the agents' difference in global rotation \\ in the state.\end{tabular}                                                                                                                                   \\ \hline
\texttt{\_ringNumPrevObs}                                                                             & int           & 0                                                                 & Number of historical ring states to encode in the state.                                                                                                                                                                                               \\ \hline
\texttt{\_ringOnlyWall}                                                                               & bool          & false                                                             & Whether the rings should only detect walls.                                                                                                                                                                                                            \\ \hline
\texttt{\_obsNearbyCars}                                                                              & bool          & false                                                             & \begin{tabular}[c]{@{}l@{}}Whether to encode the state of tracked cars \\ (if false, $n_{track} = 0$).\end{tabular}                                                                                                                                     \\ \hline
\texttt{\_obsNearbyCarsCount}                                                                         & int           & 0                                                                 & $n_{track}$                                                                                                                                                                                                                                            \\ \hline
\texttt{\_obsNearbyCarsDiameter}                                                                      & int           & 0                                                                 & Maximum distance nearby cars can be tracked (in world units).                                                                                                                                                                                          \\ \hline
\texttt{spawnCloseRatio}                                                                              & float         & 0                                                                 & Ratio in which an agent is spawned near its goal parking space.                                                                                                                                                                                        \\ \hline
\texttt{\_spawnCloseDist}                                                                             & int           & 10                                                                & \begin{tabular}[c]{@{}l@{}}Maximum distance (in world units) to spawn the agent away \\ from its goal parking space when it's to be spawned near it.\end{tabular}                                                                                      \\ \hline
\texttt{carSpawnMinDistance}                                                                          & int           & 9                                                                 & \begin{tabular}[c]{@{}l@{}}Minimum distance (in MDP units) to an agent in which the \\ agent can be legally spawned.\end{tabular}                                                                                                                      \\ \hline
\texttt{\_maxSteps}                                                                                   & int           & 85                                                                & Maximum number of steps in an agent's episode.                                                                                                                                                                                                         \\ \hline
\texttt{\_obsNearbyCarsGoal}                                                                          & bool          & false                                                             & Whether to share the goals of tracked cars.                                                                                                                                                                                                            \\ \hline
\texttt{\_obsNearbyCarsVelocity}                                                                      & bool          & false                                                             & Whether to share the velocities of tracked cars.                                                                                                                                                                                                       \\ \hline
\texttt{spawnCrashRatio}                                                                              & float         & 0                                                                 & Ratio in which an agent is spawned to crash with another agent.                                                                                                                                                                                        \\ \hline
\texttt{spawnCrashTargetAgentMinDist}                                                                 & float         & 10                                                                & \begin{tabular}[c]{@{}l@{}}Minimum distance (in MDP units) an agent must be to its goal \\ for it to be picked as a candidate to try and spawn a crash with.\end{tabular}                                                                              \\ \hline
\end{tabular}
}%
\caption{Environment parameters.}
\label{tab:env-params}
\end{table}

\begin{table}[]
\resizebox{\linewidth}{!}{%
\begin{tabular}{|l|l|l|l|}
\hline
\textbf{Parameter}                                                                                    & \textbf{Type} & \textbf{\begin{tabular}[c]{@{}l@{}}Default \\ Value\end{tabular}} & \textbf{Description}                                                                                                                                                                              \\ \hline

\texttt{\_dynamicGoals}                                                                               & bool          & false                                                             & Whether the agents have dynamic goals.                                                                                                                                                                                                                 \\ \hline
\texttt{\_obsNearbyParkingSpotsCount}                                                                 & int           & 0                                                                 & $n_{park}$                                                                                                                                                                                                                                             \\ \hline
\texttt{rewDeltaGoalContinueExp}                                                                      & float         & 0                                                                 & $r_{\delta g}(e \rightarrow e)$                                                                                                                                                                                                                   \\ \hline
\texttt{rewDeltaGoalStopExp}                                                                          & float         & 0                                                                 & $r_{\delta g}(e \rightarrow p)$                                                                                                                                                                                                                   \\ \hline
\texttt{rewDeltaGoalContinueGoal}                                                                     & float         & 0                                                                 & $r_{\delta g}(p \rightarrow p)$                                                                                                                                                                                                                   \\ \hline
\texttt{rewDeltaGoalDiffGoal}                                                                         & float         & 0                                                                 & $r_{\delta g}(p \rightarrow p')$                                                                                                                                                                                                                  \\ \hline
\texttt{\_rewDeltaGoalStopGoal}                                                                       & float         & 0                                                                 & \begin{tabular}[c]{@{}l@{}}$r_{\delta g}(p \rightarrow e)$. \\ If equal to -1, we compute it via \cref{eq:rew-change-goal-equiv-simplified}.\end{tabular}                                                                                         \\ \hline
\begin{tabular}[c]{@{}l@{}}\texttt{rewDeltaGoal}\\ \texttt{ContinueGoalBetterOtherAgent}\end{tabular} & float         & 0                                                                 & \begin{tabular}[c]{@{}l@{}}$r_{\delta g}(p \rightarrow p)$ when the \\ current goal is a bad goal.\end{tabular}                                                                                                                                   \\ \hline
\texttt{rewReachGoal}                                                                                 & float         & 10                                                                & Reward for reaching the goal parking space.                                                                                                                                                                                                            \\ \hline
\texttt{rewCrash}                                                                                     & float         & 10                                                                & Magnitude of the punishment for crashing with an obstacle.                                                                                                                                                                                             \\ \hline
\texttt{rewTimeSum}                                                                                   & float         & 0.5                                                               & Magnitude of the sum of the time-based punishment.                                                                                                                                                                                                     \\ \hline
\texttt{rewReverseSum}                                                                                & float         & 0                                                                 & Sum of the reversing punishment.                                                                                                                                                                                                                       \\ \hline
\texttt{rewDistSum}                                                                                   & float         & 1                                                                 & Magnitude of the sum of the move towards goal reward.                                                                                                                                                                                                  \\ \hline
\texttt{\_obsParkingSpotClosestAgent}                                                                 & bool          & false                                                             & \begin{tabular}[c]{@{}l@{}}Whether to encode global information of the agent's \\ current goal with respect to all agents (minimum distance \\ across all agents to the goal).\end{tabular}                                                            \\ \hline
\texttt{\_obsParkingSpotClosestGoalAgent}                                                             & bool          & false                                                             & \begin{tabular}[c]{@{}l@{}}Whether to encode global information of the agent's \\ current goal with respect to only agents with the same goal \\ (minimum distance to the goal \\ across all agents with the same goal).\end{tabular}                     \\ \hline
\texttt{\_punishBetterOtherAgentLocal}                                                                & bool          & false                                                             & Whether to enforce the local give-way scheme (global if false).                                                                                                                                                                                        \\ \hline
\texttt{\_punishBetterOtherGoalAgent}                                                                 & bool          & false                                                             & \begin{tabular}[c]{@{}l@{}}Whether to enforce the same goal give-way scheme \\ (any goal scheme if false).\end{tabular}                                                                                                                                \\ \hline
\texttt{\_numStepsTrain}                                                                              & int           & 0                                                                 & Total number of steps in the training phase.                                                                                                                                                                                                           \\ \hline
\texttt{carScaleTrain}                                                                                & float         & 1                                                                 & Multiplier of the scale of the cars during training.                                                                                                                                                                                                   \\ \hline
\texttt{rewFinalVelocitySum}                                                                          & float         & 0                                                                 & \begin{tabular}[c]{@{}l@{}}Maximum punishment for the agent's velocity upon parking. \\ We multiply the fraction between the min and max of the \\ magnitude of the agent's velocity upon parking by this \\ value to get the punishment.\end{tabular} \\ \hline

\texttt{rewDeltaThetaSum}                                                                          & float         & 0                                                                 & \begin{tabular}[c]{@{}l@{}}Sum of the magnitude of the smoothness punishment \\ to apply to the agent at each time-step.\end{tabular} \\ \hline

\texttt{\_rewDeltaThetaVelMult}                                                                          & bool         & false                                                                 & \begin{tabular}[c]{@{}l@{}}Whether the magnitude of the agent's velocity should \\ be included in the smoothness punishment at each time-step. \\ If not, only the agent's angular velocity is considered. \end{tabular} \\ \hline

\end{tabular}
}%
\caption{Environment parameters (continued).}
\label{tab:env-params-2}
\end{table}

\begin{table}[]
\resizebox{\linewidth}{!}{%
\begin{tabular}{|l|l|l|l|}
\hline
\textbf{CSV Metric ID}                                               & \textbf{Description}                                                                                                                                                    & \textbf{\begin{tabular}[c]{@{}l@{}}TensorBoard\\ Metric Path\end{tabular}}            & \textbf{\begin{tabular}[c]{@{}l@{}}Computation \\Method\end{tabular}}                                                                   \\ \hline
Max Steps                                                            & \begin{tabular}[c]{@{}l@{}}Number of steps\\ the agents were\\ evaluated over.\end{tabular}                                                                             & Via metadata.                                                                                   & Metadata.                                                                                      \\ \hline
Final Mean Reward                                                    & \begin{tabular}[c]{@{}l@{}}Mean cumulative reward\\ received by an\\ agent in its episode.\end{tabular}                                                                 & \begin{tabular}[c]{@{}l@{}}Environment\\ /Cumulative Reward\end{tabular}              & \texttt{mean\_metric\_period}                                                                 \\ \hline
\begin{tabular}[c]{@{}l@{}}Final Mean \\ Episode Length\end{tabular} & \begin{tabular}[c]{@{}l@{}}Mean number of steps\\ in each of the\\ agents' episodes.\end{tabular}                                                                       & \begin{tabular}[c]{@{}l@{}}Environment\\ /Episode Length\end{tabular}                 & \texttt{mean\_metric\_period}                                                                 \\ \hline
c                                                                    & Crash rate.                                                                                                                                                             & \begin{tabular}[c]{@{}l@{}}Metrics\\ /Total Crashes\end{tabular}                      & \texttt{per\_eps}                                                                             \\ \hline
p                                                                    & Parking rate.                                                                                                                                                           & \begin{tabular}[c]{@{}l@{}}Metrics\\ /Total Reached Goal\end{tabular}                 & \texttt{per\_eps}                                                                             \\ \hline
h                                                                    & \begin{tabular}[c]{@{}l@{}}Halt rate\\ (halting is reaching\\ $\tau$ time-steps and\\ not crashing or parking).\end{tabular}                                            & \begin{tabular}[c]{@{}l@{}}Metrics\\ /Total Halted\end{tabular}                       & \texttt{per\_eps}                                                                             \\ \hline
Num Episodes                                                         & \begin{tabular}[c]{@{}l@{}}Number of episodes\\ the agents were\\ evaluated over.\end{tabular}                                                                          & \begin{tabular}[c]{@{}l@{}}Metrics\\ /Num Episodes\end{tabular}                       & \begin{tabular}[c]{@{}l@{}}\texttt{last\_minus\_start}\\ \texttt{\_metricperiod}\end{tabular} \\ \hline
Total Crashes Car                                                    & \begin{tabular}[c]{@{}l@{}}Number of crashes\\ with cars.\end{tabular}                                                                                                  & \begin{tabular}[c]{@{}l@{}}Metrics\\ /Total Crashes Car\end{tabular}                  & \begin{tabular}[c]{@{}l@{}}\texttt{last\_minus\_start}\\ \texttt{\_metricperiod}\end{tabular} \\ \hline
Total Crashes Wall                                                   & \begin{tabular}[c]{@{}l@{}}Number of crashes\\ with walls.\end{tabular}                                                                                                 & \begin{tabular}[c]{@{}l@{}}Metrics\\ /Total Crashes Wall\end{tabular}                 & \begin{tabular}[c]{@{}l@{}}\texttt{last\_minus\_start}\\ \texttt{\_metricperiod}\end{tabular} \\ \hline
Total Crashes StaticCar                                              & \begin{tabular}[c]{@{}l@{}}Number of crashes\\ with parked cars.\end{tabular}                                                                                           & \begin{tabular}[c]{@{}l@{}}Metrics\\ /Total Crashes StaticCar\end{tabular}            & \begin{tabular}[c]{@{}l@{}}\texttt{last\_minus\_start}\\ \texttt{\_metricperiod}\end{tabular} \\ \hline
Num Lost Goal PerEps                                                 & \begin{tabular}[c]{@{}l@{}}Average number of\\ lost goals per episode.\end{tabular}                                                                                     & \begin{tabular}[c]{@{}l@{}}Metrics\\ /Num Lost Goal\end{tabular}                      & \texttt{per\_eps}                                                                             \\ \hline
Num Stop Explore PerEps                                              & \begin{tabular}[c]{@{}l@{}}Average number of\\ stop explores per\\ episode.\end{tabular}                                                                                & \begin{tabular}[c]{@{}l@{}}Metrics\\ /Num Stop Explore\end{tabular}                   & \texttt{per\_eps}                                                                             \\ \hline
Num Stop Goal PerEps                                                 & \begin{tabular}[c]{@{}l@{}}Average number of\\ stopped goals per\\ episode.\end{tabular}                                                                                & \begin{tabular}[c]{@{}l@{}}Metrics\\ /Num Stop Goal\end{tabular}                      & \texttt{per\_eps}                                                                             \\ \hline
Num Change Goal PerEps                                               & \begin{tabular}[c]{@{}l@{}}Average number of\\ changed goals per\\ episode.\end{tabular}                                                                                & \begin{tabular}[c]{@{}l@{}}Metrics\\ /Num Change Goal\end{tabular}                    & \texttt{per\_eps}                                                                                                                                         \\ \hline
\end{tabular}
}
\caption{\texttt{analysis.py} metrics.}
\label{tab:analysis-metrics}
\end{table}

\begin{table}[]
\resizebox{\linewidth}{!}{%
\begin{tabular}{|l|l|l|l|}
\hline
\textbf{CSV Metric ID}                                               & \textbf{Description}                                                                                                                                                    & \textbf{\begin{tabular}[c]{@{}l@{}}TensorBoard\\ Metric Path\end{tabular}}            & \textbf{\begin{tabular}[c]{@{}l@{}}Computation \\Method\end{tabular}}     

\\ \hline
GaveWayLocalAnyGoal                                                  & \begin{tabular}[c]{@{}l@{}}Conformity to the\\ $(L,A)$ give-way context.\end{tabular}                                                                                  & \begin{tabular}[c]{@{}l@{}}Via\\ \texttt{context\_conformance}\end{tabular}           & \texttt{context\_conformance}                                                                 \\ \hline
GaveWayGlobalAnyGoal                                                 & \begin{tabular}[c]{@{}l@{}}Conformity to the\\ $(G,A)$ give-way context.\end{tabular}                                                                                  & \begin{tabular}[c]{@{}l@{}}Via\\ \texttt{context\_conformance}\end{tabular}           & \texttt{context\_conformance}                                                                 \\ \hline
GaveWayLocalSameGoal                                                 & \begin{tabular}[c]{@{}l@{}}Conformity to the\\ $(L,S)$ give-way context.\end{tabular}                                                                                  & \begin{tabular}[c]{@{}l@{}}Via\\ \texttt{context\_conformance}\end{tabular}           & \texttt{context\_conformance}                                                                 \\ \hline
GaveWayGlobalSameGoal                                                & \begin{tabular}[c]{@{}l@{}}Conformity to the\\ $(G,S)$ give-way context.\end{tabular}                                                                                  & \begin{tabular}[c]{@{}l@{}}Via\\ \texttt{context\_conformance}\end{tabular}           & \texttt{context\_conformance}                                                                 \\ \hline
GaveWayNonLocalSameGoal                                              & \begin{tabular}[c]{@{}l@{}}Conformity to the\\ $(G^+,S)$ give-way context.\end{tabular}                                                                                & \begin{tabular}[c]{@{}l@{}}Via\\ \texttt{context\_conformance}\end{tabular}           & \texttt{context\_conformance}                                                                 \\ \hline
GaveWayNonLocalAnyGoal                                               & \begin{tabular}[c]{@{}l@{}}Conformity to the\\ $(G^+,A)$ give-way context.\end{tabular}                                                                                & \begin{tabular}[c]{@{}l@{}}Via\\ \texttt{context\_conformance}\end{tabular}           & \texttt{context\_conformance}     

\\ \hline
DeltaThetaAvg                                                        & \begin{tabular}[c]{@{}l@{}}Average value of\\ $\omega$ (angular velocity).\end{tabular}                                                                                 & \begin{tabular}[c]{@{}l@{}}Metrics\\ /DeltaThetaAvg\end{tabular}                      & \texttt{mean\_metric\_period}                                                                 \\ \hline
VelocityAvg                                                          & \begin{tabular}[c]{@{}l@{}}Average value of\\ $v_c$ (velocity).\end{tabular}                                                                                            & \begin{tabular}[c]{@{}l@{}}Metrics\\ /VelocityAvg\end{tabular}                        & \texttt{mean\_metric\_period}                                                                 \\ \hline
NearestCarDistAvg                                                    & \begin{tabular}[c]{@{}l@{}}Average distance\\ of the nearest car\\ to an agent.\end{tabular}                                                                            & \begin{tabular}[c]{@{}l@{}}Metrics\\ /NearestCarDistAvg\end{tabular}                  & \texttt{mean\_metric\_period}                                                                 \\ \hline
RatioExplorePerEps                                                   & \begin{tabular}[c]{@{}l@{}}Average ratio of\\ time-steps in which\\ the agents explore\\ per episode.\end{tabular}                                                      & \begin{tabular}[c]{@{}l@{}}Metrics\\ /RatioExplorePerEps\end{tabular}                 & \texttt{mean\_metric\_period}                                                                 \\ \hline
RatioGoalPerEps                                                      & $1 - \texttt{RatioExplorePerEps}$                                                                                                                                       & \begin{tabular}[c]{@{}l@{}}Metrics\\ /RatioGoalPerEps\end{tabular}                    & \texttt{mean\_metric\_period}                                                                 \\ \hline
RatioMoveTowardsPerEps                                               & \begin{tabular}[c]{@{}l@{}}Average ratio of\\ time-steps in which\\ the agents move\\ towards their goal\\ parking space\\ per episode.\end{tabular}                    & \begin{tabular}[c]{@{}l@{}}Metrics\\ /RatioMoveTowardsPerEps\end{tabular}             & \texttt{mean\_metric\_period}                                                                 \\ \hline
RatioMoveTowardsExploringPerEps                                      & \begin{tabular}[c]{@{}l@{}}Average ratio of\\ time-steps in which\\ the agents move\\ towards their nearest\\ parking space while\\ exploring per episode.\end{tabular} & \begin{tabular}[c]{@{}l@{}}Metrics\\ /RatioMoveTowards\\ ExploringPerEps\end{tabular} & \texttt{mean\_metric\_period}                                                                 \\ \hline
Park Velocity                                                        & \begin{tabular}[c]{@{}l@{}}Average velocity of\\ the agents upon\\ parking.\end{tabular}                                                                                & \begin{tabular}[c]{@{}l@{}}Metrics\\ /Park Velocity\end{tabular}                      & \texttt{mean\_metric\_period}                                                                 \\ \hline
\end{tabular}
}
\caption{\texttt{analysis.py} metrics (continued).}
\label{tab:analysis-metrics-2}
\end{table}

\begin{table}[]
\centering
\begin{tabular}{|l|l|}
\hline
\textbf{\begin{tabular}[c]{@{}l@{}}Argument\\ Index\end{tabular}} & \textbf{\begin{tabular}[c]{@{}l@{}}Argument Description\end{tabular}}                                                                                                                                                                           \\ \hline
1                                                                 & Path to the directory of models.                                                                                                                                                                                                                  \\ \hline
2                                                                 & \begin{tabular}[c]{@{}l@{}}Number of steps in the training period. \\ Metrics aren't collected in this period - \\ we only collect metrics after this period \\ (since after we're in the metric period).\end{tabular}                            \\ \hline
3                                                                 & \begin{tabular}[c]{@{}l@{}}Path to the output CSV file \\ containing the metrics\end{tabular}                                                                                                                                                     \\ \hline
4                                                                 & \begin{tabular}[c]{@{}l@{}}Expected number of models in the \\ directory of models. If the number of \\ models doesn't equal this value then \\ the script prematurely terminates since \\ the models haven't finished learning yet.\end{tabular} \\ \hline
5                                                                 & \begin{tabular}[c]{@{}l@{}}Paths to parameters in the\\ base PPO YAML configuration file\\ we wish to also record as a metric.\end{tabular}                                                                                                       \\ \hline
\end{tabular}
\caption{\texttt{analysis.py} arguments.}
\label{tab:analysis-args}
\end{table}

\begin{table}[]
\centering
\begin{tabular}{|l|l|}
\hline
\textbf{Argument}         & \textbf{\begin{tabular}[c]{@{}l@{}}Argument Description\end{tabular}}                                                                                                                                                       \\ \hline
\texttt{-{}-alpha}        & Q-Learning learning rate.                                                                                                                                                                                                     \\ \hline
\texttt{-{}-gamma}        & Q-Learning discount factor.                                                                                                                                                                                                   \\ \hline
\texttt{-{}-epsilon}      & Q-Learning epsilon greedy factor.                                                                                                                                                                                             \\ \hline
\texttt{-{}-episodes}     & Number of episodes to train.                                                                                                                                                                                                  \\ \hline
\texttt{-{}-episodeseval} & \begin{tabular}[c]{@{}l@{}}Number of episodes to evaluate the agent \\ ($\epsilon = 0$).\end{tabular}                                                                                                                         \\ \hline
\texttt{-{}-dump}         & \begin{tabular}[c]{@{}l@{}}Interval of number of episodes to \\ print the current results to \\ standard output.\end{tabular}                                                                                                 \\ \hline
\texttt{-{}-build}        & Directory of the Unity environment build.                                                                                                                                                                                     \\ \hline
\texttt{-{}-output}       & \begin{tabular}[c]{@{}l@{}}Directory to output the learned model \\ and metrics.\end{tabular}                                                                                                                                 \\ \hline
\texttt{-{}-workerid}     & \begin{tabular}[c]{@{}l@{}}Unique training ID, if running the program \\ in parallel.\end{tabular}                                                                                                                            \\ \hline
\texttt{-{}-alphamin}     & \begin{tabular}[c]{@{}l@{}}Minimum value of the learning rate. \\ Decays linearly to the value originally \\ from \texttt{-{}-alpha} over the number of training \\ episodes specified by \texttt{-{}-episodes}.\end{tabular} \\ \hline
\texttt{-{}-epsmin}       & \begin{tabular}[c]{@{}l@{}}Minimum value of epsilon. \\ Decays linearly to the value originally \\ from \texttt{-{}-epsilon} over the number of training \\ episodes specified by \texttt{-{}-episodes}.\end{tabular}         \\ \hline
\texttt{-{}-envconfig}    & \begin{tabular}[c]{@{}l@{}}JSON string containing the \\ \texttt{EnvironmentConfig}.\end{tabular}                                                                                                                             \\ \hline
\texttt{-{}-numagents}    & Number of agents.                                                                                                                                                                                                             \\ \hline
\end{tabular}
\caption{\texttt{iql\_train.py} arguments.}
\label{tab:iql-train-args}
\end{table}

\begin{table}[]
\centering
\begin{tabular}{|l|l|}
\hline
\textbf{Argument}      & \textbf{\begin{tabular}[c]{@{}l@{}}Argument Description\end{tabular}}                            \\ \hline
\texttt{-{}-modelpath} & \begin{tabular}[c]{@{}l@{}}Path to the pickled model file \\ (Q-Table) from training.\end{tabular} \\ \hline
\texttt{-{}-episodes}  & Number of episodes to run.                                                                         \\ \hline
\texttt{-{}-envconfig} & \begin{tabular}[c]{@{}l@{}}JSON string containing the \\ \texttt{EnvironmentConfig}.\end{tabular}  \\ \hline
\texttt{-{}-numagents} & Number of agents.                                                                                  \\ \hline
\texttt{-{}-build}     & Directory of the Unity environment build.                                                          \\ \hline
\end{tabular}
\caption{\texttt{iql\_run.py} arguments.}
\label{tab:iql-run-args}
\end{table}

\begin{table}[]
\centering
\begin{tabular}{|l|l|}
\hline
\textbf{\begin{tabular}[c]{@{}l@{}}Argument\\ Index\end{tabular}} & \textbf{Argument Description}                                                                                                                                                                                                                                                    \\ \hline
First $n$                                                         & \begin{tabular}[c]{@{}l@{}}Paths to the $n$ pickled metrics files \\ (cumulative rewards for each episode for all agents) \\ for the $n$ runs.  \\ Each path is suffixed by \texttt{\&<legend>} where \\ \texttt{<legend>} is the legend for the line for the model.\end{tabular} \\ \hline
$n+1$                                                            & \begin{tabular}[c]{@{}l@{}}Size of the window to take the rolling average over \\ in the cumulative reward graph.\end{tabular}                                                                                                                                                   \\ \hline
$n+2$                                                             & \begin{tabular}[c]{@{}l@{}}Number of episodes used for training, used to plot the \\ red dashed line indicating when the agents entered \\ their evaluation zone.\end{tabular}                                                                                                   \\ \hline
\end{tabular}
\caption{\texttt{plot.py} arguments.}
\label{tab:plot-args}
\end{table}

\begin{table}[]
\centering
\begin{tabular}{|l|l|}
\hline
\textbf{Parameter}                    & \textbf{Value}                                                         \\ \hline
\texttt{\_positionGranularity}        & 4                                                                      \\ \hline
\texttt{\_velocityGranularity}        & 4                                                                      \\ \hline
\texttt{\_thetaGranularity}           & 24                                                                     \\ \hline
\texttt{\_maxVelocityMagnitude}       & 4                                                                      \\ \hline
\texttt{\_minVelocityMagnitude}       & 2                                                                      \\ \hline
\texttt{\_maxDeltaVMagnitude}         & 2                                                                      \\ \hline
\texttt{\_minDeltaVMagnitude}         & 2                                                                      \\ \hline
\texttt{\_maxDeltaThetaMagnitude}     & 3                                                                      \\ \hline
\texttt{\_normalizeObs}               & true                                                                   \\ \hline
\texttt{\_numParkedCars}              & 16                                                                     \\ \hline
\texttt{\_obsDist}                    & true                                                                   \\ \hline
\texttt{\_obsRings}                   & true                                                                   \\ \hline
\texttt{\_ringMaxNumObjTrack}         & 1                                                                      \\ \hline
\texttt{ringDiams}                    & [11]                                                                   \\ \hline
\texttt{\_rd0}                        & 11                                                                     \\ \hline
\texttt{\_obsGoalDeltaPose}           & true                                                                   \\ \hline
\texttt{\_ringOnlyWall}               & true                                                                   \\ \hline
\texttt{\_obsNearbyCars}              & true                                                                   \\ \hline
\texttt{\_obsNearbyCarsCount}         & 1                                                                      \\ \hline
\texttt{\_obsNearbyCarsDiameter}      & \begin{tabular}[c]{@{}l@{}}300 \\ (no constrained \\ FOV)\end{tabular} \\ \hline
\texttt{spawnCloseRatio}              & 0.2                                                                    \\ \hline
\texttt{carSpawnMinDistance}          & 210                                                                    \\ \hline
\texttt{\_maxSteps}                   & 200                                                                    \\ \hline
\texttt{\_obsNearbyCarsGoal}          & true                                                                   \\ \hline
\texttt{\_obsNearbyCarsVelocity}      & true                                                                   \\ \hline
\texttt{spawnCrashRatio}              & 0.2                                                                    \\ \hline
\texttt{spawnCrashTargetAgentMinDist} & 210                                                                    \\ \hline
\texttt{rewTimeSum}                   & 0.2                                                                    \\ \hline
\texttt{rewReachGoal}                 & 1                                                                      \\ \hline
\texttt{rewCrash}                     & 1                                                                      \\ \hline
\texttt{rewReverseSum}                & 0.1                                                                    \\ \hline
\texttt{rewDistSum}                   & 0.15                                                                   \\ \hline
\texttt{rewDeltaThetaSum}                   & 0.05                                                                   \\ \hline
\end{tabular}
\caption{Values in the \texttt{EnvironmentConfig} for PPO experiments with fixed goals.}
\label{tab:ppo-fixed-goals-env-config}
\end{table}

\begin{table}[]
\centering
\begin{tabular}{|l|l|}
\hline
\textbf{Parameter}                    & \textbf{Value}                                                                                           \\ \hline
\texttt{\_dynamicGoals}               & true                                                                                                     \\ \hline
\texttt{\_obsNearbyParkingSpotsCount} & 1                                                                                                        \\ \hline
\texttt{rewDeltaGoalContinueExp}      & -0.002                                                                                                   \\ \hline
\texttt{rewDeltaGoalDiffGoal}         & -0.05                                                                                                    \\ \hline
\texttt{\_rewDeltaGoalStopGoal}       & \begin{tabular}[c]{@{}l@{}}-1 \\ (compute via\\ \cref{eq:rew-change-goal-equiv-simplified})\end{tabular} \\ \hline
\end{tabular}
\caption{Values in the \texttt{EnvironmentConfig} for PPO experiments with dynamic goals.}
\label{tab:ppo-dynamic-goals-env-config}
\end{table}

\begin{table}[]
\centering
\begin{tabular}{|l|l|}
\hline
\textbf{Argument Index} & \textbf{Description}                                                                                                                                                                                                                                                                                                                                                                                                                                                                                                   \\ \hline
1                       & \texttt{mlagents-learn} command.                                                                                                                                                                                                                                                                                                                                                                                                                                                                                       \\ \hline
2                       & \begin{tabular}[c]{@{}l@{}}Base PPO YAML configuration file to base the \\ grid-search upon,  containing the base PPO \\ hyperparameters and environment parameters. \\ We inject the grid-searched parameters into \\ copies of it.\end{tabular}                                                                                                                                                                                                                                                                      \\ \hline
3                       & \begin{tabular}[c]{@{}l@{}}Base \texttt{.slurm} job file to base the grid-search upon. \\ We inject the different commands in each job into copies \\ of it corresponding to each parameter combination.\end{tabular}                                                                                                                                                                                                                                                                                                  \\ \hline
4                       & \begin{tabular}[c]{@{}l@{}}Base port. Each run uses the ports sequentially \\ greater than the base port.\end{tabular}                                                                                                                                                                                                                                                                                                                                                                                                 \\ \hline
5                       & \begin{tabular}[c]{@{}l@{}}Base run id, which is the prefix of the run ID \\ for each generated run.\end{tabular}                                                                                                                                                                                                                                                                                                                                                                                                      \\ \hline
6 to last               & \begin{tabular}[c]{@{}l@{}}The parameters in the PPO YAML configuration \\ to grid-search over,  and their respective ranges.\\ Each argument is of the form: \texttt{<yaml key>=<values>},\\ where \texttt{<yaml key>} is a key in the YAML file,\\ and \texttt{<values>} is of the form \texttt{<type>[$v_1, ... ,v_n$]}\\ where \texttt{type} is one of \texttt{f(loat)}, \texttt{i(nt)}, \texttt{s(tring)} or \texttt{b(ool)},\\ and $v_i$ are the values, all of type \texttt{<type>}.\end{tabular} \\ \hline
\end{tabular}
\caption{\texttt{gridsearch\_ppo.py} arguments.}
\label{tab:gridsearch-arguments}
\end{table}

\end{document}